\documentclass[11pt]{article}

\usepackage[utf8]{inputenc}
\usepackage[T1]{fontenc}
\usepackage[english]{babel}
\usepackage[]{algorithm2e}
\usepackage{mwe}
\usepackage{amsmath,mathtools}
\usepackage[dvipsnames]{xcolor}

\usepackage{amsthm,bm}
\usepackage{lineno}

\usepackage{tikz,pgfplots}
\pgfplotsset{compat=1.5}
\usepgfplotslibrary{groupplots}
\definecolor{paperBlue}{HTML}{174A7E}
\definecolor{paperRed}{HTML}{D62728}
\definecolor{paperBrown}{HTML}{8C564B}
\definecolor{paperViolet}{HTML}{9467BD}
\definecolor{paperGreen}{HTML}{2D8B45}
\definecolor{paperOrange}{HTML}{EE8914}
\definecolor{paperBlack}{HTML}{000000}
\pgfplotsset{
  paper line/.style={line width=0.80000000bp,solid,line cap=round,line join=round,mark size=1.40000000bp,mark options={solid,line width=0.55000000bp}},
  paper dashed/.style={dash pattern=on 3.20000000bp off 1.80000000bp},
  legend image post style={line width=0.8bp,mark size=1.4bp,mark options={line width=0.55bp}},
  paper benchmark legend/.style={legend image post style={dash pattern=on 3.2bp off 1.8bp}},
  paper hollow/.style={mark options={solid,fill=white,line width=0.55000000bp},legend image post style={mark options={solid,fill=white,line width=0.55bp}}},
  m0/.style={paper line,paperBlue,mark=*},
  m1/.style={paper line,paperRed,mark=square*},
  m2/.style={paper line,paperBrown,mark=triangle*},
  m5/.style={paper line,paperViolet,mark=diamond*},
  m10/.style={paper line,paperGreen,mark=*},
  fullinfo/.style={paper line,paperBlack,paper dashed,mark=none,paper benchmark legend},
  paper km/.style={paper line,paperBlue,mark=*},
  paper lp/.style={paper line,paperViolet,mark=diamond*},
  paper benchmark/.style={paper line,paperBlack,paper dashed,mark=none,paper benchmark legend}
}

\newcommand{\paperlegend}[2]{\tikz[baseline=-0.5ex]{\draw[line width=0.8bp,mark size=1.4bp,mark options={solid,line width=0.55bp},#1] plot[mark repeat=2,mark phase=2] coordinates {(0,0) (0.225,0) (0.45,0)};}\,#2}

\usepackage{subfigure}
\usepackage{multirow}

\usepackage[shortlabels]{enumitem}
\usepackage[colorlinks=true, linkcolor=blue, citecolor=blue, hypertexnames=false]{hyperref}  
\usepackage{nameref}
\usepackage{cleveref}

\usepackage[right=1.0in, top=1in, bottom=1.0in, left=1.0in]{geometry}

\usepackage{lmodern}

\usepackage{graphicx}
\usepackage{placeins}
\usepackage{lscape}
\usepackage[final]{pdfpages}
\usepackage{amsthm}
\usepackage[authoryear,round]{natbib}
\usepackage{amsfonts}
\usepackage{mathtools}
\usepackage{bbm}
\usepackage{dsfont}

\usepackage{bbm}
\newtheorem{lemma}{Lemma}
\newtheorem{proposition}{Proposition}
\newtheorem{corollary}{Corollary}
\newtheorem{theorem}{Theorem}
\newtheorem*{theorem*}{Theorem}

\newtheorem{definition}{Definition}

\newtheorem{example}{Example}

\usepackage{setspace}
\newcommand{\E}{\mathbb E}
\newcommand{\Prob}{\mathbb{P}}

\newcommand {\beq}{\begin{equation}}
\newcommand {\eeq}{\end{equation}}
\newcommand {\beqn}{\begin{equation*}}
\newcommand {\eeqn}{\end{equation*}}
\newcommand {\bear}{\begin{eqnarray}}
\newcommand {\eear}{\end{eqnarray}}
\newcommand {\bearn}{\begin{eqnarray*}}
\newcommand {\eearn}{\end{eqnarray*}}

\DeclareMathOperator*{\argmax}{arg max}
\DeclareMathOperator*{\argmin}{arg min}

\usepackage[normalem]{ulem}

\newcommand{\pastI}{\bm{x}}
\newcommand{\KMF}{\hat{F}_{\mathrm{KM}}}

\newcommand{\D}{\mathcal{D}}

\renewcommand{\ss}[1]{S^{(#1)}}
\newcommand{\I}[1]{I^{(#1)}}

\newcommand{\BSAA}{\pi^{\mathrm{BSAA}}}
\newcommand{\psibsaa}{\Psi^{\mathrm{BSAA}}}

\newcommand{\KM}{\pi^{\mathrm{KM}}}
\newcommand{\ceil}[1]{\lceil #1 \rceil}

\newcommand{\opt}{\mathrm{opt}}
\newcommand{\numS}{n}

\newcommand{\cu}{c_u}
\newcommand{\co}{c_o}

\newcommand{\reg}{\mathsf{Reg}}

\title{What is the Value of Censored Data? \\ An Exact Analysis for the Data-driven Newsvendor\vspace{1em}}

\author{
\sf Rachitesh Kumar\\
\sf Carnegie Mellon University\\
\texttt{rachitesh@cmu.edu}
\and
\sf Omar Mouchtaki\\
\sf New York University\\
\texttt{om2166@stern.nyu.edu}
}
\date{
Current Version: October 8, 2026\\
Original Version: February 9, 2026
}
\begin{document}

\maketitle

\begin{abstract}
We study the offline data-driven newsvendor problem when historical demand is censored by past inventory levels. 
We introduce a class of policies for which worst-case expected regret over all bounded demand distributions admits an exact finite-dimensional characterization, valid for any sample size and historical inventory design. The dimension depends on the number of distinct inventory levels, rather than the number of observations. This framework yields exact analyses of the sales-as-demand heuristic, the Kaplan--Meier policy, and recent policies from the censored-newsvendor literature. 
We leverage this reduction to derive sharp insights on the achievable performance of standard inventory policies under demand censoring. In particular, our analysis of the Kaplan-Meier policy shows that while demand censoring fundamentally limits what can be learned from passive sales data, just a small amount of targeted exploration at high inventory levels can substantially improve worst-case guarantees.
Finally, we further leverage our reduction to design new policies  that push the frontier and  substantially improve upon the worst-case regret of known policies, reducing it by up to $40\%$.


\end{abstract}

\section{Introduction}
The newsvendor problem is a canonical model for studying inventory decisions under uncertainty. A decision maker selects an inventory level prior to demand realization and incurs overage costs when inventory exceeds demand and underage costs when demand exceeds inventory. When the demand distribution is known, the optimal inventory policy is fully characterized. In many practical retail settings, however, the demand distribution is unknown and must be inferred from historical data. This has motivated a broad literature on data-driven inventory optimization, which seeks to design policies that map past observations to inventory decisions without relying on parametric assumptions on demand.

A central difficulty in applying data-driven methods to inventory problems is that demand is often not directly observed. Instead, the decision maker observes only realized sales, which equal the minimum of the demand and the inventory level. As a consequence, demand observations are censored whenever demand exceeds inventory, and the extent of censoring depends on the past inventory levels under which the data was collected. Crucially, high-demand realizations are systematically truncated, and the historical data does not consist of independent samples from the underlying demand distribution. Such censoring is pervasive in retail and supply-chain environments and raises nontrivial challenges for both learning and performance evaluation.

Most of the existing literature addresses censored demand through online learning frameworks in which the decision maker adaptively varies inventory levels over time in order to balance learning and exploitation. By deliberately experimenting with different inventory levels, these approaches aim to gradually uncover the demand distribution and achieve asymptotic optimality or low cumulative regret relative to a clairvoyant benchmark. While theoretically appealing, such adaptive strategies may be operationally costly or infeasible in practice. Frequent changes in order quantities can create coordination frictions with suppliers, complicate logistics and replenishment planning, and require costly adjustments in shelf space or storage capacity. In many retail settings, inventory policies are therefore chosen offline using historical sales data and then applied in a relatively stable manner, without being optimized for learning.

In contrast, in the offline regime with censored demand, it is not clear which data-driven inventory policy should be used in practice. The sample average approximation (SAA) policy, which is the standard benchmark in the uncensored identically and independently distributed (i.i.d.) setting, selects an inventory level by optimizing the empirical objective obtained by treating past observations as direct samples of demand. This approach implicitly assumes that the observed data provides an unbiased representation of the demand distribution and therefore fails to properly account for the fact that the demand is censored by past inventory choices. A natural alternative, inspired by classical results in survival analysis, is to first estimate the demand distribution using a Kaplan--Meier--type estimator and then optimize inventory against this estimate. While statistically appealing, this approach induces inventory decisions that have an intricate dependence on the data. As a result, deriving sharp finite-sample performance guarantees for such policies has proven challenging, and existing results offer limited guidance on their operational effectiveness.

These observations motivate the need for a principled framework to evaluate the performance of data-driven inventory policies in the presence of censored demand information. Such a framework allows one to assess how censoring, together with past inventory decisions, shapes the information available for inventory optimization, and to develop a systematic understanding of the performance of standard data-driven policies under these informational constraints. In particular, it enables meaningful comparisons across policies as the extent of censoring varies.

\subsection{Main Contributions}
\noindent \textbf{Exact performance of a broad class of inventory policies.}
We characterize the exact worst-case expected regret of central data-driven policies in the censored-newsvendor literature: the simple sales-as-demand heuristic, the policy based on the classical Kaplan--Meier (KM) estimator \citep{kaplan1958nonparametric}, and the recent policies of \citet{fan2022sample} and \citet{hssaine2024data}. The main difficulty is that effective censoring-aware policies minimize newsvendor cost under a weighted empirical demand distribution whose weights depend intricately on the observations and the pattern of censoring. For Kaplan--Meier, for instance, each censoring event redistributes mass and changes the weights assigned to subsequent observations. This dependence makes sharp analysis substantially more difficult than for ordinary sample average approximation.

We overcome this difficulty by exploiting the structure of these policies to establish that the non-convex infinite-dimensional problem whose value equals the worst-case regret of the data-driven policy can be reduced to a finite-dimensional optimization problem, whose dimension does not depend on the number of samples observed. Our framework accommodates arbitrary fixed historical inventories and sample sizes.

\begin{quote}
    \emph{As a corollary, we obtain, to our knowledge, the first finite-sample characterization of the worst-case regret of Kaplan--Meier inventory decisions, the canonical estimator for censored observations. This constitutes a paradigm shift because even the worst-case regret rate of this policy was unknown (except when past inventories are i.i.d.---an assumption routinely violated in practice for inventory management).}
\end{quote}
 Earlier inventory analyses establish asymptotic consistency \citep{huh2011adaptive}, while classical nonasymptotic statistical bounds for the estimator assume i.i.d.\ random censoring \citep{bitouze1999dkm}, which can therefore not be applied to arbitrary prescribed inventory levels. Our analysis characterizes decision performance exactly, allowing precise comparisons across policies and censoring regimes even with small samples.

\medskip
\noindent \textbf{The value of censored data and limited exploration.}
When all historical inventories are bounded by a fixed level below the demand upper bound, additional censored observations alone cannot eliminate worst-case regret: distributions with different unobserved tails remain indistinguishable. Without exploration, our numerical evaluations show that Kaplan--Meier's performance levels off after only tens of observations. 
\begin{quote}
\emph{Yet a few exploratory samples, collected by deliberately ordering at a higher inventory to observe demand fully, can dramatically improve performance.}     
\end{quote}
Table~\ref{tab:intro_km_exploration} summarizes this effect. With 50 observations, replacing just one censored observation with an exploratory sample reduces KM's worst-case regret by about $23.5\%$; five uncensored samples reduce it by about $61\%$. Operationally, this shows that very strong performance improvement is achievable under censored information with only a tiny, targeted amount of exploration, rather than sustained experimentation.

\begin{table}[htbp]
\centering
\small
\begin{tabular}{lrrrr}
\hline
& \multicolumn{4}{c}{Exploratory samples $m$} \\
Policy  & $0$ & $1$ & $2$ & $5$ \\
\hline
Kaplan--Meier & $0.043$ & $0.033$ & $0.026$ & $0.017$ \\
Our LP policy & $73.9\%$ & $53.7\%$ & $41.9\%$ & $60.1\%$ \\
Uncensored minimax lower bound & $22.3\%$ & $29.1\%$ & $36.6\%$ & $57.1\%$ \\
\hline
\end{tabular}
\caption{Worst-case regret comparison at $n=50$, for the experimental setting in \Cref{sec:new_policies}. KM is reported in absolute terms; the other rows are percentages of KM. Entries are rounded from the computed regret bounds.}
\label{tab:intro_km_exploration}
\end{table}

The exact characterizations also allow us to compare recent policies under a common performance criterion (see \Cref{sec:ca_other_policies}). In our deterministic numerical comparison, Kaplan--Meier and the policy of \citet{fan2022sample} (FCZ) have the lowest estimated worst-case regrets, with their computed maxima differing by at most $1.1\times10^{-4}$ throughout the reported grid. This numerical agreement is striking because FCZ uses only sales and historical inventories, whereas KM additionally uses indicators distinguishing fully observed demands from lost sales. These experiments therefore suggest that essentially the same worst-case performance is achievable with less information.

\medskip
\noindent \textbf{Designing new policies with stronger guarantees.}
Finally, we use the framework to design policies that substantially improve on these existing benchmarks. We optimize over data-driven policies that randomize over sales order statistics and the support endpoints. \Cref{thm:new_policy_design} gives a linear programming formulation for optimizing over this class. Although KM was the best existing policy in our comparison, our designed policies improve on it in all instances studied. Table~\ref{tab:intro_km_exploration} illustrates the gains: at $n=50$, worst-case regret falls by $26\%$ to $58\%$ across the exploration allocations. 
\begin{quote}
   \emph{With five exploratory samples, our policy improves on KM by about $40\%$ and achieves regret only  $5.3\%$ above the minimax value with all 50 demands fully observed. } 
\end{quote}
This benchmark lower-bounds the worst-case regret of every policy based on 50 observations with censoring, certifying near-minimax performance in this setting.

\subsection{Related Literature}

The newsvendor problem is a canonical model in operations management and operations research, and serves as a foundational framework for studying inventory and capacity decisions under demand uncertainty. Its analytical tractability and economic interpretability have made it a standard benchmark for understanding the value of information, robustness, and learning in stochastic inventory systems.

A large literature studies the newsvendor problem when the demand distribution is unknown and must be inferred from limited information. Early work adopts a distributionally robust perspective, characterizing optimal policies when only partial features of the demand distribution are known. Classical contributions by \cite{scarf1958min} and \cite{gallego1993distribution} derive minimax optimal ordering policies when the mean and variance of demand are specified. This line of work was  extended by \cite{perakis2008regret,natarajan2018asymmetry}, who develop minimax regret policies under a wide range of informational assumptions.

An alternative approach considers the data-driven setting in which the decision-maker has access to  \textit{demand observations} that are independent and identically distributed and fully observed.  \cite{levi2007approximation} and \cite{levi2015data} establish probabilistic bounds on the worst-case relative regret of the Sample Average Approximation policy, which selects the inventory level which has the lowest cost on past data, while \cite{cheung2019sampling} provide matching lower bounds. \cite{lin2022data} analyzes the expected additive regret of SAA whereas the recent survey of \citet{chen2024survey} unifies and generalizes these approaches and provide upper and lower bounds on both the additive and relative regret for various sub-classes of distributions. 
Closer to us, \cite{besbes2023big} characterize the exact worst-case expected relative regret of SAA for any finite sample size and derive a minimax optimal data-driven policy. 
Our work departs from this literature by focusing on the practical settings in which historical demand observations cannot be fully observed as it may be censored by previously offered inventory levels. Therefore, the samples observed by the decision-maker do not constitute i.i.d. samples from the underlying demand distribution.

Demand censoring arises naturally in inventory systems because sales data only reveal the minimum of demand and available inventory. This issue has been widely studied in sequential decision-making settings, where the decision-maker repeatedly orders inventory and observes censored sales. Early work shows that gradient-based and stochastic approximation methods can achieve strong performance despite censored feedback \citep{burnetas2000adaptive,godfrey2001adaptive,kunnumkal2008using,huh2009nonparametric}. \cite{huh2011adaptive} establish that, under appropriate conditions, nonparametric survival-analysis tools such as the Kaplan-Meier estimator can be embedded within adaptive ordering policies and yield asymptotic consistency, while \cite{besbes2013implications} analyze how censoring fundamentally alters the information structure of repeated newsvendor problems.
These ideas have been extended to data-driven censored settings that incorporate additional sources of complexity. Examples include inventory systems with capacity limitations \citep{shi2016nonparametric}, nonconvex losses due to fixed costs \citep{yuan2021marrying}, perishability \citep{zhang2018perishable}, lead time \citep{zhang2020closing,agrawal2022learning}, substitution behavior across products \citep{chen2020dynamic} and non-stationary demand processes \citep{lugosi2024hardness}. Across these settings, learning relies on the ability to vary inventory decisions over time, trading off immediate costs against the informational value of revealing censored demand. We study the offline case, which is common in practice when firms rely on previously collected sales data, where such adaptivity is absent and censoring fundamentally changes both what can be learned and how policy performance should be evaluated.

The most closely related works are those analyzing the offline setting with censored demand. \cite{ban2020confidence} derive a consistency result and an asymptotically normal estimator of the optimal $(s,S)$ policy under the assumption that historical inventory levels exceed the optimal ordering quantity. 
\cite{fan2022sample} study offline learning when historical inventories are generated by a data-collection policy and adopt a PAC-style framework. They characterize conditions under which the optimal newsvendor solution can be learned and establish upper and lower bounds on the corresponding sample complexity; these bounds only hold when the sample size is sufficiently large. \cite{hssaine2024data} study the fundamental limits imposed by censoring through a distributionally robust regret criterion, including regimes in which the optimal decision is not identifiable even with arbitrarily many samples, and provide finite-sample guarantees and matching lower bounds across identifiable and unidentifiable regimes.
Our work differs from these approaches in that it provides a methodology to derive an \emph{exact} characterization of the worst-case performance of classical data-driven policies for \emph{all sample sizes}, including the canonical Sample Average Approximation and Kaplan--Meier–based policies. This focus leads to fundamentally different methodological tools: our analysis relies on an optimization-based framework, whereas much of the existing literature derives performance guarantees via concentration-based arguments. Moreover, similarly to \cite{hssaine2024data}, our methodology does not require identifiability of the problem.

Finally, our work relates to a growing literature that aims to provide a granular quantification of the value of data.
\cite{gupta2021small,gupta2022data,gupta2024debiasing} develop data-driven methods to solve a large number of optimization problems each involving a small number of samples. 
\cite{liu2023marginal} quantify the marginal value of an additional sample in assortment optimization, while \cite{zhang2024more} analyze the trade-off between data quality and data quantity in the newsvendor problem. More closely related to our work are papers that derive exact characterizations of the value of data in small to moderate sample regimes, including studies in pricing \citep{fu2015randomization,babaioff2018two,huang2018making,daskalakis2020more,allouah2022pricing,allouah2023optimal,bahamou2024fast}, experiment design \citep{schlag2006eleven,stoye2009minimax}, and inventory management \citep{besbes2023big,besbes2023contextual}.

\section{Model}\label{sec:model}

\paragraph{Notation.} Boldface denotes vectors and $[n]\coloneqq\{1,2,\ldots,n\}$. Superscripts in parentheses (e.g., $S^{(k)}$) are labels/indices rather than powers. For distribution functions, $F$ denotes a (right-continuous) CDF, with left limits denoted by $F(x^-)$. Functions of a scalar are applied to vectors componentwise, so that $F(\bm{y}) = (F(y_1),\ldots,F(y_m))$ and $F(\bm{y}^-) = (F(y_1^-),\ldots,F(y_m^-))$ for a vector $\bm{y} = (y_1,\ldots,y_m)$. $y_+ = \max\{0,y\}$ denotes the positive part. Finally, $\Delta(\mathcal X)$ denotes the set of distributions supported on $\mathcal X$, and $\delta_x$ denotes a point mass at $x$. We use $\mathbb{R}$, $\mathbb{Z}$, and $\mathbb{N}$ to denote the real, integer, and natural numbers, respectively, and $\mathbb{R}_+$ for the nonnegative subset of $\mathbb R$.

We consider a newsvendor problem in which the decision-maker (for example, a retail firm) needs to order an inventory quantity $a$ before the true demand $D$ is realized. We assume that the demand $D$ is bounded, as is the case in practice. Thus, without loss of generality, it can be scaled to always lie in $[0,1]$. We use $c_u > 0$ and $c_o > 0$ to denote the underage and overage cost respectively, i.e., $c_u$ is the per unit cost incurred from placing an inventory order $a$ that falls short of the demand $D$ and causes a loss in sales, and $c_o$ is the per unit cost associated with an inventory order which exceeds demand $D$ and leads to holding costs. Therefore, given an order quantity $a \in [0,1]$, the total cost is given by
\begin{align*}
    c(a,D)\ \coloneqq\ c_o \cdot (a - D)_+\ +\ c_u \cdot (D- a)_+\,.
\end{align*}

We assume that the demand $D$ is generated according to the distribution $F \in \Delta([0,1])$. Complete information about this demand distribution is rarely available in practice, and we do not assume that $F$ is known to the decision-maker. On the other hand, firms routinely collect data about sales which provides information about the demand distribution $F$. This motivates us to assume that the decision-maker has access to $n$ past observations about sales. In particular, suppose the decision-maker ordered inventory quantities $\bm x = (x_k)_{k=1}^K$ with $0 \leq x_1< x_2 < \dots < x_K \leq 1$ in the past and, for each $k \in \{1,\ldots,K\}$, observed the associated sales data $\bm{S^{(k)}} = (\ss{k}_1, \dots, \ss{k}_{n_k})$. Here, $n_k$ is the number of censored demand samples observed under inventory $x_k$, and $\bm{\ss{k}}$ are the corresponding censored samples with
$$\ss{k}_i = \min\{D^{(k)}_i, x_k\}\,,$$ 
where $D^{(k)}_i$ is an independent demand sample from $F$. We use $n = \sum_{k=1}^K n_k$ to denote the total number of historical samples, $\bm D \coloneqq \left(D^{(k)}_i \mid k \in [K],\ i \in [n_k] \right)$ to denote the vector of all historical demands, and $\bm S \coloneqq (\bm{\ss{k}})_{k \in [K]}$ to denote the collection of sales data across all inventories. Importantly, the decision-maker does not have access to the demand data $D_i^{(k)}$ directly because it was censored by the historical inventory $x_k$ ordered at that time. In other words, the sales $\ss{k}_i$ equal the true demand $D^{(k)}_i$ only when a stock-out does not occur, and equal the historical inventory $x_k$ otherwise. Moreover, for $\delta^{(k)}_i = \mathds{1}(D^{(k)}_i \leq x_k)$, we use $\I{k}_i = (\ss{k}_i, \delta^{(k)}_i)$ to denote the information obtained by the seller from sample $i \in [n_k]$ censored by historical inventory $x_k$. It includes the amount of observed sales $\ss{k}_i = \min\{D^{(k)}_i, x_k\}$ and an indicator $\delta^{(k)}_i$ which captures whether or not a stockout occurred. We use $\bm{\I{k}} \coloneqq (\I{k}_1, \dots, \I{k}_{n_k})$ to denote the corresponding vector of historical information for historical inventory~$x_k$, $\bm I \coloneqq (\bm{\I{k}})_{k\in [K]}$ to denote the combined historical information vector, and $\mathcal{I}$ to denote the set of all such information vectors.

A \emph{data-driven policy} is an inventory decision rule which prescribes an inventory quantity $a$ based on historical censored information $\bm I$. Formally, for $\bm{n} \coloneqq (n_1, \dots, n_K) \in \mathbb{N}^K$, a data-driven policy is a function $\pi_{\bm{n}}:[0,1]^K \times \mathcal{I} \to [0,1]$ which maps historical inventories $\bm{x} = (x_1, \dots, x_K)$ and their corresponding sales information $\bm I$ to an inventory quantity $a = \pi_{\bm n}(\bm x, \bm I)$ for new demand $D \sim F$ in the future. It represents an operational decision rule that can be implemented directly from routinely recorded transactional data, namely historical inventory levels and the associated censored sales and stockout observations, without assuming access to uncensored demand realizations or prior knowledge of the distribution $F$.

We measure the performance of policies based on their worst-case regret against the optimal expected newsvendor cost for the demand distribution $F$. In particular, the expected newsvendor cost for inventory quantity $a \in [0,1]$ and demand distribution $F$ is defined as
\begin{align*}
    c(a,F)\ \coloneqq\ \E_{D \sim F}[c(a,D)]\ =\ \E_{D\sim F}[c_o \cdot (a - D)_+\ +\ c_u \cdot (D- a)_+]\,.
\end{align*}
Let $\opt(F) \coloneqq \min_{a \in [0,1]} c(a,F)$ be the \emph{optimal newsvendor cost} for the demand distribution $F$, and $a^*(F) \in \argmin_{a \in [0,1]} c(a,F)$ denote a corresponding \emph{optimal inventory quantity}. Moreover, define the \emph{critical fractile} as
\begin{align*}
    q\ \coloneqq\ \frac{c_u}{c_o + c_u}\,.
\end{align*}
It is a well-known result that $a^*(F)$ is a $q$-quantile of $F$, i.e., every optimal inventory quantity satisfies $a^*(F)\ \in\ \{a \in [0,1] \mid F(a^-) \leq q \leq F(a)\}$.

For a demand distribution $F$, the \emph{regret} $R(a, F)$ of an inventory quantity $a$ is defined as the excess expected cost incurred by $a$ in excess of the optimal cost $\opt(F)$:
\begin{align*}
    R(a, F)\ \coloneqq\ c(a,F)\ -\ \opt(F)\,.
\end{align*}
Note that scaling both overage and underage costs $c_o, c_u$ by a common positive factor scales $c(a,D)$, $\opt(F)$, and $R(a,F)$ by the same factor and leaves the critical fractile $q$ unchanged---such a scaling amounts to a change of units used to measure cost. Therefore, without loss of generality, we normalize $c_o + c_u = 1$, so that $q = \cu$.

We can now rigorously define our performance measure: for a given data-driven policy $\pi_{\bm n}$\footnote{For analysis purposes, we will also consider randomized policies defined as mappings $\pi_{\bm{n}} : [0,1]^K \times \mathcal{I} \to \Delta([0,1])$. Abusing notation, we denote its regret by $R(\pi_{\bm n}(\bm x, \bm I), F) = \mathbb{E}_{a \sim \pi_{\bm n}(\bm x, \bm I)}[R(a,F)].$ }, the \emph{worst-case regret} is defined as
\begin{align}\label{eq:worst-case-regret}
    \reg(\pi_{\bm n} \mid \bm x)\ \coloneqq\ \sup_{F \in \Delta([0,1])}\ \E_{\bm{D} \sim F^n}\left[ R(\pi_{\bm n}(\bm x, \bm I), F) \right]\,,
\end{align}
where the expectation is taken over the historical demands $\bm{D} \sim F^n$ that were not directly observed but provide censored information to the policy via the information vector $\bm I$ with $${\I{k}_i = (\min\{D^{(k)}_i, x_k\}, \mathds{1}(D^{(k)}_i \leq x_k))}\,.$$

Note that $\reg(\pi_{\bm n}\mid \bm x)$ is distribution-free: it does not posit a parametric model for demand and instead evaluates a policy under the least favorable $F\in\Delta([0,1])$. Consequently, if ${\reg(\pi_{\bm n}\mid \bm x)}\le \epsilon$, then for every underlying demand distribution $F$, the policy's expected out-of-sample cost (averaging over the historical sample) is at most $\opt(F)+\epsilon$. This worst-case perspective is especially natural under censoring, where the tail of demand is only partially revealed and modeling assumptions are difficult to validate. It provides uniform performance guarantees that apply across all distributions, and thus a principled basis for comparing data-driven policies.

We conclude this section with a discussion of the following natural and popular policies for making inventory decisions with censored data. These policies constitute the main objects of evaluation in our subsequent analysis.

\noindent \textbf{Biased Sample Average Approximation (BSAA).} This policy ignores censorship for simplicity and treats the sales data $\ss{k}_i$ as independent samples from the demand distribution $F$. It then minimizes the newsvendor cost on the empirical distribution formed by the sales data $\bm S$. In particular, if $\hat F_{\mathrm BSAA}$ denotes the empirical distribution of $\left\{\ss{k}_i \mid k \in [K]; i \in [n_k]\right\}$, then
    \begin{align*}
        \BSAA_{\bm n}(\bm x, \bm I)\ \coloneqq\ \inf\left\{ u \in [0,1] \mid \hat F_{\mathrm BSAA}(u) \geq q \right\}\,,
    \end{align*}
    which is an optimal inventory to order for the demand distribution $\hat F_{\mathrm BSAA}$, i.e., $$\BSAA_{\bm n}(\bm x, \bm I) \in \argmin_{a \in [0,1]}\ c(a, \hat F_{\mathrm  BSAA})\,.$$
    Importantly, observe that $\BSAA_{\bm n}(\bm x, \bm I)$ is purely a function of sales amounts $\bm S$ and does not depend on stockout events $\delta^{(k)}_i$. This sales-as-demand heuristic mirrors common practice in retail forecasting systems that rely on aggregate point-of-sale (POS) sales reports, even though stockouts censor demand and can induce systematic downward bias for popular items~\citep{jain2015demand}. The appeal of this approach lies in its simplicity and interpretability, which come at the cost of ignoring censorship.

\noindent \textbf{The Kaplan-Meier Policy (KM).} 
This policy first constructs an estimate of the demand distribution CDF using the Kaplan-Meier estimator, which is the canonical way to deal with censored data in the statistics and inventory literature. 

Let $Y_1 \le Y_2 \le \cdots \le Y_n$ denote the sales values $\{\ss{k}_j\}$ arranged in non-decreasing order, where ties are broken by placing uncensored observations before censored ones, and for each $i \in [n]$ let $\zeta_i \coloneqq \delta^{(k)}_j$ denote the stockout indicator of the observation $(k,j)$ occupying position $i$, so that $\zeta_i = 1$ if and only if $Y_i$ is an uncensored demand realization.
The Kaplan--Meier estimate of the demand cumulative distribution function is then given by
\begin{align*}
\hat F_{\mathrm{KM}}(z)
\;=\;
\begin{cases}
1 - \displaystyle\prod_{i : Y_i \le z}
\left(\frac{n-i}{n-i+1}\right)^{\zeta_i},
& \text{if } z < 1, \\[1ex]
1,
& \text{if } z = 1 .
\end{cases}
\end{align*}
The KM policy then selects the $q$-th quantile for this KM estimate of the CDF:
\begin{align*}
        \KM_{\bm n}(\bm x, \bm I)\ \coloneqq\ \inf\left\{ u \in [0,1] \mid \hat F_{\mathrm KM}(u) \geq q \right\}\,,
\end{align*}
which is an optimal inventory to order for the demand distribution $\hat F_{\mathrm KM}$, i.e., $$\KM_{\bm n}(\bm x, \bm I) \in \argmin_{a \in [0,1]}\ c(a, \hat F_{\mathrm KM})\,.$$
The Kaplan--Meier estimator is the canonical nonparametric estimator for right-censored data. It coincides with the empirical CDF in the absence of censoring, and it can be interpreted as the nonparametric maximum likelihood estimator (via the product-limit form) under independent censoring~\citep{kaplan1958nonparametric}.

\section{A Finite-Dimensional Characterization of Worst-Case Regret}\label{sec:main_results}

In this section, we develop a general methodology to evaluate the worst-case regret of data-driven newsvendor policies under offline demand censoring. The main challenge is that the regret criterion involves a supremum over the infinite-dimensional space of demand distributions $F\in\Delta([0,1])$ with the objective itself depending on $F$ in an intricate way: the demand distribution interacts non-linearly with the past inventory levels to generate the censored sales data, which is then taken as input by the policy to make an inventory decision, and regret depends on $F$ through the distribution of that decision. This difficulty is compounded for policies such as Kaplan--Meier, whose decision rule is censoring-aware and exhibits non-local dependence on the censored data. Our approach is to (i) rewrite expected regret in a form that depends on $F$ only through the induced action distribution, (ii) identify a broad structural class of policies, namely \emph{piecewise-separable} policies, for which this dependence can be summarized by finitely many CDF values at the censoring points, and (iii) prove  \Cref{thm:general_reduction}, a master reduction theorem that collapses the worst-case regret problem to a finite-dimensional optimization problem for such policies. We then instantiate this reduction to derive an exact characterization of worst-case regret for Biased-SAA in \Cref{sec:SAA_reduction}, evaluate its performance, and highlight the drawbacks of ignoring censorship information, which motivates our study of censoring-aware policies in the next section.

\subsection{General Reduction}\label{sec:gen_reduction}

We now establish our master reduction. Evaluating worst-case regret amounts to searching over $\Delta([0,1])$ for the least favorable demand distribution, which is a complex infinite-dimensional non-convex problem. The theorem below shows that, for a broad structured class of data-driven inventory policies, this search collapses to a finite-dimensional optimization problem, with a dimension that is independent of the sample size $n$, instead being governed by the number of distinct historical inventory levels $K$.

The first step is to express regret in terms of the distribution of the action induced by a policy, rather than directly in terms of the underlying distribution of samples. This perspective allows us to isolate the role of the policy through its induced action distribution and to identify structural conditions under which the resulting regret expression admits a pointwise optimization representation. 
\begin{lemma}\label{lem:integral_form}
For any policy $\pi_{\bm{n}}$, any fixed design $\pastI$, and any demand distribution $F$, the expected regret satisfies
\begin{equation*}
\mathbb{E}_{\bm{D}\sim F^n} \Big[ R \big(\pi_{\bm{n}}(\pastI,\bm{I}), F\big) \Big]
= \int_{0}^{1} \Big[
\big(1 - \Prob_{D_1,\ldots,D_n \sim F}\big( \pi_{\bm{n}}(\pastI,\bm{I}) \le z \big) \big)\cdot \big(F(z) - q\big)
+ (q - F(z))_{+}
\Big] dz,
\end{equation*}
where $q = \cu/(\cu+\co)$ is the critical fractile.
\end{lemma}

\Cref{lem:integral_form} extends Lemma 2 in \cite{besbes2023contextual} and converts the expectation with respect to the product measure $F^{n}$ into an integral with respect to the Lebesgue measure. The integrand depends on the demand CDF $F$ and on the distribution of actions induced by $\pi_{\bm{n}}$. This reformulation separates the policy-dependent component of the out-of-sample evaluation from the demand uncertainty, and it is the starting point for our reduction.

For general policies, the induced distribution of actions can still be highly complex as a function of $F$. We therefore introduce a structural class of policies that is broad enough to cover the policies studied in this paper, while being sufficiently well-behaved to support a tractable worst-case regret analysis.

\begin{definition}[Piecewise-separable policy]\label{def:piecewise}
Fix $\bm{n} \in \mathbb{N}^K$ and a design $\pastI = (x_1,\ldots,x_K) \in [0,1]^K$. Define $x_0 = 0$ and $x_{K+1}=1$, and assume the historical inventory levels are distinct, $0=x_0 \le x_1 < \cdots < x_K \le x_{K+1}=1$.
A policy $\pi_{\bm{n}}$ is \emph{piecewise-separable} (with respect to $\pastI$) if for each $k \in \{0,\ldots,K\}$ there exists a continuous function
\begin{equation*}
P_k^{\pi_{\bm{n}}} : [0,1]^{K} \times [0,1]^{K} \times [0,1] \to [0,1]
\end{equation*}
such that for every demand distribution $F$ and every $z \in [0,1]$,
\begin{equation*}
\Prob_{\bm{D} \sim F^n}\big( \pi_{\bm{n}}(\pastI,\bm{I}) \le z \big)
=
P_k^{\pi_{\bm{n}}}\big( F(\pastI^-), F(\pastI), F(z) \big)
\qquad
\text{for } z \in [x_k, x_{k+1}).
\end{equation*} 
\end{definition}

The defining feature of piecewise-separable policies is a sharp separation between local and global dependence on the demand distribution. On each interval $[x_k,x_{k+1})$, the policy’s action CDF at a point $z$ depends on $F$ only through the local value $F(z)$, while all non-local information about $F$ enters exclusively through the finite-dimensional vectors $F(\pastI^-)$ and $F(\pastI)$ of left limits and CDF values at the historical inventory levels. In particular, once $F(\pastI^-)$ and $F(\pastI)$ are fixed, the dependence of the action distribution on $F$ within the interval is pointwise in $z$. This locality is the structural reason why worst-case regret becomes tractable: conditional on $F(\pastI^-)$ and $F(\pastI)$, the adversary can optimize over the scalar $F(z)$ independently across $z \in [x_k,x_{k+1})$, with cross-interval coupling captured only through monotonicity constraints. This formulation generalizes the separable policies of \citet{besbes2023contextual} and encompasses all policies analyzed in this paper. Importantly, prior formulations do not apply to the BSAA and KM policies, and our generalization will prove critical to derive their performance guarantees.

We now state our finite-dimensional characterization of the regret of piecewise-separable policies.
\begin{theorem}\label{thm:general_reduction}
Let $\pi_{\bm{n}}$ be a piecewise-separable policy with associated functions $P_0^{\pi_{\bm{n}}},\ldots,P_K^{\pi_{\bm{n}}}$. Then, 
\begin{equation*}
\mathsf{Reg}(\pi_{\bm{n}}  \mid \pastI)
=(\cu+\co)
\sup_{\substack{
\bm{f}^{-},\bm{f} \in [0,1]^{K},\ \bm{f}^{+} \in [0,1]^{K+1} \\
0 \le f_0^{+} \le f_1^{-} \le f_1 \le f_1^{+} \le \cdots \le f_K^{-} \le f_K \le f_K^{+} \le 1
}}
\sum_{k=0}^{K} (x_{k+1}-x_k) \cdot \Psi_k^{\pi_{\bm{n}}}(\bm{f}^{-}, \bm{f}, f_k^{+}),
\end{equation*}
where, for each $k \in \{0,\ldots,K\}$ and for all $(\bm{u}^{-},\bm{u},v) \in [0,1]^{K} \times [0,1]^{K} \times [0,1]$, the function $\Psi_k^{\pi_{\bm{n}}}$ is defined by
\begin{equation*}
\Psi_k^{\pi_{\bm{n}}}(\bm{u}^{-},\bm{u},v)
=
\big(1 - P_k^{\pi_{\bm{n}}}(\bm{u}^{-},\bm{u},v)\big)\cdot (v-q)
+ (q - v)_+.
\end{equation*}
If $x_1 = 0$ or $x_K = 1$, then we additionally require that $f_1^{-} = 0$ or $f_K = 1$ be fixed, respectively.
\end{theorem}
Theorem~\ref{thm:general_reduction} yields a substantial simplification of the worst-case regret problem. An infinite-dimensional, and generally non-convex, optimization over demand distributions reduces to a finite-dimensional problem of dimension at most $3K+1$, regardless of the sample size $n$. The optimization variables $(\bm{f}^{-},\bm{f},\bm{f}^{+})$ can be interpreted as free parameters describing the behavior of a candidate worst-case demand CDF immediately to the left of, at, and between the historical inventory levels, thereby encoding the only degrees of freedom that matter for worst-case regret under piecewise-separable policies. This reduction is particularly powerful in censored data regimes where only a small number of distinct inventory levels are used historically, while potentially many demand observations are collected at each such level.

The proof leverages two ingredients. First, Lemma~\ref{lem:integral_form} expresses regret as an integral in which the policy enters only through the distribution of its induced action. Second, piecewise-separability enforces that on each interval $[x_k,x_{k+1})$ the action distribution depends on $F$ only through $F(\pastI^-)$, $F(\pastI)$, and the local value $F(z)$. This permits an adversary to optimize pointwise over each interval, with the global feasibility of these local choices enforced solely by monotonicity constraints across breakpoints.

Finally, Theorem~\ref{thm:general_reduction} provides a powerful sufficient condition for tractability. Quantifying worst-case performance still requires optimizing over adverse demand distributions, but the theorem shows that this worst-case optimization necessarily collapses to a finite-dimensional problem as soon as the policy’s action distribution admits a piecewise-separable representation. From this perspective, the key step in analyzing a new policy under offline censored demand is to characterize its induced distribution of actions and verify piecewise-separability. As a warm-up and first application, we begin by carrying out this exercise for Biased-SAA in \Cref{sec:SAA_reduction}. Importantly, the recipe applies much more broadly, as we will show in \Cref{sec:censoring_aware}.

\subsection{Analysis of Biased SAA}\label{sec:SAA_reduction}

Having reduced the infinite-dimensional problem of computing worst-case regret to a finite-dimensional one in Theorem~\ref{thm:general_reduction}, we now leverage it to characterize the regret of the Biased Sample Average Approximation (BSAA) policy. Recall that BSAA chooses the empirical critical-fractile order quantity based on the empirical distribution $\hat F_{\mathrm{BSAA}}$ of the pooled sales data $\{\ss{k}_i \mid k\in[K],\ i\in[n_k]\}$:
\begin{align*}
\BSAA_{\bm n}(\bm x,\bm I)
\ \coloneqq\
\inf\left\{ u \in [0,1] \mid \hat F_{\mathrm{BSAA}}(u) \geq q \right\},
\end{align*}
where $q=\cu/(\cu+\co) = \cu$ is the critical fractile. 

BSAA is a natural baseline policy in settings with censored demand. It constructs an empirical distribution by treating observed sales as uncensored demand realizations and then applies the classical sample-average-approximation prescription based on this distribution. The policy is appealing because of its simplicity and transparency, and because it relies solely on the historical sales record without introducing additional modeling assumptions. At the same time, its limitation is evident: when stockouts occur, observed sales are right-censored, which can lead to a systematic distortion of the empirical distribution and, consequently, to suboptimal ordering decisions. Our goal is to provide an exact characterization of the worst-case regret of BSAA: the main result of this section (\Cref{thm:BSAA}) reduces the infinite-dimensional adversarial choice of a worst-case demand distribution to a finite collection of $K+2$ \emph{two-dimensional} optimization problems, thereby allowing us to efficiently evaluate BSAA's worst-case regret for arbitrary historical inventory levels $\bm{x}$. In particular, we show that the worst-case distribution for BSAA is always a three-point distribution with mass at 0, 1 and one of the historical inventory levels $\bm{x}$.

The general reduction of \Cref{thm:general_reduction} evaluates regret through the distribution of the action $\pi_{\bm n}(\pastI,\bm I)$ induced by the policy, so we start by analyzing this action distribution. For BSAA this distribution can be characterized explicitly, and the resulting formula exposes the precise way in which worst-case regret depends on the historical inventories $\bm x$ and the allocation of samples $(n_1,\ldots,n_K)$. In order to apply \Cref{thm:general_reduction}, we show that BSAA is indeed a piecewise-separable policy as defined in Definition~\ref{def:piecewise}. Our characterization of its action distribution will make use of Bernstein polynomials defined for every $p \in [0,1]$ as,
\begin{align*}
    B_{r,n}(p)\ \coloneqq\ \sum_{j=r}^n \binom{n}{j}  p^j (1-p)^{ n-j}\,,
\end{align*}
and, which capture the probability of a $\text{Binomial}(n,p)$ random variable being greater than or equal to~$r$. By convention, we set $B_{r,n}(p) = 1$ for all $p \in  [0,1]$ whenever $r < 0$.

\begin{lemma}[Action distribution of $\BSAA$]\label{lem:action_dist_BSAA}
For any demand distribution $F$ and any $z \in [x_k, x_{k+1})$ (for some $0 \leq k \leq K$), the probability that $\BSAA_{\bm n}$ orders at most $z$ is 
\begin{equation*}
  \Prob_{\bm{D} \sim F^n}\big( \BSAA_{\bm{n}}(\pastI,\bm{I}) \le z \big) =  B_{\lceil q n \rceil - \sigma_k,   n - \sigma_k} \big(F(z)\big) \,,
\end{equation*} 
where $\sigma_k \coloneqq \sum_{\ell=1}^k n_\ell$.
\end{lemma}

BSAA, by definition, orders the $\lceil qn \rceil$-th smallest historical sales quantity, so it orders at most $z$ exactly when at least $\lceil qn \rceil$ of the $n$ sales are at most $z$. For $z \in [x_k, x_{k+1})$, the $\sigma_k$ samples collected under the inventory levels $x_1, \ldots, x_k \le z$ are censored at or below $z$, so their sales are at most $z$ regardless of the demand. Each of the remaining $n - \sigma_k$ samples was collected under an inventory level exceeding $z$, so its sale is at most $z$ precisely when its demand is, which happens independently with probability $F(z)$. The order is therefore at most $z$ exactly when at least $\lceil qn \rceil - \sigma_k$ of these $n - \sigma_k$ independent trials succeed, which is the binomial tail probability in the lemma.

Lemma~\ref{lem:action_dist_BSAA} has precisely the structure described in Definition~\ref{def:piecewise}: on each interval $[x_k,x_{k+1})$, the action CDF depends on $F$ only through the local value $F(z)$ on $z \in [x_k, x_{k+1})$. In fact, the action distribution of BSAA possesses stronger structure: BSAA is piecewise-separable with no ``global'' dependence on $F(\pastI^-)$ and $F(\pastI)$, i.e., the action distribution on the interval $[x_k,x_{k+1})$ does not depend on the value of $F$ outside this interval. The following characterization of worst-case regret follows immediately from an application of \Cref{thm:general_reduction}.

\begin{corollary}\label{cor:BSAA-piecewise-linear}
    The BSAA policy $\BSAA_{\bm n}$ is a piecewise-separable policy with $        P_k^{\BSAA_{\bm n}}(\bm{u}^{-},\bm{u},v)\ =\ B_{\lceil q n \rceil - \sigma_k,   n - \sigma_k}(v)$,
    for all $0 \leq k \leq K$, where $\sigma_k = \sum_{\ell=1}^k n_\ell$. Therefore, 
    \begin{equation*}
        \mathsf{Reg}(\BSAA_{\bm{n}}  \mid \pastI)
        =
        \sup_{\substack{
        \bm{f}^{+} \in [0,1]^{K+1} \\
        0 \le f_0^{+} \le f_1^{+} \le \cdots \le f_K^{+} \le 1
        }}
        \sum_{k=0}^{K} (x_{k+1}-x_k) \cdot \psibsaa_k(f_k^{+})\,,
    \end{equation*}
    where the function $\Psi^{\mathrm BSAA}_k: [0,1] \to \mathbb R$ is given by
        $\psibsaa_k(v)
        =
        \big(1 - B_{\lceil q n \rceil - \sigma_k,   n - \sigma_k} \big(v\big) \big)\cdot (v-q)
        + (q - v)_+.$
\end{corollary}

Corollary~\ref{cor:BSAA-piecewise-linear} reduces the infinite-dimensional adversarial choice of $F$ to the choice of a nondecreasing sequence $\bm f^+=(f_0^+,\ldots,f_K^+)$, which can be interpreted as the value of a candidate worst-case CDF on each interval $[x_k,x_{k+1})$. The only coupling across intervals is the monotonicity constraint; absent this constraint, each term would be maximized independently. 

In fact, the coupling exhibits structure that simplifies the problem further. Each $\psibsaa_k$ can be shown to be unimodal on either side of the critical fractile, with modes that decrease in $k$, so the monotonicity constraint on $\{f_k^+\}$ binds as tightly as it can. \Cref{thm:BSAA} exploits this to reduce the search to $K+2$ problems in two scalars each.

\begin{theorem}\label{thm:BSAA}
For every $\numS \geq 1$, and every $\pastI \in [0,1]^K$, we have
\begin{equation*}
\mathsf{Reg}(\BSAA_{\bm n} \mid \pastI) = \max_{-1 \leq i \leq K}  \sup_{\substack{
       v \in [0,q], \\ w \in [q,1]}} \quad
        \sum_{k=0}^i (x_{k+1}  - x_k) \cdot \psibsaa_k(v) + \sum_{k=i+1}^K (x_{k+1}  - x_k) \cdot \psibsaa_k(w)\,,
\end{equation*}
where $\sigma_k = \sum_{\ell=1}^k n_\ell$, the function $\Psi^{\mathrm BSAA}_k: [0,1] \to \mathbb R$ is given by
    \begin{equation*}
        \psibsaa_k(v)
        =
        \big(1 - B_{\lceil q n \rceil - \sigma_k,   n - \sigma_k} \big(v\big) \big)\cdot (v-q)
        + (q - v)_{+}.
    \end{equation*}
\end{theorem}

A useful interpretation of \Cref{thm:BSAA} is that the worst-case regret of BSAA can be realized by a distribution supported on at most three points: $0$, a historical censoring threshold, and $1$. To see this, take a crossing index $i\in\{-1,0,\ldots,K\}$ and CDF levels $(v,w)$ with $v\in[0,q]$ and $w\in[q,1]$, as in the theorem. The corresponding distribution is
\[
F_{i,v,w} \;=\; v\,\delta_{0} \;+\; (w-v)\,\delta_{x_{i+1}} \;+\; (1-w)\,\delta_{1}.
\]
Its CDF equals $v$ on $[0,x_{i+1})$ and $w$ on $[x_{i+1},1)$, and equals $1$ at $1$. Thus, it has CDF level $v$ on the intervals indexed by $k\le i$ and level $w$ on those indexed by $k\ge i+1$, exactly as in \Cref{thm:BSAA}. If $x_{i+1}$ equals $0$ or $1$, the two masses combine and the distribution is supported on $\{0,1\}$.

Therefore, Theorem~\ref{thm:BSAA} provides a computationally tractable procedure to evaluate the worst-case regret of BSAA for \emph{any} sample size $\bm{n}$ and \emph{any} configuration of historical inventory $\pastI=(x_1,\ldots,x_K)$.
Indeed, the worst-case regret can be evaluated efficiently via a finite set of $K+2$ two-dimensional problems. This allows us to efficiently evaluate the \emph{value of censored sales data} for the BSAA policy. We now leverage this reduction to study the performance guarantees achievable by a decision-maker who uses BSAA, which is a natural choice when relying on a coarse point-of-sale system that only stores realized sales numbers and does not record stockout information, thereby precluding the use of more sophisticated policies.

\paragraph{Order-up-to Setting.} We consider a practical scenario in which a decision-maker has primarily operated at a single inventory level $x \in [0,1]$. Such a setting naturally arises, for instance, when the decision-maker follows an order-up-to policy for an extended period of time and gathers demand observations. While this approach facilitates stable operations, it typically results in censored sales data whenever realized demand exceeds the chosen inventory level.
To mitigate censoring and improve information about the underlying demand distribution, the decision-maker may occasionally place a larger order. We model this exploratory behavior by allowing the decision-maker to select the maximal inventory level $1$, corresponding to the upper bound of the demand support, for a subset of periods. Formally, we assume that the decision-maker has access to a total of $n$ sales observations, of which $n-m$ were generated under inventory  $x$, and the remaining $m$ were generated under inventory level $1$. The pair $(n,m)$ captures both the overall amount of available data and the intensity of exploration. When $m=0$, all observations are censored at level $x$, whereas larger values of $m$ correspond to increasingly informative data due to reduced censoring. As a benchmark, we also include the worst-case regret of the Sample Average Approximation (SAA) policy with access to $n$ i.i.d.\ samples of uncensored demand.

\begin{figure}[h!]
\centering
\begingroup
\pgfplotsset{
  paper line/.style={line width=1.06666667bp,solid,line cap=round,line join=round,mark size=1.86666667bp,mark options={solid,line width=0.73333333bp}},
  paper dashed/.style={dash pattern=on 4.26666667bp off 2.40000000bp},
  legend image post style={line width=0.8bp,mark size=1.4bp,mark options={line width=0.55bp}},
  paper benchmark legend/.style={legend image post style={dash pattern=on 3.2bp off 1.8bp}},
  paper hollow/.style={mark options={solid,fill=white,line width=0.73333333bp},legend image post style={mark options={solid,fill=white,line width=0.55bp}}},
  m0/.style={paper line,paperBlue,mark=*},
  m1/.style={paper line,paperRed,mark=square*},
  m2/.style={paper line,paperBrown,mark=triangle*},
  m5/.style={paper line,paperViolet,mark=diamond*},
  m10/.style={paper line,paperGreen,mark=*},
  fullinfo/.style={paper line,paperBlack,paper dashed,mark=none,paper benchmark legend},
  paper km/.style={paper line,paperBlue,mark=*},
  paper lp/.style={paper line,paperViolet,mark=diamond*},
  paper benchmark/.style={paper line,paperBlack,paper dashed,mark=none,paper benchmark legend}
}
\centering

\begin{tabular}{@{}lll@{}}
\paperlegend{paperBlue,mark=*}{$m=0$}&\paperlegend{paperRed,mark=square*}{$m=1$}&\paperlegend{paperBrown,mark=triangle*}{$m=2$}\\
\paperlegend{paperViolet,mark=diamond*}{$m=5$}&\paperlegend{paperGreen,mark=*}{$m=10$}&\paperlegend{paperBlack,mark=none,dash pattern=on 3.2bp off 1.8bp}{SAA (uncensored data)}
\end{tabular}\par

\vspace{0.4em}
\subfigure[$x=0.7$]{
\begin{tikzpicture}[scale=.75]
\begin{axis}[
  width=0.43\textwidth,
  height=0.35\textwidth,
  xlabel={$n$},
  ylabel={worst-case regret},
  ymin=0,
  ymax=.3,
  grid=both,
  scaled y ticks=false,
  yticklabel style={
    /pgf/number format/fixed,
    /pgf/number format/precision=2
  },
  every axis plot/.append style={paper line},
  unbounded coords=jump,
  legend columns=3,
  legend style={draw=none},
]
\addplot[m0] table[col sep=comma,x=n,y=wc_regret_0_0.7]{Data/figure1_bsaa_pgfplots.csv};

\addplot[m1] table[col sep=comma,x=n,y=wc_regret_1_0.7]{Data/figure1_bsaa_pgfplots.csv};

\addplot[m2] table[col sep=comma,x=n,y=wc_regret_2_0.7]{Data/figure1_bsaa_pgfplots.csv};

\addplot[m5] table[col sep=comma,x=n,y=wc_regret_5_0.7]{Data/figure1_bsaa_pgfplots.csv};

\addplot[m10] table[col sep=comma,x=n,y=wc_regret_10_0.7]{Data/figure1_bsaa_pgfplots.csv};

\addplot[fullinfo] table[col sep=comma,x=n,y=wc_regret_full_info_0.7]{Data/figure1_bsaa_pgfplots.csv};

\end{axis}
\end{tikzpicture}
}
\hfill
\subfigure[$x=0.8$]{
\begin{tikzpicture}[scale=.75]
\begin{axis}[
  width=0.43\textwidth,
  height=0.35\textwidth,
  xlabel={$n$},
  ymin=0,
  ymax=.3,
  grid=both,
  scaled y ticks=false,
  yticklabel style={
    /pgf/number format/fixed,
    /pgf/number format/precision=2
  },
  unbounded coords=jump,
    every axis plot/.append style={paper line},
]
\addplot[m0] table[col sep=comma,x=n,y=wc_regret_0_0.8]{Data/figure1_bsaa_pgfplots.csv};
\addplot[m1] table[col sep=comma,x=n,y=wc_regret_1_0.8]{Data/figure1_bsaa_pgfplots.csv};
\addplot[m2] table[col sep=comma,x=n,y=wc_regret_2_0.8]{Data/figure1_bsaa_pgfplots.csv};
\addplot[m5] table[col sep=comma,x=n,y=wc_regret_5_0.8]{Data/figure1_bsaa_pgfplots.csv};
\addplot[m10] table[col sep=comma,x=n,y=wc_regret_10_0.8]{Data/figure1_bsaa_pgfplots.csv};
\addplot[fullinfo] table[col sep=comma,x=n,y=wc_regret_full_info_0.8]{Data/figure1_bsaa_pgfplots.csv};
\end{axis}
\end{tikzpicture}
}
\subfigure[$x=0.9$]{
\begin{tikzpicture}[scale=.75]
\begin{axis}[
  width=0.43\textwidth,
  height=0.35\textwidth,
  xlabel={$n$},
  ymin=0,
  ymax=.3,
  grid=both,
  scaled y ticks=false,
  yticklabel style={
    /pgf/number format/fixed,
    /pgf/number format/precision=2
  },
  unbounded coords=jump,
    every axis plot/.append style={paper line},
]
\addplot[m0] table[col sep=comma,x=n,y=wc_regret_0_0.9]{Data/figure1_bsaa_pgfplots.csv};
\addplot[m1] table[col sep=comma,x=n,y=wc_regret_1_0.9]{Data/figure1_bsaa_pgfplots.csv};
\addplot[m2] table[col sep=comma,x=n,y=wc_regret_2_0.9]{Data/figure1_bsaa_pgfplots.csv};
\addplot[m5] table[col sep=comma,x=n,y=wc_regret_5_0.9]{Data/figure1_bsaa_pgfplots.csv};
\addplot[m10] table[col sep=comma,x=n,y=wc_regret_10_0.9]{Data/figure1_bsaa_pgfplots.csv};
\addplot[fullinfo] table[col sep=comma,x=n,y=wc_regret_full_info_0.9]{Data/figure1_bsaa_pgfplots.csv};
\end{axis}
\end{tikzpicture}
}

\endgroup
\caption{BSAA policy, $q = 0.8$: effect of sample size $n$.}
\label{fig:BSAA_worst_case}
\end{figure}

\Cref{fig:BSAA_worst_case} illustrates how the BSAA policy, which treats both censored and uncensored sales as demand, performs in this Order-up-to setting. In particular, we find that while a small number of uncensored observations can reduce worst-case regret for small sample sizes, their influence diminishes rapidly as additional censored data are aggregated. Indeed, as the sample size of censored data increases, the bias induced by treating censored sales as demand becomes dominant, and the worst-case regret of BSAA converges to the same level across all values of $m$. This dilution effect has important operational implications. In regimes with substantial censoring, such as $x=0.7$, aggregating additional censored data can be actively harmful. For small values of $m$, the worst-case regret of BSAA is often lower when the policy relies almost exclusively on uncensored observations, and increases as more censored samples are added. In these settings, censored data overwhelm the limited high-quality information provided by uncensored samples and degrade overall performance. When the censoring point is sufficiently high, for example $x=0.9$, censored observations are less distorted and may lead to moderate performance improvements. Even in this case, however, the gains from aggregating censored data are small relative to the potential deterioration that occurs once censored samples dominate the dataset.

Taken together, these results suggest that when operating with a PoS system that cannot capture censoring information, the decision-maker should avoid myopically aggregating large volumes of censored sales data. Instead, performance is driven primarily by a small number of high-quality, uncensored observations. In such settings, deliberately prioritizing a limited amount of uncensored data can be far more effective than collecting large datasets dominated by censored sales. In fact, our characterization in \Cref{thm:BSAA} also makes tractable the associated design problem, in which the decision-maker chooses which inventory levels to explore, and how many observations to collect at each, subject to a budget on total inventory deployed. We show in \Cref{appendix:opt-bsaa-inventory} that this problem reduces to a finite collection of linear programs. Solving it numerically for a range of budgets, we find that the optimal design spends the entire budget on uncensored samples, reinforcing the message that fewer uncensored observations are worth more to BSAA than many censored ones. More fundamentally, this behavior of BSAA highlights the need to properly account for censoring, as done by censoring-aware policies such as KM and those proposed by \citet{fan2022sample} and \citet{hssaine2024data}, which will be our focus in the next section.

\section{Exact Performance of Censoring-Aware Policies}\label{sec:censoring_aware}

The analysis of Biased-SAA in the previous section relied on the simple structure of its empirical distribution: each observed sale receives the same weight $1/n$, independently of the data. In contrast, policies that account for censoring typically assign different weights to different observations. These weights depend on which demands are observed, which are censored, and where the observed demands lie relative to the historical inventory levels. The resulting dependence on the data makes their performance more difficult to analyze. Our goal in this section is to develop a framework to analyze common policies in the literature involving more complex dependency in the data: in particular we aim at characterizing the worst-case performance of the classical Kaplan--Meier policy, along with more recent policies developed by \citet{fan2022sample} and \citet{hssaine2024data}.

We begin with the Kaplan--Meier weight construction, then introduce \emph{censoring-aware sample average approximation} (C-SAA) policies and prove that every such policy is piecewise-separable. We apply the general reduction to Kaplan--Meier and use it to evaluate the value of exploration. Finally, we extend the finite-sample analysis to the policies of \citet{fan2022sample} and \citet{hssaine2024data}.

Recall from \Cref{sec:model} that the Kaplan--Meier policy orders the lower $q$-quantile of the Kaplan--Meier estimate \citep{kaplan1958nonparametric}. Write $(Y_i,\zeta_i)_{i=1}^n$ for the sales and their corresponding uncensoring indicators, ordered by increasing sales with uncensored observations preceding censored ones at a tie. In this notation, the estimate is
\begin{equation}
\KMF(z)=
\begin{cases}
1-\displaystyle\prod_{i:Y_i\le z}
\left(\dfrac{n-i}{n-i+1}\right)^{\zeta_i},&z<1,\\[1ex]
1,&z=1,
\end{cases}
\end{equation}
where factors with $\zeta_i=0$ are taken to be one. The Kaplan--Meier policy selects the lower $q$-quantile:
$\KM_{\bm n}(\pastI,\bm I)
=\inf\big\{z\in[0,1]:\KMF(z)\ge q\big\}.$

The weights have a useful interpretation. Start by assigning mass $1/n$ to each sales observation and process the observations in increasing order. An uncensored observation retains its current mass. At each censoring level, first retain the masses of uncensored observations at that level, then divide the combined mass of the censored observations equally among observations with strictly larger sales. If none remain, the mass is placed at $1$. Each redistribution changes the masses available, creating an intricate relation between the observed data and the weights of the empirical distribution.

\begin{example}\label{ex:ca_km_weights}
Consider five observations whose ordered sales and uncensoring indicators are
\[
\begin{array}{c|ccccc}
i&1&2&3&4&5\\ \hline
Y_i&0.1&0.3&0.5&0.7&0.9\\
\zeta_i&1&0&1&0&1
\end{array}
\]
The censored observation at $0.3$ transfers its mass equally to the three observations to its right, increasing each of their masses to $1/5+1/15=4/15$. The observation at $0.5$ retains this mass. The censored observation at $0.7$ then transfers its entire mass, now $4/15$, to the observation at $0.9$. The final estimate is therefore
\[
\KMF
=\frac15\delta_{0.1}+\frac4{15}\delta_{0.5}+\frac8{15}\delta_{0.9}.
\]
In particular, the mass transferred at $0.7$ depends on the earlier censoring at $0.3$ and on the observations between these levels. For $q=1/2$, the KM policy orders $0.9$, whereas Biased-SAA orders~$0.5$.
\end{example}

\Cref{ex:ca_km_weights} illustrates the difficulty in analyzing Kaplan--Meier. The weight of an observed demand depends on earlier censoring and on how many observations remain at each stage. Nevertheless, these weights have a specific structure: their values are determined by comparisons of the observations with the historical inventory levels. The exact positions of uncensored demands between consecutive inventory levels determine where the estimator jumps, but do not affect the weights assigned to those jumps. We next show that this distinction is sufficient for an exact worst-case analysis.

\subsection{Censoring-Aware Sample Average Approximation}
We now introduce the class of policies to which our analysis of Kaplan--Meier will apply.
Define the tuple of demand-inventory comparisons
\begin{equation}
\label{eq:ca_comparisons}
\bm H_{\pastI}
=\left(
\mathds{1}\{D_i^{(k)}<x_j\},\,
\mathds{1}\{D_i^{(k)}\le x_j\}
\;\middle|\;
k\in[K],\ i\in[n_k],\ j\in \{1,\ldots,k\}
\right).
\end{equation}
For each sample collected at inventory $x_k$, the tuple $\bm H_{\pastI}$ records its strict and weak comparisons with every historical inventory $x_j\le x_k$.
Although written in terms of demand, every entry of $\bm H_{\pastI}$ is observable. For $j<k$, the comparisons can be made directly using $S_i^{(k)}$. For $j=k$, the strict comparison is $\mathds{1}\{S_i^{(k)}<x_k\}$ and the weak comparison is the recorded indicator $\delta_i^{(k)}$. Keeping both comparisons allows the weights to distinguish a demand exactly equal to an inventory level from one strictly below it. This distinction matters when the demand distribution has atoms.

\begin{definition}[Censoring-aware SAA]\label{def:censoring_aware_saa}
A policy $\pi_{\bm n}$ is a \emph{censoring-aware SAA policy} if there are nonnegative weights $w_i^{(k)}(\bm H_{\pastI})$ and a positive target $\tau(\bm H_{\pastI})$, all fixed functions of $\bm H_{\pastI}$, such that the policy chooses the inventory
\begin{equation*}
    \pi_{\bm n}(\pastI,\bm I)
=\inf\big\{z\in[0,1):\widehat H_{\bm n}(z)
\ge\tau(\bm H_{\pastI})\big\},
\end{equation*}
with value $1$ if the set is empty, and where
\begin{equation*}
    \widehat H_{\bm n}(z)
=\sum_{k=1}^K\sum_{i=1}^{n_k}
w_i^{(k)}(\bm H_{\pastI})\mathds{1}\{S_i^{(k)}\le z\},
\end{equation*}
is a weighted empirical cumulative function of the observed sales. The weight and target functions may depend on $\pastI$, $\bm n$, and the policy's parameters, but not on the unknown distribution $F$.
\end{definition}

The definition permits the weights and target to depend jointly on all the comparisons in $\bm H_{\pastI}$. In particular, it allows products, ratios, and successive updates based on counts of censored and uncensored observations. The restriction is that these weights use the threshold comparisons in \eqref{eq:ca_comparisons}; they cannot be arbitrary functions of the numerical sales values. Once the comparisons are fixed, the observations enter the score only through whether their sales are at most $z$.

Biased-SAA is recovered by taking $w_i^{(k)}=1/n$ and $\tau=q$. More generally, if the weights sum to at most one and $\tau=q$, the policy minimizes the newsvendor cost under the estimated distribution with CDF
\begin{equation*}
\widehat F(z)=
\begin{cases}
\widehat H_{\bm n}(z),&0\le z<1,\\
1,&z=1.
\end{cases}
\end{equation*}
If the weights sum to one, $\widehat F=\widehat H_{\bm n}$ on all of $[0,1]$.

Many policies in the literature solve the newsvendor problem under an estimated demand distribution. For such policies, \Cref{def:censoring_aware_saa} gives a direct membership criterion: check whether the estimated weights depend only on the comparisons in $\bm H_{\pastI}$.

Our next result connects this simple description of a policy in terms of its observed data to piecewise-separability (\Cref{def:piecewise}), which describes the distribution of actions of a policy.

\begin{proposition}\label{prop:ca_piecewise}
Every censoring-aware SAA policy is piecewise-separable. 
\end{proposition}

The intuition behind \Cref{prop:ca_piecewise} is that for any action $z \in [0,1]$ moving a demand without changing its comparisons with the historical inventories or with $z$ changes neither the weights nor whether the policy orders at most $z$. These comparisons partition the demand support into finitely many intervals involving only the $x_i$s or $z$ as endpoint. Hence, the probabilities that a demand sample fall into one of these intervals are determined by $F(\pastI^-)$, $F(\pastI)$, and $F(z)$. Consequently, these CDF values alone determine the probability that the policy orders at most $z$.

Applying \Cref{thm:general_reduction} to the piecewise-separable representation in \Cref{prop:ca_piecewise} implies the following corollary.

\begin{corollary}\label{cor:ca_reduction}
For every censoring-aware SAA policy, every sample-size vector $\bm n$, and every historical inventory design $\pastI$, the finite-dimensional characterization of worst-case regret in \Cref{thm:general_reduction} applies. 
\end{corollary}

\Cref{cor:ca_reduction} gives a direct route to analyzing a policy: verify that its weights and target depend only on the demand-inventory comparisons, then apply the general reduction. We follow this route for Kaplan--Meier and subsequently examine other policies from the literature.

\subsection{Exact Worst-Case Regret of the Kaplan--Meier Policy}\label{sec:ca_km_reduction}

We now show that Kaplan--Meier belongs to the class of censoring-aware SAA policies. Although the estimator is constructed observation by observation, its weights change only when censoring occurs. At each historical inventory $x_j$, two counts describe the effect on all later observations: the number $C_j$ of observations censored at $x_j$, and the number $R_j$ of sales strictly greater than $x_j$.

Immediately before the censored observations at $x_j$ are processed, every observation still awaiting processing carries the same weight. The $C_j$ censored observations transfer their combined weight equally to the $R_j$ observations to their right. When $R_j>0$, this multiplies each recipient's weight by $1+C_j/R_j$.  If no observations remain to the right, the censored mass is placed at $1$.

The following lemma formalizes this representation and shows that the resulting weights use precisely the comparisons allowed in \Cref{def:censoring_aware_saa}.

\begin{lemma}\label{lem:ca_km_saa}
Fix a sample-size vector $\bm n\in\mathbb N^K$ and a historical inventory design $\pastI$. For a possible value $\bm h$ of the comparison tuple $\bm H_{\pastI}$ in \eqref{eq:ca_comparisons}, let $h_{i,j}^{(r)}$ denote the entry corresponding to $\mathds{1}\{D_i^{(r)}\le x_j\}$, for $r\in[K]$, $i\in[n_r]$, and $j\in[r]$. Define, for $j\in[K]$,
\begin{equation}\label{eq:ca_km_counts}
C_j(\bm h)=\sum_{i=1}^{n_j}\big(1-h_{i,j}^{(j)}\big),
\qquad
R_j(\bm h)=\sum_{r=j+1}^K\sum_{i=1}^{n_r}
\big(1-h_{i,j}^{(r)}\big).
\end{equation}
The Kaplan--Meier policy is a censoring-aware SAA policy with target $\tau(\bm h)=q$ and weights
\begin{equation}
\label{eq:ca_km_weights}
 w_i^{(r)}(\bm h)
 =\frac{h_{i,r}^{(r)}}{n}
 \prod_{\substack{1\le j<r\\h_{i,j}^{(r)}=0}}
 \left(1+\frac{C_j(\bm h)}{R_j(\bm h)}\right).
\end{equation}
\end{lemma}

\begin{samepage}
\Cref{lem:ca_km_saa} and \Cref{prop:ca_piecewise} establish that Kaplan--Meier is piecewise-separable, so the general regret reduction applies. For Kaplan--Meier, the action-CDF functions do not require the left-limit argument in \Cref{def:piecewise}. We therefore write $P_k^{\KM_{\bm n}}:[0,1]^K\times[0,1]\to[0,1]$ for continuous functions satisfying
\[
\Prob_{\bm D\sim F^n}\big(\KM_{\bm n}(\pastI,\bm I)\le z\big)
=P_k^{\KM_{\bm n}}\big(F(\pastI),F(z)\big),
\qquad z\in[x_k,x_{k+1}).
\]
As shown in the appendix proof, $P_k^{\KM_{\bm n}}(\bm u,v)$ can be chosen to depend only on $u_1,\ldots,u_k,v$.
\end{samepage}

\begin{theorem}[Worst-case regret of the Kaplan--Meier policy]\label{thm:ca_km_reduction}
For every sample-size vector $\bm n\in\mathbb N^K$ and historical inventory design $\pastI$,
\begin{equation*}
\reg(\KM_{\bm n}\mid\pastI)
=
\max_{\bm f,\bm f^+}
\sum_{k=0}^K(x_{k+1}-x_k)
\Psi_k^{\KM_{\bm n}}\big(\bm f,f_k^+\big),
\end{equation*}
where the maximum is over $\bm f\in[0,1]^K$ and $\bm f^+\in[0,1]^{K+1}$ satisfying
\[
0\le f_0^+\le f_1\le f_1^+\le\cdots\le f_K\le f_K^+\le1,
\]
and, for $(\bm u,v)\in[0,1]^K\times[0,1]$ we have,
$\Psi_k^{\KM_{\bm n}}(\bm u,v)
=\big(1-P_k^{\KM_{\bm n}}(\bm u,v)\big)(v-q)+(q-v)_+.$
\end{theorem}

\Cref{thm:ca_km_reduction} yields an \textit{exact} characterization of the worst-case regret of the Kaplan-Meier policy that holds for \emph{any} finite sample size and for \emph{any} past inventory levels $\bm{x}$. Indeed, it provides an explicit reduction of the original worst-case evaluation problem to a finite-dimensional optimization problem whose dimension depends only on the number of distinct historical inventory levels (and not on the sample size $n$). As a consequence, the exact worst-case regret of the Kaplan-Meier policy can be efficiently computed in finite samples, including in small and moderate data regimes.

To the best of our knowledge, such a finite-sample regret characterization has not been available for Kaplan-Meier based inventory policies. Existing work that applies Kaplan-Meier techniques in inventory control primarily establishes consistency or asymptotic optimality, see, for example, \citep{huh2011adaptive} or the discussion in \citep{fan2022sample}, and does not provide finite-sample guarantees. This limitation is inherent to the structure of the problem: the Kaplan-Meier estimator induces decision rules that depend on the entire history of censoring, so that evaluating regret requires controlling a nonlinear and path-dependent functional of heterogeneously censored observations. In the fixed-design censoring regime considered here, this dependence lies outside the scope of standard concentration-based analyses\footnote{The limitation of classical statistical learning methods to analyze setting with fixed design choices similarly motivated works on the uncensored contextual newsvendor problem and is an important practical issue \citep{qi2022distributionally}.}. While the survival-analysis literature derives nonasymptotic deviation inequalities for the Kaplan-Meier estimator under classical random-censoring assumptions (e.g., \citealp{gill1983large,csorgo1983rate,bitouze1999dkm}), those results do not apply to inventory settings, where censoring is induced by past inventory decisions and therefore cannot be assumed to be random. \Cref{thm:ca_km_reduction} shows that, despite these challenges, the Kaplan-Meier policy admits sufficient structural regularity under fixed censoring to allow an exact finite-sample worst-case regret characterization.

We also note that $\bm f=(f_1,\ldots,f_K)$ records the CDF values at the historical inventory levels, while $\bm f^+=(f_0^+,\ldots,f_K^+)$ records the CDF levels selected between those inventories. The left-limit vector $\bm f^-$ can be eliminated because Kaplan--Meier's weights depend only on weak comparisons between demand and inventory, without requiring strict comparisons. Eliminating this unused vector reduces the number of optimization variables from $3K+1$ to $2K+1$. We use this reduction in the subsequent analysis to evaluate the value of censored observations and targeted exploration.

\subsection{On the Value of Exploration for Kaplan-Meier}

We return to the order-up-to setting of \Cref{sec:SAA_reduction}, with $n-m$ observations collected at inventory $x\in[0,1]$ and $m$ fully observed exploratory samples collected at inventory $1$. For each $(x,n,m)$, we use \Cref{thm:ca_km_reduction} to compute the KM policy's exact worst-case regret.

We present in \Cref{fig:ca_km_exploration} the worst-case performance of the Kaplan-Meier policy as a function of $n$ for different values of $x$ and $m$. As a benchmark, we also plot the worst-case performance of the Sample Average Approximation algorithm which has access to $n$ i.i.d.\ samples of \textit{uncensored} demand. This is the full-information benchmark, corresponding to collecting all $n$ observations at inventory level $1$.

\begin{figure}[h!]
\centering
\begingroup
\pgfplotsset{
  paper line/.style={line width=1.06666667bp,solid,line cap=round,line join=round,mark size=1.86666667bp,mark options={solid,line width=0.73333333bp}},
  paper dashed/.style={dash pattern=on 4.26666667bp off 2.40000000bp},
  legend image post style={line width=0.8bp,mark size=1.4bp,mark options={line width=0.55bp}},
  paper benchmark legend/.style={legend image post style={dash pattern=on 3.2bp off 1.8bp}},
  paper hollow/.style={mark options={solid,fill=white,line width=0.73333333bp},legend image post style={mark options={solid,fill=white,line width=0.55bp}}},
  m0/.style={paper line,paperBlue,mark=*},
  m1/.style={paper line,paperRed,mark=square*},
  m2/.style={paper line,paperBrown,mark=triangle*},
  m5/.style={paper line,paperViolet,mark=diamond*},
  m10/.style={paper line,paperGreen,mark=*},
  fullinfo/.style={paper line,paperBlack,paper dashed,mark=none,paper benchmark legend},
  paper km/.style={paper line,paperBlue,mark=*},
  paper lp/.style={paper line,paperViolet,mark=diamond*},
  paper benchmark/.style={paper line,paperBlack,paper dashed,mark=none,paper benchmark legend}
}
\centering

\begin{tabular}{@{}lll@{}}
\paperlegend{paperBlue,mark=*}{$m=0$}&\paperlegend{paperRed,mark=square*}{$m=1$}&\paperlegend{paperBrown,mark=triangle*}{$m=2$}\\
\paperlegend{paperViolet,mark=diamond*}{$m=5$}&\paperlegend{paperGreen,mark=*}{$m=10$}&\paperlegend{paperBlack,mark=none,dash pattern=on 3.2bp off 1.8bp}{SAA (uncensored data)}
\end{tabular}\par

\vspace{0.4em}
\subfigure[$x=0.7$]{
\begin{tikzpicture}[scale=.75]
\begin{axis}[
  width=0.43\textwidth,
  height=0.35\textwidth,
  xlabel={$n$},
  ylabel={worst-case regret},
  ymin=0,
  ymax=.1,
  grid=both,
  scaled y ticks=false,
  yticklabel style={
    /pgf/number format/fixed,
    /pgf/number format/precision=2
  },
  every axis plot/.append style={paper line},
  unbounded coords=jump,
  legend columns=3,
  legend style={draw=none},
]
\addplot[m0] table[col sep=comma,x=n,y=wc_regret_0_0.7]{Data/figure2_km_pgfplots.csv};

\addplot[m1] table[col sep=comma,x=n,y=wc_regret_1_0.7]{Data/figure2_km_pgfplots.csv};

\addplot[m2] table[col sep=comma,x=n,y=wc_regret_2_0.7]{Data/figure2_km_pgfplots.csv};

\addplot[m5] table[col sep=comma,x=n,y=wc_regret_5_0.7]{Data/figure2_km_pgfplots.csv};

\addplot[m10] table[col sep=comma,x=n,y=wc_regret_10_0.7]{Data/figure2_km_pgfplots.csv};

\addplot[fullinfo] table[col sep=comma,x=n,y=wc_regret_full_info_0.7]{Data/figure2_km_pgfplots.csv};

\end{axis}
\end{tikzpicture}
}
\hfill
\subfigure[$x=0.8$]{
\begin{tikzpicture}[scale=.75]
\begin{axis}[
  width=0.43\textwidth,
  height=0.35\textwidth,
  xlabel={$n$},
  ymin=0,
  ymax=.1,
  grid=both,
  scaled y ticks=false,
  yticklabel style={
    /pgf/number format/fixed,
    /pgf/number format/precision=2
  },
  unbounded coords=jump,
    every axis plot/.append style={paper line},
]
\addplot[m0] table[col sep=comma,x=n,y=wc_regret_0_0.8]{Data/figure2_km_pgfplots.csv};
\addplot[m1] table[col sep=comma,x=n,y=wc_regret_1_0.8]{Data/figure2_km_pgfplots.csv};
\addplot[m2] table[col sep=comma,x=n,y=wc_regret_2_0.8]{Data/figure2_km_pgfplots.csv};
\addplot[m5] table[col sep=comma,x=n,y=wc_regret_5_0.8]{Data/figure2_km_pgfplots.csv};
\addplot[m10] table[col sep=comma,x=n,y=wc_regret_10_0.8]{Data/figure2_km_pgfplots.csv};
\addplot[fullinfo] table[col sep=comma,x=n,y=wc_regret_full_info_0.8]{Data/figure2_km_pgfplots.csv};
\end{axis}
\end{tikzpicture}
}
\subfigure[$x=0.9$]{
\begin{tikzpicture}[scale=.75]
\begin{axis}[
  width=0.43\textwidth,
  height=0.35\textwidth,
  xlabel={$n$},
  ymin=0,
  ymax=.1,
  grid=both,
  scaled y ticks=false,
  yticklabel style={
    /pgf/number format/fixed,
    /pgf/number format/precision=2
  },
  unbounded coords=jump,
    every axis plot/.append style={paper line},
]
\addplot[m0] table[col sep=comma,x=n,y=wc_regret_0_0.9]{Data/figure2_km_pgfplots.csv};
\addplot[m1] table[col sep=comma,x=n,y=wc_regret_1_0.9]{Data/figure2_km_pgfplots.csv};
\addplot[m2] table[col sep=comma,x=n,y=wc_regret_2_0.9]{Data/figure2_km_pgfplots.csv};
\addplot[m5] table[col sep=comma,x=n,y=wc_regret_5_0.9]{Data/figure2_km_pgfplots.csv};
\addplot[m10] table[col sep=comma,x=n,y=wc_regret_10_0.9]{Data/figure2_km_pgfplots.csv};
\addplot[fullinfo] table[col sep=comma,x=n,y=wc_regret_full_info_0.9]{Data/figure2_km_pgfplots.csv};
\end{axis}
\end{tikzpicture}
}

\endgroup
\caption{KM policy, $q = 0.8$: effect of sample size $n$.}
\label{fig:ca_km_exploration}
\end{figure}

\Cref{fig:ca_km_exploration} provides several insights into the performance of the KM policy and the value of censored information.
We first observe that for fixed values of $n$ and $m$, increasing the inventory level $x$ at which demand is censored  reduces the worst-case regret of the KM policy. This behavior is intuitive: higher values of $x$ reduce the extent of censoring in the observed data and therefore provide more information about the underlying demand distribution.

We also note that, in all settings, the KM policy converges to a strictly positive worst-case regret level when some observations are censored. This reflects the fundamental impossibility of fully learning the demand distribution from censored data alone: increasing the sample size $n$ cannot drive worst-case regret to zero in the absence of uncensored observations. Moreover, this convergence occurs very quickly: in all cases, the worst-case regret flattens after only tens of samples, indicating that once the informational content of censored data is exhausted, additional observations provide little marginal benefit.

The most striking insight from the experiments is that a small amount of exploration goes a very long way. Introducing even a handful of uncensored observations leads to disproportionately large improvements in worst-case performance. For example, when $x=0.7$ and $n=100$, the worst-case regret decreases by approximately $20\%$ when \textit{a single observation is uncensored}, and is reduced by a factor of three when $m=5$ uncensored observations are available. In fact, a decision-maker with $95$ observations collected at inventory level $x=0.7$ and only $5$ uncensored observations collected at inventory level $1$ achieves essentially the same worst-case regret guarantee as a decision-maker with $100$ observations collected at $x=0.9$, that is, one who commits to holding $30\%$ more inventory under an order-up-to policy.

Finally and quite notably, very limited exploration is sufficient to nearly match the fully uncensored benchmark. With $m=10$ uncensored observations out of $n=100$ samples, the worst-case regret guarantee of the KM policy lies within $50\%$, $25\%$, and $12.5\%$ of the regret achieved under fully uncensored data (SAA) when $x$ equals $0.7$, $0.8$, and $0.9$, respectively. Taken together, these results demonstrate that while censoring fundamentally constrains what can be learned from passive data, deliberate and limited exploration can dramatically improve worst-case guarantees, and can do so with only minimal departures from standard operational practices.

\subsection{Beyond Kaplan--Meier: Other Policies from the Literature}\label{sec:ca_other_policies}

We conclude this section by leveraging our methodology to characterize the exact worst-case performance of recent policies developed in the literature on the censored newsvendor problem.

\noindent\textbf{The Fan--Chen--Zhou policy (FCZ).}
\citet[Algorithm~1]{fan2022sample} estimate demand mass separately on each interval $[x_{j-1},x_j)$ using the observations collected at inventories at least $x_j$. Let $M_j=\sum_{r=j}^K n_r$. Their cumulative score can be written as
\begin{equation}\label{eq:ca_fcz_score}
\widehat H_{\mathrm{FCZ}}(z)
=\sum_{j=1}^K\frac1{M_j}\sum_{r=j}^K\sum_{i=1}^{n_r}
\mathds{1}\{x_{j-1}\le S_i^{(r)}<x_j,\ S_i^{(r)}\le z\}.
\end{equation}
For a fixed tuning parameter $0\le\beta<q/2$, set $q'=q-2\beta$. The policy is
\begin{equation}\label{eq:ca_fcz_policy}
\pi_{\bm n}^{\mathrm{FCZ}}(\pastI,\bm I)
=\inf\{z\in[0,1):\widehat H_{\mathrm{FCZ}}(z)\ge q'\}.
\end{equation}
As in \Cref{def:censoring_aware_saa}, the policy chooses $1$ if the set is empty. The prescribed tuning in \citet{fan2022sample} is $\beta=\min\{q,1-q\}\epsilon/18$, where $\epsilon$ is their accuracy parameter. The downward shift by $2\beta$ is introduced to establish their sample-complexity guarantee: it places the target quantile in a region where their analysis controls estimation error. We use $\beta=0$ as the natural version of the policy, which applies the original critical fractile $q$ to the estimated cumulative score. Our exact characterization below covers both this choice and the shifted rule.

\noindent\textbf{The Hssaine--Sinclair policy (RCN$^+$).}
The RCN$^+$ policy of \citet[Algorithm~2]{hssaine2024data} pools observations from inventory levels at which the data indicate that the demand $q$-quantile lies below the inventory. For fixed confidence radii $\zeta_1,\ldots,\zeta_K>0$, define
\[
W_r=\frac1{n_r}\sum_{i=1}^{n_r}\mathds{1}\{S_i^{(r)}<x_r\},
\qquad
\widehat{\mathcal U}=\{r\in[K]:W_r\ge q+\zeta_r\}.
\]

When $\widehat{\mathcal U}\ne\emptyset$, let $N_{\widehat{\mathcal U}}=\sum_{r\in\widehat{\mathcal U}}n_r$ and define the pooled empirical lower $q$-quantile
\[
\widehat a_{\widehat{\mathcal U}}
=\inf\left\{z\in[0,1]:
\frac1{N_{\widehat{\mathcal U}}}\sum_{r\in\widehat{\mathcal U}}\sum_{i=1}^{n_r}
\mathds{1}\{S_i^{(r)}\le z\}\ge q\right\}.
\]
The policy is
\begin{equation}\label{eq:ca_rcnplus_policy}
\pi_{\bm n}^{\mathrm{RCN+}}(\pastI,\bm I)
=\begin{cases}
\widehat a_{\widehat{\mathcal U}},&\widehat{\mathcal U}\ne\emptyset,\\
x_K,&\widehat{\mathcal U}=\emptyset,\ W_K\ge q-\zeta_K,\\[0.3em]
\displaystyle x_K +(1-x_K)\frac{q-W_K}{1-W_K},
&\widehat{\mathcal U}=\emptyset,\ W_K<q-\zeta_K.
\end{cases}
\end{equation}

\begin{proposition}[Finite-dimensional reductions for FCZ and RCN$^+$]\label{prop:ca_other_reductions}
For every sample-size vector $\bm n\in\mathbb N^K$, historical inventory design $\pastI$, and fixed parameters $0\le\beta<q/2$ and $\bm\zeta\in(0,\infty)^K$, the worst-case regrets of \eqref{eq:ca_fcz_policy} and \eqref{eq:ca_rcnplus_policy} admit exact finite-dimensional characterizations. Specifically, they are maxima of continuous functions over compact sets with at most $3K+1$ and $2K$ continuous variables, respectively. 
\end{proposition}

\Cref{prop:ca_other_reductions} complements the concentration-based sample-complexity bounds of \citet{fan2022sample} and the finite-sample bounds for RCN in \citet{hssaine2024data}. It gives the exact worst-case expected regret for each fixed sample-size vector and inventory design, including the unshifted FCZ policy and the pooled RCN$^+$ policy. This is particularly useful when historical observations come from only a few inventory levels: the optimization has at most $3K+1$ variables for FCZ and $2K$ for RCN$^+$, regardless of the number of samples. These formulas therefore allow us to quantify how censoring levels and sample allocation affect performance even in small and moderate samples. The explicit formulations and proofs are given in \Cref{app:ca_literature}.

\medskip
\noindent\textbf{Numerical comparison.} We use the finite-dimensional reductions to evaluate the worst-case regret of BSAA, KM, FCZ, and a capped variant of RCN$^+$. We fix five censoring levels, equally spaced across $[0.5,0.8]$ with both endpoints included, so that $x_1=0.5$ and $x_5=0.8$. The sample size $n$ counts all historical observations: $m\in\{0,1,2,5\}$ of them are uncensored observations at inventory $1$, and the remaining $n-m$ censored observations are divided as evenly as possible across the five levels, with $\lfloor (n-m)/5\rfloor$ per level and any remainder assigned to the highest levels. Thus, the full historical design has $L=5+\mathds{1}\{m>0\}$ active inventory levels, with the exploration group indexed last when present. We set $q=0.8$ and $c_u+c_o=1$, and FCZ uses the unshifted rule ($\beta=0$). The RCN$^+$ confidence radii are $\zeta_j=\sqrt{\log\!\big(2(L-1)/\delta\big)/(2n_j)}$ for $j<L$ and $\zeta_L=\sqrt{\log(2/\delta)/(2n_L)}$, with $\delta=1/n_L$. In the plotted variant, the selected set is $\widehat{\mathcal U}_{\mathrm{cap}}=\{j\in[L]:W_j\ge\min\{q+\zeta_j,1\}\}$, with the pooled quantile and fallback branches defined as in \eqref{eq:ca_rcnplus_policy}.

\begin{figure}[htbp]
  \centering
\begin{tabular}{@{}llll@{}}
\paperlegend{paperRed,mark=square*}{BSAA}&\paperlegend{paperBlue,mark=*}{KM}&\paperlegend{paperBlack,mark=triangle*,mark options={fill=white},dash pattern=on 3.2bp off 1.8bp}{FCZ}&\paperlegend{paperOrange,mark=triangle*,mark options={rotate=180}}{RCN$^+$}
\end{tabular}\par\smallskip
\begin{tikzpicture}
\begin{groupplot}[
 group style={group size=4 by 1,horizontal sep=0.035\textwidth},
 width=0.20\textwidth,height=0.22\textwidth,scale only axis,
 xmode=log,ymode=log,xmin=9,xmax=110,ymin=.011,ymax=.51,
 xtick={10,15,20,30,40,60,100},ytick={.01,.02,.05,.1,.2,.5},
 xticklabels={10,15,20,30,40,60,100},yticklabels={.01,.02,.05,.1,.2,.5},
 xlabel={$n$},grid=major,tick label style={font=\scriptsize},
 label style={font=\small},title style={font=\small}]
\nextgroupplot[title={$m=0$},ylabel={Worst-case regret}]
\addplot[paper line,paperRed,mark=square*] table[x=n,y=BSAA,col sep=comma]{Data/figure3_m0.csv};
\addplot[paper line,paperBlue,mark=*] table[x=n,y=KM,col sep=comma]{Data/figure3_m0.csv};
\addplot[paper line,paperBlack,paper dashed,mark=triangle*,mark options={solid,fill=white}] table[x=n,y=FCZ,col sep=comma]{Data/figure3_m0.csv};
\addplot[paper line,paperOrange,mark=triangle*,mark options={solid,rotate=180}] table[x=n,y=RCNplus,col sep=comma]{Data/figure3_m0.csv};
\nextgroupplot[title={$m=1$},yticklabels={}]
\addplot[paper line,paperRed,mark=square*] table[x=n,y=BSAA,col sep=comma]{Data/figure3_m1.csv};
\addplot[paper line,paperBlue,mark=*] table[x=n,y=KM,col sep=comma]{Data/figure3_m1.csv};
\addplot[paper line,paperBlack,paper dashed,mark=triangle*,mark options={solid,fill=white}] table[x=n,y=FCZ,col sep=comma]{Data/figure3_m1.csv};
\addplot[paper line,paperOrange,mark=triangle*,mark options={solid,rotate=180}] table[x=n,y=RCNplus,col sep=comma]{Data/figure3_m1.csv};
\nextgroupplot[title={$m=2$},yticklabels={}]
\addplot[paper line,paperRed,mark=square*] table[x=n,y=BSAA,col sep=comma]{Data/figure3_m2.csv};
\addplot[paper line,paperBlue,mark=*] table[x=n,y=KM,col sep=comma]{Data/figure3_m2.csv};
\addplot[paper line,paperBlack,paper dashed,mark=triangle*,mark options={solid,fill=white}] table[x=n,y=FCZ,col sep=comma]{Data/figure3_m2.csv};
\addplot[paper line,paperOrange,mark=triangle*,mark options={solid,rotate=180}] table[x=n,y=RCNplus,col sep=comma]{Data/figure3_m2.csv};
\nextgroupplot[title={$m=5$},yticklabels={}]
\addplot[paper line,paperRed,mark=square*] table[x=n,y=BSAA,col sep=comma]{Data/figure3_m5.csv};
\addplot[paper line,paperBlue,mark=*] table[x=n,y=KM,col sep=comma]{Data/figure3_m5.csv};
\addplot[paper line,paperBlack,paper dashed,mark=triangle*,mark options={solid,fill=white}] table[x=n,y=FCZ,col sep=comma]{Data/figure3_m5.csv};
\addplot[paper line,paperOrange,mark=triangle*,mark options={solid,rotate=180}] table[x=n,y=RCNplus,col sep=comma]{Data/figure3_m5.csv};
\end{groupplot}
\end{tikzpicture}

\caption{Worst-case regret at $K=5$ censoring levels for different data-driven policies.}
  \label{fig:new_policies_comparison}
\end{figure}

\Cref{fig:new_policies_comparison} illustrates how the reductions allows to compare policies 
precisely at each fixed sample allocation. Two findings stand out. First, KM and FCZ are the 
two best-performing policies in every panel, and their worst-case regrets agree to within $10^{-3}$ at every configuration in the grid, even though the two policies are not equivalent  and select different actions on many sample paths. The agreement is striking because the two  rules do not use the same data: KM relies on the lost-sales indicators $\delta_i^{(k)}$, whereas FCZ uses only the censored sales and the historical inventory levels. A policy that can be run on the point-of-sale records a firm already keeps therefore attains essentially the same worst-case performance as one that additionally requires knowing which observations were censored.

We finally note that RCN+ was designed for a slightly different objective, which is best represented by the case m = 0 above, explaining its stronger performance in that regime.

\section{On the Design of New Policies}\label{sec:new_policies}

In the previous sections, we used \Cref{thm:general_reduction} to evaluate the worst-case regret of existing data-driven policies. We now use the reduction to design policies with lower worst-case regret. We restrict attention to a subclass for which worst-case regret admits a finite-dimensional characterization, and formulate the design of an optimal policy within this subclass as a linear program.

Our policy class is motivated by two observations. First, in the uncensored i.i.d. newsvendor problem, \citet{besbes2023big} show that a fixed randomization over demand order statistics minimizes worst-case relative regret among all data-driven policies. Second, the performance of BSAA in \Cref{sec:SAA_reduction}, compared with the censoring-aware policies in \Cref{sec:censoring_aware}, motivates allowing the randomization to adapt to censoring information. We therefore consider sales order statistics whose selection probabilities depend on weak comparisons with historical inventories.

Fix $0<x_1<\cdots<x_K\le1$ and sample sizes $\bm n$. Define the weak comparison counts and their range by
\[
 H^{\le}_{\pastI}
 =\left(\sum_{i=1}^{n_k}\mathds{1}\{D_i^{(k)}\le x_j\}
 \;\middle|\; k\in[K],\ j\in[k]\right),
 \qquad
 \mathcal H_{\pastI}
 =\left\{H^{\le}_{\pastI}(\bm d):\bm d\in[0,1]^n\right\}.
\]
As in \Cref{sec:model}, let $Y_1\le\cdots\le Y_n$ denote the pooled sales order statistics, and set $Y_0=0$ and $Y_{n+1}=1$.

\begin{definition}\label{def:new_policy_class}
A policy is a \emph{censoring-aware randomized order-statistic policy} if, for each $h\in\mathcal H_{\pastI}$, there is a probability vector $(\lambda_{h,r})_{r=0}^{n+1}$ such that
\begin{equation*}
 \pi^{\bm\lambda}_{\bm n}(\pastI,\bm I)=Y_R,
 \qquad
 \Prob(R=r\mid\bm I)=\lambda_{H^{\le}_{\pastI},r}.
\end{equation*}
Write $\bm\lambda=(\lambda_{h,r})_{h,r}$ and $\Lambda_{\pastI}=\Delta(\{0,\ldots,n+1\})^{\mathcal H_{\pastI}}$ for the set of admissible weights.
\end{definition}
We note that Kaplan--Meier is a deterministic member of this class. Indeed, once the weak signature is fixed, the ordered sequence of uncensored observations and censoring events is determined, with uncensored observations preceding censored ones at a tie. The rank at which the product-limit estimate reaches $q$ is consequently fixed. Assigning probability one to that rank reproduces the KM action; if the estimate does not reach $q$ below $1$, assigning probability one to $Y_{n+1}$ recovers the action of KM.

\paragraph{Linear Program and Worst-Case Guarantees.}
Our goal is to provide a tractable formulation to design new data-driven policies by minimizing the worst-case regret over censoring-aware randomized order-statistic policies.

Define the set of points which will represent the support of our candidate worst-case distributions as
$\mathcal B_{\pastI}=\{0,1\}\cup\{x,x^+:x\in\{x_1,\ldots,x_K\}\cap(0,1)\},$ where the label $x^+$ represents demand approaching $x$ from above and let $d$ be the cardinality of this set.
For a probability vector $\bm{p}\in\Delta_d$ on these labels, let $a_{h,r}(\bm{p})$ denote the expected regret contribution from selecting rank $r$ on the event $H^{\le}_{\pastI}=h$. These coefficients depend only on the demand probabilities and the historical design; their explicit expression is given in \Cref{sec:proof_new_policies}.

The policy-design problem can be formulated as the following epigraph LP
\begin{equation}\label{eq:new_policy_silp}
\begin{aligned}
 t^*=\min_{\bm\lambda,t}\quad &t\\
 \text{subject to}\quad
 &\sum_{h\in\mathcal H_{\pastI}}\sum_{r=0}^{n+1}
       a_{h,r}(\bm{p})\lambda_{h,r}\le t,
       &&\bm{p}\in\Delta_d,\\
 &\sum_{r=0}^{n+1}\lambda_{h,r}=1,
       &&h\in\mathcal H_{\pastI},\\
 &\lambda_{h,r}\ge0,
       &&h\in\mathcal H_{\pastI},\ r=0,\ldots,n+1.
\end{aligned}
\end{equation}
The decision vector $\bm\lambda$ specifies the rank probabilities for every weak comparison signature. The first set of constraints ensures that the regret is bounded by $t$ for every probability vector on the boundary labels $\mathcal B_{\pastI}$. The reduction in \Cref{thm:general_reduction}, specialized to weak comparisons in \Cref{sec:proof_new_policies}, makes these constraints sufficient to control regret under \textit{every} demand distribution, so minimizing $t$ selects the best worst-case guarantee within the policy class. Problem \eqref{eq:new_policy_silp} is therefore a semi-infinite LP.
Finally, let $\Gamma_\epsilon\subset\Delta_d$ be a grid with spacing at most $\epsilon$, i.e. such that every $\bm{p}\in\Delta_d$ has some $\bm{g}\in\Gamma_\epsilon$ with $\frac12\|\bm{p}-\bm{g}\|_1\le\epsilon$. Restricting $\bm{p}$ to this grid gives an ordinary LP, denoted $\mathrm{LP}(\Gamma_\epsilon)$.

The following theorem formalizes that this LP provides a computational tool to design data-driven policies with a certified worst-case regret and bounds the approximation loss.

\begin{theorem}[Policy design and worst-case guarantees]\label{thm:new_policy_design}
Every policy in \Cref{def:new_policy_class} is piecewise-separable. The optimum in \eqref{eq:new_policy_silp} is attained, and any optimal solution $(\bm\lambda^*,t^*)$ satisfies
\begin{equation}\label{eq:new_policy_km_dominance}
 \reg(\pi^{\bm\lambda^*}_{\bm n}\mid\pastI)
 =t^*
 =\min_{\bm\lambda\in\Lambda_{\pastI}}\reg(\pi^{\bm\lambda}_{\bm n}\mid\pastI).
\end{equation}
If $(\bm\lambda^\epsilon,t_\epsilon)$ solves $\mathrm{LP}(\Gamma_\epsilon)$, then,
\begin{equation}\label{eq:new_policy_finite_guarantee}
 t_\epsilon\le t^*\le
 \reg(\pi^{\bm\lambda^\epsilon}_{\bm n}\mid\pastI)
 \le t_\epsilon+\max\{\co,\cu\}(n+2)\epsilon.
\end{equation}
\end{theorem}

\paragraph{Numerical Comparison with Kaplan--Meier.}
We consider $q=0.8$ and normalize costs so that $\cu+\co=1$. For each total sample size $n\in\{1,10,20,30,40,50\}$, we allocate $n-m$ observations to inventory $x=0.8$ and $m\in\{0,1,2,5\}$ observations to inventory $1$, omitting allocations with $m>n$. We optimize a separate policy for each of the resulting 22 instances. \Cref{fig:new_policy_curves} reports certified upper bounds on the worst-case regrets of these policies and KM as a function of $n$, with one panel for each exploration allocation $m$. The additional reference curve is a certified lower bound on the minimax additive regret with $n$ fully observed i.i.d. demands. It therefore also lower-bounds the worst-case regret of any policy based on $n$ censored observations.
To compute this benchmark, we use \Cref{thm:new_policy_uncensored} in \Cref{sec:new_policy_uncensored}, which establishes that a randomization over two adjacent demand order statistics is minimax-optimal for worst-case additive regret in the i.i.d. uncensored newsvendor problem (Theorem 3 of \citet{besbes2023big} establishes this fact for relative regret).

\begin{figure}[!h]
\centering
\begingroup
\pgfplotsset{newpolicyaxis/.style={
  width=0.44\textwidth,height=0.31\textwidth,
  xlabel={Total observations $n$},ylabel={Worst-case regret},
  xmin=0,xmax=52,xtick={1,10,20,30,40,50},ymin=0,
  scaled y ticks=false,grid=major,
  tick label style={font=\footnotesize},label style={font=\small},
  yticklabel style={/pgf/number format/fixed,/pgf/number format/precision=2},
  legend style={draw=none,font=\small},legend columns=3,
  every axis plot/.append style={paper line}},
  newpolicylp/.style={paper lp},
  newpolicykm/.style={paper km},
  newpolicyiid/.style={paper benchmark}}
\begin{tabular}{@{}lll@{}}\paperlegend{paperViolet,mark=diamond*}{LP-based policy}&\paperlegend{paperBlue,mark=*}{Kaplan--Meier}&\paperlegend{paperBlack,mark=none,dash pattern=on 3.2bp off 1.8bp}{Uncensored minimax}\end{tabular}\par\smallskip
\subfigure[$m=0$]{\begin{tikzpicture}
\begin{axis}[newpolicyaxis,ymax=.17,ytick={0,.04,.08,.12,.16}]
\addplot[newpolicylp] table[x=n,y=lp,col sep=comma]{Data/new_policy_m0.csv};
\addplot[newpolicykm] table[x=n,y=km,col sep=comma]{Data/new_policy_m0.csv};
\addplot[newpolicyiid] table[x=n,y=iid,col sep=comma]{Data/new_policy_m0.csv};
\end{axis}\end{tikzpicture}}
\hspace{0.025\textwidth}
\subfigure[$m=1$]{\begin{tikzpicture}
\begin{axis}[newpolicyaxis,ymax=.17,ytick={0,.04,.08,.12,.16}]
\addplot[newpolicylp] table[x=n,y=lp,col sep=comma]{Data/new_policy_m1.csv};
\addplot[newpolicykm] table[x=n,y=km,col sep=comma]{Data/new_policy_m1.csv};
\addplot[newpolicyiid] table[x=n,y=iid,col sep=comma]{Data/new_policy_m1.csv};
\end{axis}\end{tikzpicture}}
\par\medskip
\subfigure[$m=2$]{\begin{tikzpicture}
\begin{axis}[newpolicyaxis,ymax=.045,ytick={0,.01,.02,.03,.04}]
\addplot[newpolicylp] table[x=n,y=lp,col sep=comma]{Data/new_policy_m2.csv};
\addplot[newpolicykm] table[x=n,y=km,col sep=comma]{Data/new_policy_m2.csv};
\addplot[newpolicyiid] table[x=n,y=iid,col sep=comma]{Data/new_policy_m2.csv};
\end{axis}\end{tikzpicture}}
\hspace{0.025\textwidth}
\subfigure[$m=5$]{\begin{tikzpicture}
\begin{axis}[newpolicyaxis,ymax=.045,ytick={0,.01,.02,.03,.04}]
\addplot[newpolicylp] table[x=n,y=lp,col sep=comma]{Data/new_policy_m5.csv};
\addplot[newpolicykm] table[x=n,y=km,col sep=comma]{Data/new_policy_m5.csv};
\addplot[newpolicyiid] table[x=n,y=iid,col sep=comma]{Data/new_policy_m5.csv};
\end{axis}\end{tikzpicture}}
\endgroup

\caption{Certified worst-case regret bounds for LP-based policies and Kaplan--Meier, with the uncensored minimax lower bound at $q=x=0.8$ and $\cu+\co=1$.}
\label{fig:new_policy_curves}
\end{figure}

Kaplan--Meier matched the best performance among the existing policies evaluated in \Cref{sec:ca_other_policies}. And yet, \Cref{fig:new_policy_curves} shows that our LP-based policies considerably improve on this benchmark in all 22 instances: at $n=50$, our policy reduces the worst-case regret by $26\%$ to $58\%$ across the four exploration allocations. Thus, the reduction provides a practical route from evaluating existing policies to designing substantially better ones.

The comparison with the uncensored minimax benchmark gives a stronger indication of how much room for improvement remains. Although we do not establish minimax optimality among all policies based on censored sales, at $n=50$ and $m=5$ the LP-based policy's certified regret bound is approximately $5.3\%$ above the minimax value attainable when all 50 demands \textit{are fully observed.} Since that value is a lower bound for every policy using censored data, the certified bounds establish near-optimal performance even against this stronger information benchmark. This also quantifies the value of censored data: with only five uncensored observations---$10\%$ of the sample---the remaining 45 censored observations suffice to nearly recover fully observed performance. For smaller exploration allocations, the gap is larger but it remains open whether this gap mainly reflects limitations of the proposed policy class or the information lost through censoring.

\section{Conclusion}

We study the offline data-driven newsvendor problem under demand censoring, where historical observations consist of sales, inventory levels, and stockout indicators. Our main contribution is an optimization-based framework that yields an exact characterization of the worst-case regret of data-driven policies for any sample size and fixed historical inventory design by reducing an infinite-dimensional and nonconvex adversarial problem to a finite-dimensional one. This reduction enables sharp regret analysis for a broad class of inventory policies, including sales-as-demand, Kaplan--Meier, and recent policies from the literature. Our results quantify the limits of censored data and show that a few exploratory samples can substantially improve Kaplan--Meier's performance. We also use linear programming to design new data-driven policies that substantially improve on Kaplan--Meier in our experiments and, with limited exploration, can nearly attain the fully observed minimax benchmark. More broadly, our analysis clarifies how information constraints reshape the value of data and the sample complexity of learning in operational settings.
Extending this approach to other offline operational problems is a natural direction for future research.

{\setstretch{1.0}
\bibliographystyle{agsm}
\bibliography{refs}

@article{qi2022distributionally,
  title={Distributionally robust conditional quantile prediction with fixed design},
  author={Qi, Meng and Cao, Ying and Shen, Zuo-Jun},
  journal={Management Science},
  volume={68},
  number={3},
  pages={1639--1658},
  year={2022},
  publisher={INFORMS}
}

@article{gupta2024debiasing,
  title={Debiasing in-sample policy performance for small-data, large-scale optimization},
  author={Gupta, Vishal and Huang, Michael and Rusmevichientong, Paat},
  journal={Operations Research},
  volume={72},
  number={2},
  pages={848--870},
  year={2024},
  publisher={INFORMS}
}

@article{csorgo1983rate,
  author  = {Cs{\"o}rg{\H{o}}, Mikl{\'o}s and Horv{\'a}th, Lajos},
  title   = {The Rate of Strong Uniform Consistency for the Kaplan--Meier Estimator},
  journal = {Zeitschrift f{\"u}r Wahrscheinlichkeitstheorie und Verwandte Gebiete},
  volume  = {62},
  number  = {3},
  pages   = {411--426},
  year    = {1983}
}

@article{gill1983large,
  author  = {Gill, Richard D.},
  title   = {Large Sample Behaviour of the Product--Limit Estimator on the Whole Line},
  journal = {The Annals of Statistics},
  volume  = {11},
  number  = {1},
  pages   = {49--58},
  year    = {1983}
}

@article{bitouze1999dkm,
  author  = {Bitouz{\'e}, Denis and Laurent, Beatrice and Massart, Pascal},
  title   = {A {D}voretzky--{K}iefer--{W}olfowitz Type Inequality for the Kaplan--Meier Estimator},
  journal = {Annales de l'Institut Henri Poincar\'e, Probabilit\'es et Statistiques},
  volume  = {35},
  number  = {6},
  pages   = {735--763},
  year    = {1999}
}

@article{allouah2023optimal,
  title={Optimal pricing with a single point},
  author={Allouah, Amine and Bahamou, Achraf and Besbes, Omar},
  journal={Management Science},
  volume={69},
  number={10},
  pages={5866--5882},
  year={2023},
  publisher={INFORMS}
}

@article{bahamou2024fast,
  title={Fast revenue maximization},
  author={Bahamou, Achraf and Besbes, Omar and Mouchtaki, Omar},
  journal={arXiv preprint arXiv:2407.07316},
  year={2024}
}

@article{stoye2009minimax,
  title={Minimax regret treatment choice with finite samples},
  author={Stoye, J{\"o}rg},
  journal={Journal of Econometrics},
  volume={151},
  number={1},
  pages={70--81},
  year={2009},
  publisher={Elsevier}
}

@article{schlag2006eleven,
  title={ELEVEN-Tests needed for a Recommendation},
  author={Schlag, Karl H},
  year={2006},
  publisher={European University Institute}
}

@article{liu2023marginal,
  title   = {Marginal Value of One Data Point in Assortment Personalization},
  author  = {Liu, Mo and Cao, Junyu and Shen, Zuo-Jun Max},
  note    = {Working paper},
  year    = {2023},
  institution = {University of California, Berkeley and University of Texas at Austin},
}

@article{zhang2024more,
  title={More Data or Better Data? Impact of Costly Data Collection on the Newsvendor Problem},
  author={Zhang, Zijin and Ahn, Hyun-Soo and Baardman, Lennart},
  journal={Impact of Costly Data Collection on the Newsvendor Problem (September 06, 2024)},
  year={2024}
}

@article{hssaine2024data,
  title={The Data-Driven Censored Newsvendor Problem},
  author={Hssaine, Chamsi and Sinclair, Sean R},
  journal={arXiv preprint arXiv:2412.01763},
  year={2024}
}

@article{fan2022sample,
  title={Sample complexity of policy learning for inventory control with censored demand},
  author={Fan, Xiaoyu and Chen, Boxiao and Zhou, Zhengyuan},
  journal={Available at SSRN 4178567},
  year={2022}
}

@article{chen2020dynamic,
  title={Dynamic inventory control with stockout substitution and demand learning},
  author={Chen, Boxiao and Chao, Xiuli},
  journal={Management Science},
  volume={66},
  number={11},
  pages={5108--5127},
  year={2020},
  publisher={INFORMS}
}

@article{agrawal2022learning,
  title={Learning in Structured MDPs with Convex Cost Functions:: Improved Regret Bounds for Inventory Management},
  author={Agrawal, Shipra and Jia, Randy},
  journal={Operations Research},
  volume={70},
  number={3},
  pages={1646--1664},
  year={2022},
  publisher={INFORMS}
}

@article{ban2020confidence,
  title={Confidence intervals for data-driven inventory policies with demand censoring},
  author={Ban, Gah-Yi},
  journal={Operations Research},
  volume={68},
  number={2},
  pages={309--326},
  year={2020},
  publisher={INFORMS}
}

@article{zhang2020closing,
  title={Closing the gap: A learning algorithm for lost-sales inventory systems with lead times},
  author={Zhang, Huanan and Chao, Xiuli and Shi, Cong},
  journal={Management Science},
  volume={66},
  number={5},
  pages={1962--1980},
  year={2020},
  publisher={INFORMS}
}

@article{zhang2018perishable,
  title={Perishable inventory systems: Convexity results for base-stock policies and learning algorithms under censored demand},
  author={Zhang, Huanan and Chao, Xiuli and Shi, Cong},
  journal={Operations Research},
  volume={66},
  number={5},
  pages={1276--1286},
  year={2018},
  publisher={INFORMS}
}

@article{yuan2021marrying,
  title={Marrying stochastic gradient descent with bandits: Learning algorithms for inventory systems with fixed costs},
  author={Yuan, Hao and Luo, Qi and Shi, Cong},
  journal={Management Science},
  volume={67},
  number={10},
  pages={6089--6115},
  year={2021},
  publisher={Informs}
}

@article{shi2016nonparametric,
  title={Nonparametric data-driven algorithms for multiproduct inventory systems with censored demand},
  author={Shi, Cong and Chen, Weidong and Duenyas, Izak},
  journal={Operations Research},
  volume={64},
  number={2},
  pages={362--370},
  year={2016},
  publisher={INFORMS}
}

@article{natarajan2018asymmetry,
  title={Asymmetry and ambiguity in newsvendor models},
  author={Natarajan, Karthik and Sim, Melvyn and Uichanco, Joline},
  journal={Management Science},
  volume={64},
  number={7},
  pages={3146--3167},
  year={2018},
  publisher={INFORMS}
}

@article{burnetas2000adaptive,
  title={Adaptive ordering and pricing for perishable products},
  author={Burnetas, Apostolos N and Smith, Craig E},
  journal={Operations Research},
  volume={48},
  number={3},
  pages={436--443},
  year={2000},
  publisher={INFORMS}
}

@article{kunnumkal2008using,
  title={Using stochastic approximation methods to compute optimal base-stock levels in inventory control problems},
  author={Kunnumkal, Sumit and Topaloglu, Huseyin},
  journal={Operations Research},
  volume={56},
  number={3},
  pages={646--664},
  year={2008},
  publisher={INFORMS}
}

@article{chen2024survey,
  title={Survey of data-driven newsvendor: Unified analysis and spectrum of achievable regrets},
  author={Chen, Zhuoxin and Ma, Will},
  journal={arXiv preprint arXiv:2409.03505},
  year={2024}
}

@article{lin2022data,
  title={Data-driven newsvendor problem: Performance of the sample average approximation},
  author={Lin, Meichun and Huh, Woonghee Tim and Krishnan, Harish and Uichanco, Joline},
  journal={Operations Research},
  volume={70},
  number={4},
  pages={1996--2012},
  year={2022},
  publisher={INFORMS}
}

@article{bagnoli2005log,
  title={Log-concave probability and its applications},
  author={Bagnoli, Mark and Bergstrom, Ted},
  journal={Economic theory},
  volume={26},
  number={2},
  pages={445--469},
  year={2005},
  publisher={Springer}
}

@article{besbes2023contextual,
  title={From contextual data to newsvendor decisions: On the actual performance of data-driven algorithms},
  author={Besbes, Omar and Ma, Will and Mouchtaki, Omar},
  journal={arXiv preprint arXiv:2302.08424},
  year={2023}
}

@article{gupta2021small,
  title={Small-data, large-scale linear optimization with uncertain objectives},
  author={Gupta, Vishal and Rusmevichientong, Paat},
  journal={Management Science},
  volume={67},
  number={1},
  pages={220--241},
  year={2021},
  publisher={INFORMS}
}

@article{gupta2022data,
  title={Data pooling in stochastic optimization},
  author={Gupta, Vishal and Kallus, Nathan},
  journal={Management Science},
  volume={68},
  number={3},
  pages={1595--1615},
  year={2022},
  publisher={INFORMS}
}

@article{godfrey2001adaptive,
  author = {Godfrey, Gregory A and Powell, Warren B},
  journal = {Management Science},
  number = {8},
  pages = {1101--1112},
  publisher = {INFORMS},
  title = {An adaptive, distribution-free algorithm for the newsvendor problem with censored demands, with applications to inventory and distribution},
  volume = {47},
  year = {2001}}

@article{lugosi2024hardness,
  title={On the hardness of learning from censored and nonstationary demand},
  author={Lugosi, G{\'a}bor and Markakis, Mihalis G and Neu, Gergely},
  journal={INFORMS Journal on Optimization},
  volume={6},
  number={2},
  pages={63--83},
  year={2024},
  publisher={INFORMS}
}

@article{huh2009nonparametric,
  author = {Huh, Woonghee Tim and Rusmevichientong, Paat},
  journal = {Mathematics of Operations Research},
  number = {1},
  pages = {103--123},
  publisher = {INFORMS},
  title = {A nonparametric asymptotic analysis of inventory planning with censored demand},
  volume = {34},
  year = {2009}}

@article{huh2011adaptive,
  author = {Huh, Woonghee Tim and Levi, Retsef and Rusmevichientong, Paat and Orlin, James B},
  journal = {Operations Research},
  number = {4},
  pages = {929--941},
  publisher = {INFORMS},
  title = {Adaptive data-driven inventory control with censored demand based on Kaplan-Meier estimator},
  volume = {59},
  year = {2011}}

@article{besbes2013implications,
  author = {Besbes, Omar and Muharremoglu, Alp},
  journal = {Management Science},
  number = {6},
  pages = {1407--1424},
  publisher = {INFORMS},
  title = {On implications of demand censoring in the newsvendor problem},
  volume = {59},
  year = {2013}}

@article{perakis2008regret,
  author = {Perakis, Georgia and Roels, Guillaume},
  journal = {Operations Research},
  number = {1},
  pages = {188--203},
  publisher = {INFORMS},
  title = {Regret in the newsvendor model with partial information},
  volume = {56},
  year = {2008}}

@article{scarf1958min,
  author = {Scarf, Herbert},
  journal = {Studies in the mathematical theory of inventory and production},
  publisher = {Stanford Univ. Press},
  title = {A min-max solution of an inventory problem},
  year = {1958}}

@article{gallego1993distribution,
  author = {Gallego, Guillermo and Moon, Ilkyeong},
  journal = {Journal of the Operational Research Society},
  number = {8},
  pages = {825--834},
  publisher = {Taylor \& Francis},
  title = {The distribution free newsboy problem: review and extensions},
  volume = {44},
  year = {1993}}

@article{cheung2019sampling,
  title={Sampling-based approximation schemes for capacitated stochastic inventory control models},
  author={Cheung, Wang Chi and Simchi-Levi, David},
  journal={Mathematics of Operations Research},
  volume={44},
  number={2},
  pages={668--692},
  year={2019},
  publisher={INFORMS}
}

@article{besbes2023big,
  title={How big should your data really be? Data-driven newsvendor: Learning one sample at a time},
  author={Besbes, Omar and Mouchtaki, Omar},
  journal={Management Science},
  volume={69},
  number={10},
  pages={5848--5865},
  year={2023},
  publisher={INFORMS}
}

@article{levi2007approximation,
  author = {Levi, Retsef and P{\'a}l, Martin and Roundy, Robin O and Shmoys, David B},
  journal = {Mathematics of Operations Research},
  number = {2},
  pages = {284--302},
  publisher = {INFORMS},
  title = {Approximation algorithms for stochastic inventory control models},
  volume = {32},
  year = {2007}}

@article{levi2015data,
  title={The data-driven newsvendor problem: new bounds and insights},
  author={Levi, Retsef and Perakis, Georgia and Uichanco, Joline},
  journal={Operations Research},
  volume={63},
  number={6},
  pages={1294--1306},
  year={2015},
  publisher={INFORMS}
}

@article{huang2018making,
  title={Making the most of your samples},
  author={Huang, Zhiyi and Mansour, Yishay and Roughgarden, Tim},
  journal={SIAM Journal on Computing},
  volume={47},
  number={3},
  pages={651--674},
  year={2018},
  publisher={SIAM}
}

@article{allouah2022pricing,
author = {Allouah, Amine and Bahamou, Achraf and Besbes, Omar},
title = {Pricing with Samples},
journal = {Operations Research},
volume = {70},
number = {2},
pages = {1088-1104},
year = {2022}
}

@inproceedings{daskalakis2020more,
  title={More revenue from two samples via factor revealing SDPs},
  author={Daskalakis, Constantinos and Zampetakis, Manolis},
  booktitle={Proceedings of the 21st ACM Conference on Economics and Computation},
  pages={257--272},
  year={2020}
}

@inproceedings{babaioff2018two,
  title={Are Two (Samples) Really Better Than One?},
  author={Babaioff, Moshe and Gonczarowski, Yannai A and Mansour, Yishay and Moran, Shay},
  booktitle={Proceedings of the 2018 ACM Conference on Economics and Computation},
  pages={175--175},
  year={2018}
}

@inproceedings{fu2015randomization,
  title={Randomization beats second price as a prior-independent auction},
  author={Fu, Hu and Immorlica, Nicole and Lucier, Brendan and Strack, Philipp},
  booktitle={Proceedings of the Sixteenth ACM Conference on Economics and Computation},
  pages={323--323},
  year={2015}
}

@book{boyd2004convex,
  title={Convex optimization},
  author={Boyd, Stephen and Vandenberghe, Lieven},
  year={2004},
  publisher={Cambridge university press}
}

@article{jain2015demand,
  title={Demand estimation and ordering under censoring: Stock-out timing is (almost) all you need},
  author={Jain, Aditya and Rudi, Nils and Wang, Tong},
  journal={Operations Research},
  volume={63},
  number={1},
  pages={134--150},
  year={2015},
  publisher={Informs}
}

@article{kaplan1958nonparametric,
  title={Nonparametric estimation from incomplete observations},
  author={Kaplan, Edward L and Meier, Paul},
  journal={Journal of the American statistical association},
  volume={53},
  number={282},
  pages={457--481},
  year={1958},
  publisher={Taylor \& Francis}
}
}

\newpage
\appendix

\section{Proofs of Results in \Cref{sec:gen_reduction}}

\begin{proof}[\textbf{Proof of \Cref{lem:integral_form}}]
The structure of the argument follows (\citealp[Lemma 2]{besbes2023contextual}).
Fix $\pastI$ and a demand distribution $F$ on $[0,1]$, and recall that
$\pi_{\bm{n}}(\pastI,\bm{I})$
denote the (random) action induced by $\bm{D}\sim F^n$.

For any deterministic $a\in[0,1]$, we can write the two newsvendor terms pointwise as

\begin{equation*}
 (a-D)_+ = \int_{0}^{a} \mathds{1}\{D \le z\}\,dz \quad \mbox{and} \quad (D-a)_+ = \int_{a}^{1} \mathds{1}\{D > z\}\,dz
\end{equation*}
Therefore, recalling $q=\cu/(\cu+\co)$ and $c(a,F)=\mathbb{E}_{D\sim F}[\co(a-D)_+ + \cu(D-a)_+]$, we obtain
\begin{equation*}
c(a,F)
= \co \int_{0}^{a} F(z)\,dz + \cu \int_{a}^{1} \big(1-F(z)\big)\,dz
= \cu\int_{0}^{1}\big(1-F(z)\big)\,dz + (\cu+\co)\int_{0}^{a}\big(F(z)-q\big)\,dz.
\end{equation*}
The first term does not depend on $a$. Since $F$ is nondecreasing, the map
$a \mapsto \int_{0}^{a}(F(z)-q)\,dz$ is convex and is minimized at any $q$-quantile of $F$; 
In particular, for
$a^* := \inf\{z\in[0,1]: F(z)\ge q\}$ we have
\begin{equation*}
    \opt(F) = c(a^*,F) = \cu\int_{0}^{1}\big(1-F(z)\big)\,dz - (\cu + \co) \int_{0}^{1} \big(q-F(z)\big)_+\,dz,
\end{equation*}
where the equality uses that $(q-F(z))_+ =q-F(z)$ for $z<a^*$ and $(q-F(z))_+ =0$ for $z\ge a^*$.
Hence, for any $a\in[0,1]$,
\begin{align*}
R(a,F)
&= c(a,F)-\opt(F)\\
&= (\cu+\co)\left(\int_{0}^{a}\big(F(z)-q\big)\,dz + \int_{0}^{1}\big(q-F(z)\big)_+\,dz\right)\\
&= (\cu+\co)\int_{0}^{1}\Big[\mathds{1}\{a>z\}\big(F(z)-q\big) + \big(q-F(z)\big)_+\Big]dz\\
&= \int_{0}^{1}\Big[\mathds{1}\{a>z\}\big(F(z)-q\big) + \big(q-F(z)\big)_+\Big]dz \tag{$\cu + \co = 1$}
\end{align*}
Applying this identity with $a=\pi_{\bm{n}}(\pastI,\bm{I})$ and using Tonelli's theorem to exchange expectation and integration (the integrand is nonnegative) gives
\begin{equation*}
\mathbb{E}_{\bm{D}\sim F^n}\big[R(\pi_{\bm{n}}(\pastI,\bm{I}),F)\big]
= \int_{0}^{1}\Big[\Prob_{\bm{D}\sim F^n}(\pi_{\bm{n}}(\pastI,\bm{I})>z)\cdot\big(F(z)-q\big) + \big(q-F(z)\big)_+\Big]dz.
\end{equation*}
Finally, $\Prob(\pi_{\bm{n}}(\pastI,\bm{I})>z)=1-\Prob(\pi_{\bm{n}}(\pastI,\bm{I})\le z)$, which yields the claimed representation.
\end{proof}

\begin{proof}[\textbf{Proof of \Cref{thm:general_reduction}}]
Let $F \in \Delta \left( [0,1] \right)$ and assume $x_1 > 0$ and $x_K < 1$. We have that, 
\begin{align*}
\mathbb{E}_{\bm{D}\sim F^n} \Big[ R \big(\pi_{\bm{n}}(\pastI,\bm{I}), F\big) \Big]
&\stackrel{(a)}{=} \int_{0}^{1} \Big[
\big(1 - \Prob_{\bm{D} \sim F^n}\big( \pi_{\bm{n}}(\pastI,\bm{I}) \le z \big) \big)\cdot \big(F(z) - q\big)
+ (q - F(z))_+ \Big] dz\\ 
&\stackrel{(b)}{=} \sum_{k=0}^{K} \int_{x_{k}}^{x_{k+1}}\Big[
\big(1 - P_k^{\pi_{\bm{n}}}\big( F(\pastI^-), F(\pastI), F(z) \big) \big)\cdot \big(F(z) - q\big) + (q - F(z))_+ \Big] dz
\end{align*}
where (a) follows from \Cref{lem:integral_form} and (b) from the definition of a piecewise-separable policy. 
For each $k \in \{0,\ldots,K\}$ and for all $(\bm{u}^{-},\bm{u},v) \in [0,1]^{K} \times [0,1]^{K} \times [0,1]$, define the mapping 
\begin{equation*}
\Psi_k^{\pi_{\bm{n}}}(\bm{u}^{-},\bm{u},v) = \big(1 - P_k^{\pi_{\bm{n}}}(\bm{u}^{-},\bm{u},v)\big)\cdot (v-q) + (q - v)_+,
\end{equation*}
and note that this mapping is continuous by continuity of $P_k^{\pi_{\bm{n}}}$.

We have just established that,
\begin{equation*}
\mathbb{E}_{\bm{D}\sim F^n} \Big[ R \big(\pi_{\bm{n}}(\pastI,\bm{I}), F\big) \Big] = (\cu+\co) \sum_{k=0}^{K} \int_{x_{k}}^{x_{k+1}} \Psi_k^{\pi_{\bm{n}}}(F(\pastI^-),F(\pastI),F(z))dz.
\end{equation*}
We next show that,
\begin{align}
&\sup_{F \in \Delta([0,1])} \sum_{k=0}^{K} \int_{x_{k}}^{x_{k+1}} \Psi_k^{\pi_{\bm{n}}}(F(\pastI^-),F(\pastI),F(z))dz \nonumber \\
&\qquad = 
\sup_{\substack{
\bm{f}^{-},\bm{f} \in [0,1]^{K},\ \bm{f}^{+} \in [0,1]^{K+1} \\
0 \le f_0^{+} \le f_1^{-} \le f_1 \le f_1^{+} \le \cdots \le f_K^{-} \le f_K \le f_K^{+} \le 1
}}
\sum_{k=0}^{K} (x_{k+1}-x_k) \cdot \Psi_k^{\pi_{\bm{n}}}(\bm{f}^{-}, \bm{f}, f_k^{+}). \label{eq:reduction}
\end{align}
Throughout, we use the conventions $f_0 = 0$ and $f_{K+1}^{-} = 1$, so that every feasible triple $(\bm{f}^{-},\bm{f},\bm{f}^{+})$ satisfies $f_k \le f_k^{+} \le f_{k+1}^{-}$ for all $k \in \{0,\ldots,K\}$.

\textit{Step 1:}
Let $\epsilon >0$ small enough.
For every feasible $(\bm{f}^{-},\bm{f},\bm{f}^{+})$, define the function $F_{\epsilon}$ such that for every $z \in [0,1]$,
\begin{equation*}
F_{\epsilon}(z) = \begin{cases}
f_{k} \quad \text{if $z \in [x_{k},x_{k} + \epsilon)$ for some $k \in \{1,\ldots,K\}$}\\
f^{-}_{k+1} \quad \text{if $z \in [x_{k+1} - \epsilon,x_{k+1})$ for some $k \in \{0,\ldots,K-1\}$}\\
f^{+}_{k} \quad \text{for all other $z \in [x_{k},x_{k+1})$}\\
1 \quad \text{if $z = 1$.}  
\end{cases}
\end{equation*}
Note that $F_{\epsilon}$ is a cdf on $[0,1]$ because it is non-decreasing from $[0,1]$ to $[0,1]$ (by feasibility of $(\bm{f}^{-},\bm{f},\bm{f}^{+})$) and it is right-continuous. Moreover, $F_{\epsilon}(x_k^-) = f_k^{-}$ and $F_{\epsilon}(x_k) = f_k$ for every $k \in \{1,\ldots,K\}$, that is, $F_{\epsilon}(\pastI^-) = \bm{f}^{-}$ and $F_{\epsilon}(\pastI) = \bm{f}$.

Let $C > 0$ be an upper bound on the family of functions $(\Psi_k^{\pi_{\bm{n}}})_{k \in \{0,\ldots,K\}}$. For each $k \in \{0,\ldots,K\}$, let $E_k \subseteq [x_k,x_{k+1})$ be the set on which $F_{\epsilon}$ differs from $f_k^{+}$; by definition of $F_{\epsilon}$, it is the union of at most two intervals of length $\epsilon$. We have
\begin{align*}
\sum_{k=0}^{K} \int_{x_{k}}^{x_{k+1}} \Psi_{k}^{\pi_{\bm{n}}}(F_{\epsilon}(\pastI^-),F_{\epsilon}(\pastI),F_{\epsilon}(z))dz 
&= \sum_{k=0}^{K} (x_{k+1}-x_k) \cdot \Psi_k^{\pi_{\bm{n}}}(\bm{f}^{-}, \bm{f}, f_k^{+}) \\
 &\qquad  + \sum_{k=0}^{K} \int_{E_k} \left( \Psi_{k}^{\pi_{\bm{n}}}(\bm{f}^{-},\bm{f},F_{\epsilon}(z)) - \Psi_{k}^{\pi_{\bm{n}}}(\bm{f}^{-},\bm{f},f^{+}_k) \right) dz \\
&\geq  \sum_{k=0}^{K} (x_{k+1}-x_k) \cdot \Psi_k^{\pi_{\bm{n}}}(\bm{f}^{-}, \bm{f}, f_k^{+}) - 2(K+1) \cdot C \cdot \epsilon ,
 \end{align*}
where the last inequality holds  because for every $k \in \{0,\ldots,K\}$, the function $0 \leq \Psi_k^{\pi_{\bm{n}}} \leq C$ and $E_k$ has Lebesgue measure at most $2\epsilon$.

Consequently, for every $\epsilon > 0$,
 \begin{align*}
&\sup_{F \in \Delta([0,1])} \sum_{k=0}^{K} \int_{x_{k}}^{x_{k+1}} \Psi_k^{\pi_{\bm{n}}}(F(\pastI^-),F(\pastI),F(z))dz \\
&\qquad \geq \sum_{k=0}^{K} \int_{x_{k}}^{x_{k+1}} \Psi_{k}^{\pi_{\bm{n}}}(F_{\epsilon}(\pastI^-),F_{\epsilon}(\pastI),F_{\epsilon}(z))dz\\ 
&\qquad \geq  \sum_{k=0}^{K} (x_{k+1}-x_k) \cdot \Psi_k^{\pi_{\bm{n}}}(\bm{f}^{-}, \bm{f}, f_k^{+}) - 2(K+1) \cdot C \cdot \epsilon.
\end{align*}
By taking a supremum over feasible $(\bm{f}^{-},\bm{f},\bm{f}^{+})$ and sending $\epsilon$ to $0$, we obtain,
\begin{align*}
&\sup_{F \in \Delta([0,1])} \sum_{k=0}^{K} \int_{x_{k}}^{x_{k+1}} \Psi_k^{\pi_{\bm{n}}}(F(\pastI^-),F(\pastI),F(z))dz \\
&\qquad \geq
\sup_{\substack{
\bm{f}^{-},\bm{f} \in [0,1]^{K},\ \bm{f}^{+} \in [0,1]^{K+1} \\
0 \le f_0^{+} \le f_1^{-} \le f_1 \le f_1^{+} \le \cdots \le f_K^{-} \le f_K \le f_K^{+} \le 1
}}
\sum_{k=0}^{K} (x_{k+1}-x_k) \cdot \Psi_k^{\pi_{\bm{n}}}(\bm{f}^{-}, \bm{f}, f_k^{+}).
\end{align*}

\textit{Step 2:} To prove the reverse inequality, we note that,
\begin{align}
    &\sup_{F \in \Delta([0,1])} \sum_{k=0}^{K} \int_{x_{k}}^{x_{k+1}} \Psi_k^{\pi_{\bm{n}}}(F(\pastI^-),F(\pastI),F(z))dz \nonumber \\
    &\qquad = 
    \sup_{\substack{ \bm{f}^{-},\bm{f} \in [0,1]^{K} \\ 0 \le f_1^{-} \le f_1 \le \cdots \le f_K^{-} \le f_K \le 1 }} 
    \sup_{\substack{ F \in \Delta([0,1]) \\ F(\pastI^-) = \bm{f}^{-},\ F(\pastI) = \bm{f} }} \sum_{k=0}^{K} \int_{x_{k}}^{x_{k+1}} \Psi_k^{\pi_{\bm{n}}}(F(\pastI^-),F(\pastI),F(z))dz  \nonumber  \\
    &\qquad = \sup_{\substack{ \bm{f}^{-},\bm{f} \in [0,1]^{K} \\ 0 \le f_1^{-} \le f_1 \le \cdots \le f_K^{-} \le f_K \le 1 }} 
    \sup_{\substack{ F \in \Delta([0,1]) \\ F(\pastI^-) = \bm{f}^{-},\ F(\pastI) = \bm{f} }} \sum_{k=0}^{K} \int_{x_{k}}^{x_{k+1}} \Psi_k^{\pi_{\bm{n}}}(\bm{f}^{-},\bm{f},F(z))dz. \label{eq:sup_decomposition}
\end{align}
Let $(\bm{f}^{-},\bm{f}) \in [0,1]^{K} \times [0,1]^{K}$ satisfy $0 \le f_1^{-} \le f_1 \le \cdots \le f_K^{-} \le f_K \le 1$, and let $F \in \Delta ([0,1])$ such that $F(\pastI^-) = \bm{f}^{-}$ and $F(\pastI) = \bm{f}$. Fix $k \in \{0,\ldots,K\}$, and note that for every $z \in [x_{k},x_{k+1})$, we have
\begin{equation*}
\label{eq:pointwise_bound}
\Psi_k^{\pi_{\bm{n}}}(\bm{f}^{-},\bm{f},F(z)) \leq \sup_{f_{k}^{+} \in [F(x_{k}),F(x_{k+1}^-)]} \Psi_{k}^{\pi_{\bm{n}}}(\bm{f}^{-},\bm{f},f_{k}^{+}) \leq \sup_{f^{+}_{k} \in [f_k,f_{k+1}^{-}]} \Psi_{k}^{\pi_{\bm{n}}}(\bm{f}^{-},\bm{f},f_{k}^{+}),
\end{equation*}
where the first inequality holds because $F$ is non-decreasing hence $F(z) \in [F(x_{k}),F(x_{k+1}^-)]$, and the second because $[F(x_{k}),F(x_{k+1}^-)] \subseteq [f_k,f_{k+1}^{-}]$ under the conventions $f_0 = 0$ and $f_{K+1}^{-} = 1$ (with equality for $k \in \{1,\ldots,K-1\}$). 

This implies that for every such $(\bm{f}^{-},\bm{f})$, 
\begin{align*}
  &\sup_{\substack{ F \in \Delta([0,1]) \\ F(\pastI^-) = \bm{f}^{-},\ F(\pastI) = \bm{f} }} \sum_{k=0}^{K}\int_{x_{k}}^{x_{k+1}} \Psi_k^{\pi_{\bm{n}}}(\bm{f}^{-},\bm{f},F(z))dz \\
  &\qquad \leq \sum_{k=0}^{K}   \sup_{\substack{ F \in \Delta([0,1]) \\ F(\pastI^-) = \bm{f}^{-},\ F(\pastI) = \bm{f} }}  \int_{x_{k}}^{x_{k+1}} \Psi_k^{\pi_{\bm{n}}}(\bm{f}^{-},\bm{f},F(z))dz \\
  &\qquad \leq \sum_{k=0}^{K} \sup_{f^{+}_{k} \in [f_k,f_{k+1}^{-}]}  (x_{k+1} - x_{k}) \cdot \Psi_{k}^{\pi_{\bm{n}}}(\bm{f}^{-},\bm{f},f_{k}^{+})\\
  &\qquad =  \sup_{\substack{ \bm{f}^{+} \in [0,1]^{K+1} \\ f^{+}_k \in [f_k,f_{k+1}^{-}] \, \text{for all} \, k \in \{0,\ldots,K\} }}  \sum_{k=0}^{K} (x_{k+1} - x_{k}) \cdot \Psi_{k}^{\pi_{\bm{n}}}(\bm{f}^{-},\bm{f},f_{k}^{+}).
\end{align*}
By combining this inequality to \eqref{eq:sup_decomposition}, we conclude that,
\begin{align*}
        &\sup_{F \in \Delta([0,1])} \sum_{k=0}^{K} \int_{x_{k}}^{x_{k+1}} \Psi_k^{\pi_{\bm{n}}}(F(\pastI^-),F(\pastI),F(z))dz \\
        &\qquad \leq 
        \sup_{\substack{ \bm{f}^{-},\bm{f} \in [0,1]^{K} \\ 0 \le f_1^{-} \le f_1 \le \cdots \le f_K^{-} \le f_K \le 1 }} 
        \sup_{\substack{ \bm{f}^{+} \in [0,1]^{K+1} \\ f^{+}_k \in [f_k,f_{k+1}^{-}] \, \text{for all} \, k }}  \sum_{k=0}^{K} (x_{k+1} - x_{k}) \cdot \Psi_{k}^{\pi_{\bm{n}}}(\bm{f}^{-},\bm{f},f_{k}^{+})\\
        &\qquad = \sup_{\substack{
\bm{f}^{-},\bm{f} \in [0,1]^{K},\ \bm{f}^{+} \in [0,1]^{K+1} \\
0 \le f_0^{+} \le f_1^{-} \le f_1 \le f_1^{+} \le \cdots \le f_K^{-} \le f_K \le f_K^{+} \le 1
}}
\sum_{k=0}^{K} (x_{k+1}-x_k) \cdot \Psi_k^{\pi_{\bm{n}}}(\bm{f}^{-}, \bm{f}, f_k^{+}).
\end{align*}
Combining the two steps, we have established that \eqref{eq:reduction} holds. It is straightforward to check that $x_1 = 0$ or $x_K = 1$ lead to the corresponding conditions $f_1^{-} = 0$ or $f_K = 1$, respectively.
\end{proof}

\section{Proofs for Results in \Cref{sec:SAA_reduction}}

\begin{proof}[\textbf{Proof of \Cref{lem:action_dist_BSAA}}]
Fix a $k \in [K]$ and $z \in [x_k, x_{k+1})$. Define
\begin{align*}
    \mathcal S_{k+1}(z)\ \coloneqq\ \left| \left\{ (\ell, i) \mid \ell \geq k+1,\ S^{(\ell)}_i \leq z \right\} \right|
\end{align*}
to be the number of historical sales samples larger than $z$ which were observed for some historical inventory $x_\ell \geq x_{k+1}$. Then, $\mathcal S_{k+1}(z) \leq \mathcal S_{k+1}(1) = \sum_{\ell=k+1}^K n_\ell = n - \sigma_k$.

As $\hat F_{\mathrm BSAA}$ is the empirical distribution of the sales data $\{\ss{k}_i\}$, note that
\begin{align*}
    \inf\left\{ u \in [0,1] \mid \hat F_{\mathrm BSAA}(u) \geq q \right\} \leq z \qquad \iff \qquad \left| \left\{ (\ell,i) \mid \ss{\ell}_i \leq z \right\}\right| \geq \ceil{qn}\,.
\end{align*}

Next, observe that $\ss{\ell}_i \leq x_\ell \leq x_k$ for all $\ell \leq k$ and $i \in [n_\ell]$. Consequently, all of the $\sigma_k = \sum_{\ell=1}^k n_\ell$ samples corresponding to the inventory levels $\{x_1, \dots \leq x_k\}$ are \emph{always} less than $z$. Therefore, we must have 
\begin{align}\label{eq:inter_1_action_dist_BSAA}
    \BSAA_{\bm n}(\bm x, \bm I) \leq z \qquad \iff \qquad \mathcal S_{k+1}(z) \geq \ceil{qn} - \sum_{\ell = 1}^k n_\ell\,.
\end{align}

If $\sigma_k > \ceil{qn}$, then we always have $\BSAA_{\bm n}(\bm x, \bm I) \leq z$ and the lemma holds trivially. Suppose $\sigma_k \leq \ceil{qn}$. We can now evaluate the desired probability:
\begin{align*}
     \Prob_{\bm{D} \sim F^n}\big( \pi_{\bm{n}}(\pastI,\bm{I}) \le z \big) &=  \Prob_{\bm{D} \sim F^n}\left( \mathcal S_{k+1}(z) \geq \ceil{qn} - \sum_{\ell = 1}^k n_\ell \right)\\
     &= \sum_{j = \ceil{qn} - \sigma_k}^{n - \sigma_k}\ \Prob_{\bm{D} \sim F^n}\left( \mathcal S_{k+1}(z) = j \right)\\
     &= \sum_{j = \ceil{qn} - \sigma_k}^{n - \sigma_k}\ \binom{n-\sigma_k}{j} \cdot F(z)^j \cdot (1 - F(z))^j\,,
\end{align*}
where the last equality follows from the fact that $\Prob(\ss{\ell}_i \leq z) = \Prob(D^{(\ell)}_i \leq z)$ because $z < x_\ell$ for all $\ell \geq k+1$ and therefore
\begin{align*}
    \mathds{1}(\ss{\ell}_i \leq z) = \mathds{1}(\D^{(\ell)}_i \leq z \text{ or } x_\ell \leq z) = \mathds{1}(\D^{(\ell)}_i \leq z)\,.
\end{align*}
Plugging in the definition of the Bernstein polynomial $B_{\lceil q n \rceil - \sigma_k,   n - \sigma_k}$ completes the proof.
\end{proof}

\subsection{Proof of \Cref{thm:BSAA}}

\paragraph{Structural properties of $\psibsaa_k$.} Observe that $\psibsaa_k(v) = 0$ whenever $v = q$. Thus the critical fractile $q$ bifurcates the functions $\psibsaa_k$ into two pieces, one on $[0,q]$ and the other on $[q,1]$. Now, since we are concerned with a maximization problem, it would be desirable for these pieces to be concave. Although the functions are typically not concave, we show that both pieces of $\psibsaa_k$ are strictly log-concave.

\begin{lemma}
\label{lem:BSAA-log-concavity}
    For all $0 \leq k \leq K$ with $\sum_{\ell=1}^k n_\ell < \ceil{qn}$, we have
    \begin{enumerate}
        \item $\psibsaa_k$ is strictly log concave when restricted to $[0,q]$.
        \item $\psibsaa_k$ is strictly log concave when restricted to $[q,1]$.
    \end{enumerate}
\end{lemma}

\begin{proof}[\textbf{Proof of \Cref{lem:BSAA-log-concavity}}]
Fix a $k \in \{0,1,\dots,K\}$ and set $\sigma_k = \sum_{\ell=1}^k n_\ell < \ceil{qn}$. 

We first prove part 1; part 2 follows analogously.   Then, for every $v \leq q$, we get
\begin{align*}
        \psibsaa_k(v) &= \big( 1 - B_{\lceil q n \rceil - \sigma_k,   n - \sigma_k} \big(v\big) \big)\cdot (v-q)
        + (q - v)_+\\
        &= (q - v) \cdot B_{\lceil q n \rceil - \sigma_k,   n - \sigma_k} \big(v\big)\\
        &= (q - v) \cdot \left(1 - \underbrace{\sum_{j=0}^{\ceil{qn} - \sigma_k - 1} \binom{n - \sigma_k}{j} \cdot v^j \cdot (1 - v)^{n-\sigma_k-j}}_{H(v)} \right)
\end{align*}
Note that $H(v)$ is the probability of a $\text{Binomial}(n - \sigma_k, v)$ random variable being less than or equal to $\ceil{qn} - \sigma_k - 1$. Therefore, we get that $H(v) = 1 - \Prob(X \leq v)$ for $X \sim \text{Beta}(\ceil{qn} - \sigma_k, n - \ceil{qn} + 1)$, i.e., $1 - H$ is the CDF of the $\text{Beta}(\ceil{qn} - \sigma_k, n - \ceil{qn} + 1)$ distribution. Thus, we can use the fact that the CDF of a $\text{Beta}(\alpha, \beta)$-distribution is log concave whenever $\alpha \geq 1$ and $\beta \geq 1$ \citep{bagnoli2005log}. Furthermore, it is easy to see that $v \mapsto (q - v)$ is strictly log concave. Combining these facts allows us to conclude that $\psibsaa_k$ is strictly log-concave on $[0,q]$, because it is the product of a strictly log-concave function $v \mapsto (q - v)$ and the log-concave function $\psibsaa_k$ (which is positive on $(0,q)$).

For part 2, when $v \geq q$, we get
\begin{align*}
        \psibsaa_k(v) &= \big(1 - B_{\lceil q n \rceil - \sigma_k,   n - \sigma_k} \big(v\big) \big)\cdot (v-q)
        + (q - v)_+\\
        &= (v - q) \cdot (1 - B_{\lceil q n \rceil - \sigma_k,   n - \sigma_k} \big(v\big))\\
        &= (v - q) \cdot \sum_{j=0}^{\ceil{qn} - \sigma_k - 1} \binom{n - \sigma_k}{j} \cdot v^j \cdot (1 - v)^{n-\sigma_k-j}
\end{align*}
The complementary CDF $(1 - H)$ of a $\text{Beta}(\alpha, \beta)$-distribution is also log concave whenever $\alpha \geq 1$ and $\beta \geq 1$ \citep{bagnoli2005log}. Hence, the argument from part 1 implies $\psibsaa_k$ is strictly log concave on $[q,1]$.
\end{proof}

\begin{figure}[t]
\centering
\begin{tikzpicture}
\begin{axis}[
    width=0.95\linewidth,
    height=0.4\linewidth,
    xlabel={$t$},
    ylabel={$\psi_k(t)$},
    xmin=0.2, xmax=1,
    ymin=0, ymax=0.018,
    grid=both,
    legend style={at={(0.02,0.98)},anchor=north west,draw=none,fill=none},
    legend cell align=left,
    unbounded coords=discard,
    mark=none,              % ensures no markers at data points
    smooth,
]

\addplot+[paper line,mark=none, smooth, paperBlue] table [col sep=comma, x=t, y=psi0] {Data/psi_curves.csv};
\addlegendentry{$k=0$}

\addplot+[paper line,mark=none, smooth, paperOrange] table [col sep=comma, x=t, y=psi1] {Data/psi_curves.csv};
\addlegendentry{$k=1$}

\addplot+[paper line,mark=none, smooth, paperRed] table [col sep=comma, x=t, y=psi2] {Data/psi_curves.csv};
\addlegendentry{$k=2$}

\addplot+[paper line,mark=none, smooth, paperViolet] table [col sep=comma, x=t, y=psi3] {Data/psi_curves.csv};
\addlegendentry{$k=3$}

\addplot+[paper line,mark=none, smooth, paperGreen] table [col sep=comma, x=t, y=psi4] {Data/psi_curves.csv};
\addlegendentry{$k=4$}

\end{axis}
\end{tikzpicture}

\caption{$\psibsaa_k(t)$ functions for $K=5$, $q=0.9$, and sample allocation $\bm n=(1,1,1,1,15)$.}
\label{fig:bsaa_psi}
\end{figure}

Log-concavity implies that each $\psibsaa_k$ has a unique maximizer on $(0,q)$ (respectively on $(q,1)$) in the non-degenerate regime. The next lemma shows that these maximizers can only decrease as $k$ increases. \Cref{fig:bsaa_psi} depicts a particular instantiation of the $\{\psibsaa_k\}_k$ functions to highlight the structure we aim to capture in these lemmas.

\begin{lemma}
\label{lem:montonic_modes_SAA}
    For each $0 \leq k \leq K$, define
    \begin{align*}
        v_{k} \coloneqq \inf\left( \argmax_{v \in [0,q]}\ \psibsaa_k(v) \right) \qquad \text{and} \qquad w_{k} \coloneqq \inf\left( \argmax_{w \in [q,1]}\ \psibsaa_k(w) \right)\,, 
    \end{align*}
    Then, the sequences $(v_{k})_{k \in \{0,\ldots,K\}}$ and $(w_{k})_{k \in \{0,\ldots,K\}}$ are non-increasing.
\end{lemma}

\begin{proof}[\textbf{Proof of \Cref{lem:montonic_modes_SAA}}]

\noindent \textbf{Case 1:} $v \in [0,q]$ and the sequence $(v_{k})_{k \in \{0,\ldots,K\}}$.

If $q = 0$, the monotonicity of $(v_{k})_{k \in \{0,\ldots,K\}}$ is trivial; assume $q > 0$. Consider $k \in \{0,\ldots,K\}$. We can focus on $k$ such that $\sigma_k \coloneqq \sum_{\ell=1}^k n_\ell < \ceil{qn}$, because otherwise the unique maximizer of $\psibsaa_k$ is $v_{k} = 0$.

Let $H_{\alpha, \beta}$ denote the CDF of the $\text{Beta}(\alpha, \beta)$ distribution. Recall that in the proof of \Cref{lem:BSAA-log-concavity} we showed
\begin{align*}
    \psibsaa_k = (q - v) \cdot H_{\ceil{qn} - \sigma_k, n - \ceil{qn} + 1}(v)
\end{align*}
As $\psibsaa_k$ is strictly log concave (by \Cref{lem:BSAA-log-concavity}) and strictly positive on $(0,q)$, it must have a unique maximizer $v_{k} \in (0,q)$. Moreover, the maximizer $v_{k}$ must satisfy the following first order optimality condition
\begin{align*}
    \phi_{k}(v_{k}) = 0 \qquad \text{where} \qquad \phi_{k}(v) \coloneqq (\log \psibsaa_k(v))' = \frac{H'_{\ceil{qn} - \sigma_k, n - \ceil{qn} + 1}(v)}{H_{\ceil{qn} - \sigma_k, n - \ceil{qn} + 1}(v)} - \frac{1}{q-v} \,.
\end{align*}
    
Next, we argue that $\phi_{k-1}(v) > \phi_{k}(v)$ for every $v \in (0,q)$ and $k \geq 1$ with $\sigma_k < \ceil{qn}$. Observe that, as $H'_{\alpha,\beta}$ is the density of a $\text{Beta}(\alpha,\beta)$ distribution, there exists a constant $C > 0$ which only depends on $(\alpha, \beta)$, such that
    \begin{equation*}
            C \cdot v^{n_{k}} \cdot H'_{\ceil{qn} - \sigma_{k}, n - \ceil{qn} + 1}(v) = H'_{\ceil{qn} - \sigma_{k-1}, n - \ceil{qn} + 1}(v)
    \end{equation*}
Therefore, for $v \in (0,q)$, we get
\begin{align*}
    H_{\ceil{qn} - \sigma_{k-1}, n - \ceil{qn} + 1}(v) &= \int_0^v H'_{\ceil{qn} - \sigma_{k-1}, n - \ceil{qn} + 1}(t) dt \\
    &= C \cdot \int_0^v   t^{n_k} \cdot H'_{\ceil{qn} - \sigma_{k}, n - \ceil{qn} + 1}(t) \cdot dt\\
    &< C \cdot v^{n_{k}} \cdot \int_0^v H'_{\ceil{qn} - \sigma_{k}, n - \ceil{qn} + 1}(t) dt\\
    &= C \cdot v^{n_{k}} \cdot H_{\ceil{qn} - \sigma_{k}, n - \ceil{qn} + 1}(v)\,.
\end{align*}
Hence, the definition of $\phi_{k}$ implies $\phi_{k-1}(v) > \phi_{k}(v)$ for every $v \in (0,q)$. Additionally, note that $\phi_{k}$ is non-increasing because it is the derivative of a concave function. Combining these facts implies $v_{k-1} \geq v_{k}$: for contradiction, suppose $v_{k-1} < v_{k}$ and note
\begin{align*}
    0 = \phi_{k-1}(v_{k-1}) > \phi_{k}(v_{k-1}) \geq \phi_{k}(v_{k}) = 0,
\end{align*}
which is a contradiction. Therefore, the sequence $(v_{k})_k$ is non-increasing as desired.

\noindent \textbf{Case 2:} $v \in [q,1]$ and the sequence $(w_{k})_{k \in \{0,\ldots,K\}}$.

If $q = 1$, the monotonicity of $(w_{k})_{k \in \{0,\ldots,K\}}$ is trivial; assume $q < 1$. Consider $k \in \{0,\ldots,K\}$. We can focus on $k$ such that $\sigma_k \coloneqq \sum_{\ell=1}^k n_\ell < \ceil{qn}$, because otherwise $\psibsaa_k = 0$ on $[q,1]$ and $w_k = q$.

Let $\bar H_{\alpha, \beta}$ denote the complementary CDF of the $\text{Beta}(\alpha, \beta)$ distribution. Recall that in the proof of \Cref{lem:BSAA-log-concavity} we showed
\begin{align*}
    \psibsaa_k = (v - q) \cdot \bar H_{\ceil{qn} - \sigma_k, n - \ceil{qn} + 1}(v)
\end{align*}
As $\psibsaa_k$ is strictly log concave (by \Cref{lem:BSAA-log-concavity}) and strictly positive on $(q,1)$, it must have a unique maximizer $w_{k} \in (q,1)$. Moreover, the maximizer $w_{k}$ must satisfy the following first order optimality condition
\begin{align*}
    \phi_{k}(w_{k}) = 0 \qquad \text{where} \qquad \phi_{k}(v) \coloneqq (\log \psibsaa_k(v))' = - \frac{ H'_{\ceil{qn} - \sigma_k, n - \ceil{qn} + 1}(v)}{\bar H_{\ceil{qn} - \sigma_k, n - \ceil{qn} + 1}(v)} + \frac{1}{v-q} \,.
\end{align*}
    
Next, we argue that $\phi_{k-1}(v) > \phi_{k}(v)$ for every $v \in (q,1)$ and $k \geq 1$ with $\sigma_k < \ceil{qn}$. As $H'_{\alpha,\beta}$ is the density of a $\text{Beta}(\alpha,\beta)$ distribution, there exists a constant $C > 0$ which only depends on $(\alpha, \beta)$, such that
    \begin{equation*}
            C \cdot v^{n_{k}} \cdot H'_{\ceil{qn} - \sigma_{k}, n - \ceil{qn} + 1}(v) = H'_{\ceil{qn} - \sigma_{k-1}, n - \ceil{qn} + 1}(v)
    \end{equation*}
Therefore, for $v \in (q,1)$, we get
\begin{align*}
    \bar H_{\ceil{qn} - \sigma_{k-1}, n - \ceil{qn} + 1}(v) &= \int_v^1 H'_{\ceil{qn} - \sigma_{k-1}, n - \ceil{qn} + 1}(t) dt \\
    &= C \cdot \int_v^1   t^{n_k} \cdot H'_{\ceil{qn} - \sigma_{k}, n - \ceil{qn} + 1}(t) \cdot dt\\
    &> C \cdot v^{n_{k}} \cdot \int_v^1 H'_{\ceil{qn} - \sigma_{k}, n - \ceil{qn} + 1}(t) dt\\
    &= C \cdot v^{n_{k}} \cdot \bar H_{\ceil{qn} - \sigma_{k}, n - \ceil{qn} + 1}(v)\,.
\end{align*}
Hence, the definition of $\phi_{k}$ implies $\phi_{k-1}(v) > \phi_{k}(v)$ for every $v \in (q,1)$. Additionally, note that $\phi_{k}$ is non-increasing because it is the derivative of a concave function. Combining these facts implies $w_{k-1} \geq w_{k}$: for contradiction, suppose $w_{k-1} < w_{k}$ and note
\begin{align*}
    0 = \phi_{k-1}(w_{k-1}) > \phi_{k}(w_{k-1}) \geq \phi_{k}(w_{k}) = 0,
\end{align*}
which is a contradiction. Therefore, the sequence $(w_{k})_k$ is non-increasing as desired.
\end{proof}

\paragraph{Putting it together.} Corollary~\ref{cor:BSAA-piecewise-linear} reduced the evaluation of worst-case regret to the maximization of the separable sum $\sum_k (x_{k+1}-x_k) \cdot \psibsaa_k(f_k^+)$ over non-decreasing sequences $\bm f^+$. However, if we ignore the constraint that $\bm f^+$ needs to be non-decreasing, \Cref{lem:montonic_modes_SAA} shows that the optimizers of the individual functions $\psibsaa_k(f_k^+)$ are actually \emph{non-increasing}. This tendency of the maximization to push against the constraints causes them to become binding. In fact, they are maximally binding and the maximizer of the joint maximization problem has identical coordinates, as the following lemma shows.

\begin{lemma}
\label{lem:reduction_1D_unimodal}
Let $C \subseteq \mathbb R$ be a compact interval and $(g_i)_{i=1}^N$ be a family of continuous unimodal functions with $g_i: C\to\mathbb R$.
Moreover, assume that there exist maximizers $o_i^* \in \argmax_{o \in C} g_i(o)$ which are non-increasing, i.e., $o_1^*\ge o_2^*\ge \cdots \ge o_N^*$.
Then,
\[
\sup_{\substack{\bm v\in C^N\\ v_1\le\cdots\le v_N}}\ \sum_{i=1}^N g_i(v_i)
\;=\;
\sup_{v\in C}\ \sum_{i=1}^N g_i(v).
\]
\end{lemma}

\begin{proof}[\textbf{Proof of \Cref{lem:reduction_1D_unimodal}}]
Let $\mathcal V := \{\bm v\in C^N:\ v_1\le\cdots\le v_N\}$. Since $C^N$ is compact and $\mathcal V$
is closed, $\mathcal V$ is compact by the Bolzano-Weierstrass Theorem.  Furthermore, the objective
\[
F(\bm v):=\sum_{i=1}^N g_i(v_i)
\]
is continuous, hence it attains a maximum on $\mathcal V$ by Weierstrass Extreme Value Theorem. Let
\[
\mathcal V^* := \argmax_{\bm v\in \mathcal V} F(\bm v),
\]
be the set of optimal solutions, which is nonempty and compact.

Fix $i$ and $o_i^*\in\argmax_K g_i$. Unimodality of $g_i$ implies
\begin{align}
x\le y\le o_i^* &\implies g_i(x)\le g_i(y), \label{eq:mono_left}\\
o_i^*\le x\le y &\implies g_i(x)\ge g_i(y). \label{eq:mono_right}\,.
\end{align}

\textbf{Tie–break among optimizers.}
Define the continuous functional
\[
D(\bm v):=\sum_{i=1}^N |v_i-o_i^*|.
\]
Since $\mathcal V^*$ is compact, there exists $\bm v^*\in\mathcal V^*$ minimizing $D$ over $\mathcal V^*$.

We claim that $\bm v^*$ must be constant, i.e. $v_1^*=\cdots=v_N^*$. Assume for contradiction that $\bm v^*$ is not constant. Then there exists an index $i$ such that
\[
v_i^*<v_{i+1}^*.
\]
We first note that either $v_i^*<o_i^*$ or $v_{i+1}^*>o_{i+1}^*$. Indeed, if $v_i^*\ge o_i^*$ then
\[
v_{i+1}^*>v_i^*\ge o_i^*\ge o_{i+1}^*,
\]
so $v_{i+1}^*>o_{i+1}^*$. We now treat the two cases.

\medskip
\noindent \textbf{Case 1: $v_i^*<o_i^*$.}
Set
\[
t:=\min\{o_i^*,\,v_{i+1}^*\}\in C,
\qquad
\tilde v_j :=
\begin{cases}
v_j^* & j\neq i,\\
t & j=i.
\end{cases}
\]
Then $\tilde{\bm v}\in\mathcal V$ because $v_{i-1}^*\le v_i^*\le t\le v_{i+1}^*$ and all other
inequalities are unchanged.

Moreover, $v_i^*\le t\le o_i^*$, so by \eqref{eq:mono_left},
\[
g_i(t)\ge g_i(v_i^*),
\]
hence
\[
F(\tilde{\bm v}) \;=\; F(\bm v^*) - g_i(v_i^*) + g_i(t)\;\ge\; F(\bm v^*).
\]
Since $\bm v^*$ already maximizes $F$ on $\mathcal V$, we have $F(\tilde{\bm v})=F(\bm v^*)$, so
$\tilde{\bm v}\in\mathcal V^*$.

Finally, because $v_i^*<t\le o_i^*$, we get
\[
|t-o_i^*| = o_i^*-t \;<\; o_i^*-v_i^* = |v_i^*-o_i^*|.
\]
All other coordinates are unchanged, so $D(\tilde{\bm v})<D(\bm v^*)$, contradicting the choice of
$\bm v^*$ as a minimizer of $D$ over $\mathcal V^*$.

\medskip

\noindent \textbf{Case 2: $v_{i+1}^*>o_{i+1}^*$.}
Set
\[
t:=\max\{o_{i+1}^*,\,v_i^*\}\in C,
\qquad
\tilde v_j :=
\begin{cases}
v_j^* & j\neq i+1,\\
t & j=i+1.
\end{cases}
\]
Then $\tilde{\bm v}\in\mathcal V$ because $v_i^*\le t\le v_{i+1}^*\le v_{i+2}^*$ (when $i+2\le N$;
the endpoint cases are trivial).

Also, $o_{i+1}^*\le t\le v_{i+1}^*$, so by \eqref{eq:mono_right},
\[
g_{i+1}(t)\ge g_{i+1}(v_{i+1}^*),
\]
and the same argument shows $\tilde{\bm v}\in\mathcal V^*$.

Since $o_{i+1}^*\le t < v_{i+1}^*$, we have
\[
|t-o_{i+1}^*| = t-o_{i+1}^* \;<\; v_{i+1}^*-o_{i+1}^* = |v_{i+1}^*-o_{i+1}^*|,
\]
so again $D(\tilde{\bm v})<D(\bm v^*)$, a contradiction.

Both cases contradict the minimality of $D(\bm v^*)$, so no such $i$ can exist. Therefore
$v_1^*=\cdots=v_N^*$ and the lemma holds.
\end{proof}

Finally, we put everything together to get the desired characterization of worst-case regret. In particular, \Cref{cor:BSAA-piecewise-linear} states that the worst-case regret is given by a separable weighted sum of $(\psibsaa_k(f_k^+))_{k}$ over non-decreasing $\bm f^+$. \Cref{lem:BSAA-log-concavity} in turn states that each function $\psibsaa_k$ is the amalgamation of two log-concave functions, one on $[0,q]$ and the other on $[q,1]$. Finally, \Cref{lem:montonic_modes_SAA} implies that the optimization of these log-concave (and therefore unimodal) pieces over non-decreasing vectors results in constant vectors being the optimal. Therefore, in order to find the worst-case regret one only needs to determine which functions $\psibsaa_k$ to maximize on $[0,q]$ and which ones to maximize on $[q,1]$, with the maximization on each piece reducing to a simple 1-dimensional search.

\begin{proof}[\textbf{Proof of \Cref{thm:BSAA}}]
    First, we apply \Cref{cor:BSAA-piecewise-linear} to write
    \begin{equation*}
        \mathsf{Reg}(\BSAA_{\bm{n}}  \mid \pastI)
        =
        \sup_{\substack{
        \bm{f}^{+} \in [0,1]^{K+1} \\
        0 \le f_0^{+} \le f_1^{+} \le \cdots \le f_K^{+} \le 1
        }}
        \sum_{k=0}^{K} (x_{k+1}-x_k) \cdot \psibsaa_k(f_k^{+})\,,
    \end{equation*}

    Define $\bar k = \max\{k \in \{0,\ldots,K\} \mid \sum_{\ell=1}^k n_\ell < \ceil{qn}\}$. Then, the definition of $\psibsaa$ (as given in \Cref{cor:BSAA-piecewise-linear}, implies that $\psibsaa_k(v) = (q - v)_+$ for all $k > \bar k$.

    Next, note that \Cref{lem:BSAA-log-concavity} implies that $\psibsaa_k$ is strictly log concave on both $[0,q]$ and $[q,1]$ for all $k \leq \bar k$. As log-concave functions are quasiconcave, $\psibsaa_k$ is unimodal on both $[0,q]$ and $[q,1]$ for all $0 \leq k \leq K$ (e.g., see Section~3.4.2 of \citealt{boyd2004convex}). Moreover, \Cref{lem:montonic_modes_SAA} implies that there exist non-increasing sequences $(v_{k})_i$ and $(w_{k})_i$ such that
    \begin{align*}
        v_{k} \in \argmax_{v \in [0,q]}\ \psibsaa_k(v) \qquad \text{and} \qquad w_{k} \in \argmax_{w \in [q,1]}\ \psibsaa_k(w)\,. 
    \end{align*}
    Therefore, \Cref{lem:reduction_1D_unimodal} applies with $g_k(t) = (x_{k+1} - x_k) \cdot \psibsaa_k(t)$.

    Now, consider any feasible solution $\bm{f}^{+} \in [0,1]^{K+1}$ with the index $-1 \leq i \leq K$ defined to satisfy
    \begin{align*}
        0 \le f_0^{+} \le \dots \leq f_i^+ \leq q < f_{i+1}^+ \le \dots \leq f_K^{+} \le 1\,.
    \end{align*}
    The first block is empty when $i=-1$, and the second is empty when $i=K$; empty blocks in the display are omitted and empty sums are zero.
    In other words, $f_k^+ \in [0,q]$ for all $k \leq i$ and $f_k^+ \in [q,1]$ for all $k > i$. Thus, by \Cref{lem:reduction_1D_unimodal}, applied to each nonempty block, there exist $v \in [0,q]$ and $w \in [q,1]$ such that
    \begin{align*}
        \sum_{k=0}^{K} (x_{k+1}-x_k) \cdot \psibsaa_k(f_k^{+}) \ \leq \ \sum_{k=0}^i (x_{k+1}  - x_k) \cdot \psibsaa_k(v) + \sum_{k=i+1}^K (x_{k+1}  - x_k) \cdot \psibsaa_k(w)\,.
    \end{align*}
    As this is true for all feasible $\bm f^+$, we get
    \begin{align*}
        \mathsf{Reg}(\BSAA_{\bm{n}}  \mid \pastI)
        =
        \max_{i \in \{-1,\,0, \dots, K\}} \sup_{\substack{
       v \in [0,q], \\ w \in [q,1]}} \quad
        \sum_{k=0}^i (x_{k+1}  - x_k) \cdot \psibsaa_k(v) + \sum_{k=i+1}^K (x_{k+1}  - x_k) \cdot \psibsaa_k(w)\,,
    \end{align*}
    thereby establishing the theorem.
\end{proof}

\section{Proof of results in \Cref{sec:censoring_aware}}

\subsection{Censoring-aware SAA and Kaplan--Meier}

\begin{proof}[\textbf{Proof of \Cref{prop:ca_piecewise}}]
Fix a censoring-aware SAA policy and $k\in\{0,\ldots,K\}$ with $x_k<x_{k+1}$. We must construct a continuous function $P_k^{\pi_{\bm n}}$ such that
\begin{equation}
\label{eq:ca_piecewise}
\Prob_{\bm D\sim F^n}\big(\pi_{\bm n}(\pastI,\bm I)\le z\big)
=P_k^{\pi_{\bm n}}\big(F(\pastI^-),F(\pastI),F(z)\big)
\end{equation}
for every $F$ and every $z\in[x_k,x_{k+1})$.

\smallskip
\noindent\emph{Step 1:}
We first express $\mathds{1}\{\pi_{\bm n}(\pastI,\bm I) \le z\}$ as a fixed function $g_k$ of binary comparisons. For $z\in[x_k,x_{k+1})$, the policy orders at most $z$ if and only if its empirical score $\widehat H_{\bm n}(z)$ reaches its target. This follows because the score is a nondecreasing, right-continuous step function. 

Moreover, we note that if an observation was collected at $x_r\le x_k$, its sale is automatically at most $z$. If $r>k$, its sale is at most $z$ exactly when its demand is at most $z$. Hence
\begin{equation}\label{eq:ca_local_crossing}
\{\pi_{\bm n}(\pastI,\bm I)\le z\}
=\left\{
\sum_{r=1}^{k}\sum_{i=1}^{n_r}w_i^{(r)}(\bm H_{\pastI})
+\sum_{r=k+1}^{K}\sum_{i=1}^{n_r}
 w_i^{(r)}(\bm H_{\pastI})\mathds{1}\{D_i^{(r)}\le z\}
\ge\tau(\bm H_{\pastI})
\right\}.
\end{equation}
Let $\bm B(z)=(\mathds{1}\{D_i^{(r)}\le z\}\mid r>k,\ i\in[n_r])$. Once $\bm H_{\pastI}$ and $\bm B(z)$ are given, every weight, the target, and the score in \eqref{eq:ca_local_crossing} are determined. There is therefore a function $g_k$, taking values in $\{0,1\}$, such that
\[
\mathds{1}\{\pi_{\bm n}(\pastI,\bm I)\le z\}
=g_k\big(\bm H_{\pastI},\bm B(z)\big).
\]
The same $g_k$ works for every $z$ in the interval and does not depend on $F$.

\smallskip
\noindent\emph{Step 2: Construct $P_k^{\pi_{\bm n}}$ using only the required CDF values.}
Fix $F$ and $z\in[x_k,x_{k+1})$. Consider the following coupling. Let $(U_i^{(r)})$ be independent uniform random variables on $(0,1)$ and set $D_i^{(r)}=F^{-1}(U_i^{(r)})$, where $F^{-1}$ is the lower quantile function. The resulting demand history has distribution $F^n$. For each observation and each inventory level, almost surely,
\[
\begin{aligned}
\mathds{1}\{D_i^{(r)}<x_j\}
&=\mathds{1}\{U_i^{(r)}\le F(x_j^-)\},\\
\mathds{1}\{D_i^{(r)}\le x_j\}
&=\mathds{1}\{U_i^{(r)}\le F(x_j)\},\\
\mathds{1}\{D_i^{(r)}\le z\}
&=\mathds{1}\{U_i^{(r)}\le F(z)\}.
\end{aligned}
\]
There are only finitely many observations and thresholds, so all these equalities hold simultaneously on an event of probability one. 

For any $(\bm u^-,\bm u,v)\in[0,1]^K\times[0,1]^K\times[0,1]$, define
\[
\begin{aligned}
\bm H^U(\bm u^-,\bm u)
&=\left(
\mathds{1}\{U_i^{(r)}\le u_j^-\},\,
\mathds{1}\{U_i^{(r)}\le u_j\}
\;\middle|\;
r\in[K],\ i\in[n_r],\ j\in[r]
\right),\\
\bm B^U(v)
&=\left(\mathds{1}\{U_i^{(r)}\le v\}
\;\middle|\; r>k,\ i\in[n_r]\right),
\end{aligned}
\]
At the CDF arguments, the simultaneous equalities above give
\begin{equation}\label{eq:ca_joint_comparisons}
\big(\bm H_{\pastI},\bm B(z)\big)
=\big(\bm H^U(F(\pastI^-),F(\pastI)),\bm B^U(F(z))\big)
\qquad\text{almost surely}.
\end{equation}
Extend $g_k$ by zero on binary arrays that cannot arise from any demand history for any $z\in[x_k,x_{k+1})$, and define
\begin{equation}\label{eq:ca_uniform_representation}
P_k^{\pi_{\bm n}}(\bm u^-,\bm u,v)
=\E_{\bm U}\!\left[
 g_k\big(\bm H^U(\bm u^-,\bm u),\bm B^U(v)\big)
\right].
\end{equation}
This function takes values in $[0,1]$ on the entire cube. Since the coupled demand history has law $F^n$, \eqref{eq:ca_joint_comparisons} and Step~1 give
\[
\begin{aligned}
P_k^{\pi_{\bm n}}\big(F(\pastI^-),F(\pastI),F(z)\big)
&=\E_{\bm D\sim F^n}\!\left[g_k\big(\bm H_{\pastI},\bm B(z)\big)\right]\\
&=\Prob_{\bm D\sim F^n}\big(\pi_{\bm n}(\pastI,\bm I)\le z\big),
\end{aligned}
\]
proving \eqref{eq:ca_piecewise}.

\smallskip
\noindent\emph{Step 3: Prove continuity on the full cube.}
Evaluate \eqref{eq:ca_uniform_representation} at two arguments using the same uniforms. For any $a,b\in[0,1]$,
\[
\Prob\big(\mathds{1}\{U\le a\}\ne\mathds{1}\{U\le b\}\big)=|a-b|.
\]
The value of $g_k$ can change only if at least one of its comparison entries changes. Since each CDF coordinate is used in at most $n$ comparisons, a union bound gives
\[
\begin{aligned}
&\left|P_k^{\pi_{\bm n}}(\bm u^-,\bm u,v)
-P_k^{\pi_{\bm n}}(\widetilde{\bm u}^-,\widetilde{\bm u},\widetilde v)\right|\\
&\qquad\le n\left(
\|\bm u^--\widetilde{\bm u}^-\|_1
+\|\bm u-\widetilde{\bm u}\|_1
+|v-\widetilde v|
\right).
\end{aligned}
\]
Therefore $P_k^{\pi_{\bm n}}$ is continuous. Together with \eqref{eq:ca_piecewise}, this proves piecewise-separability.
\end{proof}

\begin{proof}[\textbf{Proof of \Cref{lem:ca_km_saa}}]
\begin{samepage}
Fix a demand history, write $\bm h=\bm H_{\pastI}$, and use the ordered observations $(Y_m,\zeta_m)_{m=1}^n$ from product-limit formula. We first verify that the proposed weights are well defined. Since $S_i^{(r)}=\min\{D_i^{(r)},x_r\}$,
\[
h_{i,r}^{(r)}=\delta_i^{(r)},
\qquad
1-h_{i,j}^{(r)}=\mathds{1}\{S_i^{(r)}>x_j\}
\quad\text{for }j<r.
\]
\end{samepage}
Consequently,
\[
C_j(\bm h)=\sum_{m:\,Y_m=x_j}(1-\zeta_m),
\qquad
R_j(\bm h)=\sum_{m=1}^n\mathds{1}\{Y_m>x_j\}.
\]
For every factor in \eqref{eq:ca_km_weights}, its index satisfies $j<r$ and $h_{i,j}^{(r)}=0$, so
\[
R_j(\bm h)
=\sum_{s=j+1}^K\sum_{\ell=1}^{n_s}\bigl(1-h_{\ell,j}^{(s)}\bigr)
\ge 1-h_{i,j}^{(r)}=1>0.
\]
Consequently, every weight is a well-defined, nonnegative function of the allowed comparisons.

It remains to show that the weights in \eqref{eq:ca_km_weights} reproduce the Kaplan--Meier estimate. Write $C_j=C_j(\bm h)$ and $R_j=R_j(\bm h)$, and define
\[
Q_0=1,\qquad
Q_m=\prod_{\ell=1}^m
\left(\frac{n-\ell}{n-\ell+1}\right)^{\zeta_\ell},
\qquad
\omega_m=Q_{m-1}-Q_m.
\]
Then
\[
\omega_m
=Q_{m-1}\left[1-\left(\frac{n-m}{n-m+1}\right)^{\zeta_m}\right]
=\frac{\zeta_m}{n-m+1}Q_{m-1}.
\]
To express $Q_{m-1}$ using only censored observations, first telescope:
\[
\prod_{\ell=1}^{m-1}\frac{n-\ell}{n-\ell+1}
=\frac{n-m+1}{n}.
\]
For $m=1$, this is the empty-product identity. Separating the factors with $\zeta_\ell=0$ therefore gives
\[
\begin{aligned}
Q_{m-1}
&=\left(\prod_{\ell=1}^{m-1}\frac{n-\ell}{n-\ell+1}\right)
\prod_{\substack{\ell<m\\\zeta_\ell=0}}
\frac{n-\ell+1}{n-\ell}
=\frac{n-m+1}{n}
\prod_{\substack{\ell<m\\\zeta_\ell=0}}
\frac{n-\ell+1}{n-\ell},\\
\omega_m
&=\frac{\zeta_m}{n}
\prod_{\substack{\ell<m\\\zeta_\ell=0}}
\frac{n-\ell+1}{n-\ell}.
\end{aligned}
\]
Every denominator $n-\ell$ in these products is positive because $\ell<m\le n$.

\begin{samepage}
Suppose $\zeta_m=1$. Uncensored observations precede censored observations at the same sales value, so the censored observations with index $\ell<m$ are exactly those at levels $x_j<Y_m$.

At such a level with $C_j>0$, the censored observations occupy positions $n-R_j-C_j+1,\ldots,n-R_j$. The corresponding factors in $\omega_m$ telescope as
\[
\prod_{\ell=n-R_j-C_j+1}^{n-R_j}\frac{n-\ell+1}{n-\ell}
=\frac{R_j+C_j}{R_j}
=1+\frac{C_j}{R_j}.
\]
\end{samepage}
If $C_j=0$, the same identity holds with an empty product. Grouping the censored factors by their inventory level yields
\[
\omega_m
=\frac1n\prod_{j\in[K]:\,x_j<Y_m}\left(1+\frac{C_j}{R_j}\right)
\qquad\text{when }\zeta_m=1,
\]
whereas $\omega_m=0$ when $\zeta_m=0$. For the observation $(r,i)$ at position $m$, we have
\[
\zeta_m=h_{i,r}^{(r)},
\qquad
\{j\in[K]:x_j<Y_m\}=\{1\le j<r:h_{i,j}^{(r)}=0\}.
\]
Hence $\omega_m=w_i^{(r)}(\bm h)$, as defined in \eqref{eq:ca_km_weights}.

For $0\le z<1$, start from the Kaplan--Meier estimate in its product-limit form. The observations with $Y_m\le z$ form an
initial segment of the ordered sample, including all ties at $z$.
Telescoping and substituting the weights just identified gives
\begin{equation}\label{eq:ca_km_score}
\begin{aligned}
\KMF(z)
&=1-\prod_{m:\,Y_m\le z}
\left(\frac{n-m}{n-m+1}\right)^{\zeta_m}
= \sum_{m:\,Y_m\le z}\omega_m\\
&=\sum_{r=1}^K\sum_{i=1}^{n_r}
w_i^{(r)}(\bm H_{\pastI})\mathds{1}\{S_i^{(r)}\le z\}.
\end{aligned}
\end{equation}
The empirical score in \Cref{def:censoring_aware_saa} therefore agrees with $\KMF$ below $1$. Taking the target to be $q$ gives the same lower quantile as the Kaplan--Meier policy; if the score never reaches $q$ below $1$, both policies choose $1$. This proves the lemma.
\end{proof}

\begin{proof}[\textbf{Proof of \Cref{thm:ca_km_reduction}}]
\begin{samepage}
We first express $\KMF(z)$ using only comparisons at $x_1,\ldots,x_k$ and $z$. Fix $z\in[x_k,x_{k+1})$ and a demand history. For $r\in[K]$ and $i\in[n_r]$, write $\bm h$ and $\bm b(z)$ for the arrays with entries
\[
\begin{aligned}
h_{i,j}^{(r)}&=\mathds{1}\{D_i^{(r)}\le x_j\},
&&1\le j\le\min\{r,k\},\\
b_i^{(r)}(z)&=\mathds{1}\{D_i^{(r)}\le z\},
&&r>k,
\end{aligned}
\]
\end{samepage}
For $j\in[k]$, the counts in \eqref{eq:ca_km_counts} can be written as
\[
\begin{aligned}
C_j(\bm h)&=\sum_{i=1}^{n_j}(1-h_{i,j}^{(j)}),\\
R_j(\bm h)&=\sum_{r=j+1}^{k}\sum_{i=1}^{n_r}(1-h_{i,j}^{(r)})
+\sum_{r=k+1}^{K}\sum_{i=1}^{n_r}(1-h_{i,j}^{(r)}).
\end{aligned}
\]
In particular, samples from groups $r>k$ enter $R_j$ through
$h_{i,j}^{(r)}=\mathds{1}\{D_i^{(r)}\le x_j\}$: the comparison threshold solely involves $x_j$, with $j\le k$.

We now substitute the weights in \eqref{eq:ca_km_weights} into \eqref{eq:ca_km_score}. If $r\le k$, then $S_i^{(r)}\le x_r\le z$, and every factor in $w_i^{(r)}$ already has index $j<r\le k$. If $r>k$, then $z<x_r$ and
$\delta_i^{(r)}\mathds{1}\{S_i^{(r)}\le z\}
=b_i^{(r)}(z).$
Moreover, when $b_i^{(r)}(z)=1$,
\[
D_i^{(r)}\le z<x_j
\quad\Longrightarrow\quad
\mathds{1}\{D_i^{(r)}\le x_j\}=1
\qquad\text{for every }k<j\le r.
\]
Thus none of the factors indexed by $j>k$ enters the weight of an observation contributing to $\KMF(z)$. 

Consequently, the event that the KM policy orders at most $z$ is
determined by $\bm h$ and $\bm b(z)$. Applying the uniform construction
in the proof of \Cref{prop:ca_piecewise} to this reduced comparison tuple
gives continuous action-CDF functions that depend only on
$F(x_1),\ldots,F(x_k)$ and $F(z)$. In particular,
$P_k^{\KM_{\bm n}}(\bm u,v)$ can be chosen independent of
$u_{k+1},\ldots,u_K$ on the entire cube.

Apply \Cref{thm:general_reduction} with this representation, viewed as independent of the left-limit argument. The objective is independent of $\bm f^-$, and the remaining constraints are
\[
0\le f_0^+\le f_1\le f_1^+\le\cdots\le f_K\le f_K^+\le1:
\]
any such pair $(\bm f,\bm f^+)$ can be completed by choosing $f_j^-=f_j$. If $x_1=0$, first set $f_0^+=0$ and choose $f_1^-=0$; this changes only the zero-length term. If $x_K=1$, set $f_K=f_K^+=1$; every preceding term is unchanged because its action-CDF function ignores the $K$th coordinate. These observations give the stated optimization. The objective is continuous on a compact feasible set, so its supremum is a maximum.
\end{proof}

\subsection{Proof of \Cref{prop:ca_other_reductions}}\label{app:ca_literature}

For FCZ, we express the policy directly as a censoring-aware SAA rule. For RCN$^+$, we construct a randomized policy with the same worst-case regret, then use its piecewise-separable action CDF to derive a finite-dimensional regret formula.

\subsubsection{The Fan--Chen--Zhou policy}\label{app:ca_fcz}

\begin{proof}[Proof for FCZ]

We prove that \eqref{eq:ca_fcz_policy} is a censoring-aware SAA policy by expressing its weights directly as functions of $\bm H_{\pastI}$. For a possible value $\bm h$ of this tuple, write $h_{i,j}^{(r),<}$ for the coordinate corresponding to $\mathds{1}\{D_i^{(r)}<x_j\}$, for $r\in[K]$, $i\in[n_r]$, and $j\in[r]$. Set $h_{i,0}^{(r),<}=0$ and define
\begin{equation}\label{eq:ca_fcz_weights}
 w_i^{(r)}(\bm h)
 =\sum_{j=1}^r
 \frac{h_{i,j}^{(r),<}\bigl(1-h_{i,j-1}^{(r),<}\bigr)}{M_j},
 \qquad
 \tau(\bm h)=q'.
\end{equation}
These are fixed functions of the allowed comparisons. Every denominator is positive because $M_j=\sum_{r=j}^K n_r\ge n_j>0$, and all weights are nonnegative. We show that the associated weighted empirical score equals $\widehat H_{\mathrm{FCZ}}$ for every demand history.

Fix an observation collected at $x_r$ and an index $j\le r$. Since $x_j\le x_r$, a sale strictly below $x_j$ reveals the demand exactly. Consequently,
\begin{align*}
\mathds{1}\{x_{j-1}\le S_i^{(r)}<x_j\}
=\mathds{1}\{x_{j-1}\le D_i^{(r)}<x_j\}
=h_{i,j}^{(r),<}\bigl(1-h_{i,j-1}^{(r),<}\bigr).
\end{align*}
For $j=1$, the last equality uses $D_i^{(r)}\ge0=x_0$, which explains the convention $h_{i,0}^{(r),<}=0$. Thus the product in \eqref{eq:ca_fcz_weights} identifies precisely whether observation $(r,i)$ belongs to interval $j$.
Using this identity and interchanging the finite sums, for every $z\in[0,1)$ we obtain
\begin{align*}
\sum_{r=1}^K\sum_{i=1}^{n_r}
 w_i^{(r)}(\bm H_{\pastI})\mathds{1}\{S_i^{(r)}\le z\}
&=\sum_{r=1}^K\sum_{i=1}^{n_r}\sum_{j=1}^r
 \frac{\mathds{1}\{x_{j-1}\le S_i^{(r)}<x_j\}
 \mathds{1}\{S_i^{(r)}\le z\}}{M_j}\\
&=\sum_{j=1}^K\frac1{M_j}\sum_{r=j}^K\sum_{i=1}^{n_r}
 \mathds{1}\{x_{j-1}\le S_i^{(r)}<x_j,\ S_i^{(r)}\le z\}\\
&=\widehat H_{\mathrm{FCZ}}(z).
\end{align*}
Taking the target to be $q'$ therefore reproduces exactly the lower crossing in \eqref{eq:ca_fcz_policy}; if no crossing occurs below $1$, both definitions prescribe $1$. This proves that the FCZ policy is C-SAA, so piecewise-separability follows from \Cref{prop:ca_piecewise}.

Denote by $P_k^{\pi_{\bm n}^{\mathrm{FCZ}}}$ the resulting action-CDF functions. Applying \Cref{thm:general_reduction} gives
\begin{equation}\label{eq:ca_fcz_reduction}
\reg(\pi_{\bm n}^{\mathrm{FCZ}}\mid\pastI)
=(c_u+c_o)\max_{\bm f^-,\bm f,\bm f^+}
\sum_{k=0}^K(x_{k+1}-x_k)
\Psi_k^{\pi_{\bm n}^{\mathrm{FCZ}}}(\bm f^-,\bm f,f_k^+),
\end{equation}
where $\Psi_k$ is defined as in \Cref{thm:general_reduction}, and the maximum is over $\bm f^-,\bm f\in[0,1]^K$ and $\bm f^+\in[0,1]^{K+1}$ satisfying
\[
0\le f_0^+\le f_1^-\le f_1\le f_1^+\le\cdots\le f_K^-\le f_K\le f_K^+\le1.
\]
As in \Cref{thm:general_reduction}, fix $f_1^-=0$ if $x_1=0$ and
$f_K=1$ if $x_K=1$.
There are at most $3K+1$ continuous variables, independently of the sample size.
\end{proof}

\subsubsection{Hssaine--Sinclair: a randomized policy with the same worst-case regret}
\label{app:ca_hs}

The randomized surrogate for RCN$^+$ combines a C-SAA policy and its
capped version. We first extend \Cref{prop:ca_piecewise} to such mixtures.

\begin{lemma}[Randomization between C-SAA policies]\label{lem:ca_randomization}
Let $\pi_1,\pi_2$ be C-SAA policies for the same inventory design, and
let $\eta(\bm H_{\pastI})\in[0,1]$ be a fixed function of the comparison
tuple. A policy that, conditional on the data, selects $\pi_1$ with
probability $\eta(\bm H_{\pastI})$ and $\pi_2$ otherwise is
piecewise-separable. The conclusion also holds if either rule is
replaced by $\min\{\pi_\ell,x_j\}$ for a fixed historical inventory $x_j$.
\end{lemma}

\begin{proof}[\textbf{Proof of \Cref{lem:ca_randomization}}]
Let $\bar\pi$ denote the randomized policy. Conditioning on the data
and applying the tower property gives, for $z\in[x_k,x_{k+1})$,
\[
\begin{aligned}
\Prob_F(\bar\pi\le z)
&=\E_F\big[\Prob(\bar\pi\le z\mid\bm I)\big]\\
&=\E_F\left[
\eta(\bm H_{\pastI})\mathds{1}\{\pi_1\le z\}
+(1-\eta(\bm H_{\pastI}))\mathds{1}\{\pi_2\le z\}\right].
\end{aligned}
\]
For $\ell\in\{1,2\}$, the proof of \Cref{prop:ca_piecewise} writes
$\mathds{1}\{\pi_\ell\le z\}
=g_{k,\ell}(\bm H_{\pastI},\bm B(z))$.
Use its uniform comparison arrays
$\bm H^U=\bm H^U(\bm u^-,\bm u)$ and $\bm B^U=\bm B^U(v)$.
Extend $\eta$ by zero to binary comparison tuples that cannot occur.
For $(\bm u^-,\bm u,v)\in[0,1]^{2K+1}$, define
\[
\begin{aligned}
P_k^{\bar\pi}(\bm u^-,\bm u,v)
=\E_{\bm U}\Big[&\eta(\bm H^U)g_{k,1}(\bm H^U,\bm B^U)
+(1-\eta(\bm H^U))g_{k,2}(\bm H^U,\bm B^U)\Big].
\end{aligned}
\]
At $(\bm u^-,\bm u,v)=(F(\pastI^-),F(\pastI),F(z))$, the joint
comparison identity from \Cref{prop:ca_piecewise} and the tower
calculation above give $P_k^{\bar\pi}=\Prob_F(\bar\pi\le z)$.
The same continuity argument as in \Cref{prop:ca_piecewise} applies,
since the integrand takes values in $[0,1]$ and changes only when a
comparison changes.

Finally, capping $\pi_\ell$ at $x_j$ replaces its crossing indicator by
one when $j\le k$ and leaves it unchanged when $j>k$.
The same construction therefore proves the capped case.
\end{proof}

\begin{proof}[Proof for RCN$^+$]
Write $\pi=\pi^{\mathrm{RCN+}}_{\bm n}$ and $\lambda=x_K$.
We construct a randomized policy with the same worst-case regret,
and apply \Cref{lem:ca_randomization,thm:general_reduction} to it.
For the fixed confidence radii $\zeta_1,\ldots,\zeta_K>0$, recall that
\[
W_r=\frac1{n_r}\sum_{i=1}^{n_r}\mathds{1}\{S_i^{(r)}<x_r\},
\qquad
\widehat{\mathcal U}=\{r\in[K]:W_r\ge q+\zeta_r\}.
\]
When $\widehat{\mathcal U}\ne\emptyset$, let
$N_{\widehat{\mathcal U}}=\sum_{r\in\widehat{\mathcal U}}n_r$ and set
\[
\widehat a_{\widehat{\mathcal U}}
=\inf\left\{z\in[0,1]:
\frac1{N_{\widehat{\mathcal U}}}\sum_{r\in\widehat{\mathcal U}}
\sum_{i=1}^{n_r}\mathds{1}\{S_i^{(r)}\le z\}\ge q\right\}.
\]
Define $\theta(w)=(q-w)/(1-w)$ for $w<q$ and $\theta(w)=0$ for
$w\ge q$. Let $V$ be uniform on $(0,1)$, independently of the
observations and future demand, and define
\begin{equation}\label{eq:ca_rcnrandom_policy}
\widetilde\pi(\pastI,\bm I,V)=
\begin{cases}
\widehat a_{\widehat{\mathcal U}},&\widehat{\mathcal U}\ne\emptyset,\\
\lambda,&\widehat{\mathcal U}=\emptyset,\ W_K\ge q-\zeta_K,\\
\lambda+(1-\lambda)\mathds{1}\{V\le\theta(W_K)\},
&\widehat{\mathcal U}=\emptyset,\ W_K<q-\zeta_K.
\end{cases}
\end{equation}
We abbreviate the induced actions by $\pi$ and $\widetilde\pi$.
Expectations $\E_F$ and probabilities $\Prob_F$ include both the
historical demands drawn from $F$ and the independent seed $V$.

\smallskip
\noindent\emph{Step 1: Equality of worst-case regrets.}
For
\[
\mathcal F_\lambda
=\{F\in\Delta([0,1]):\Prob_F(\lambda<D<1)=0\},
\]
we will prove
\begin{equation}\label{eq:ca_hs_second_flat}
\reg(\pi\mid\pastI)
=\reg(\widetilde\pi\mid\pastI)
=\sup_{F\in\mathcal F_\lambda}\E_F[R(\widetilde\pi,F)].
\end{equation}
The two policies coincide on their first two branches, and on the
last branch their conditional means agree:
\[
\E[\widetilde\pi\mid\bm I]
=\lambda+(1-\lambda)\theta(W_K)=\pi.
\]
Thus $\E[\widetilde\pi\mid\bm I]=\pi$ on every branch.

For $F\in\mathcal F_\lambda$ and $a\in[\lambda,1]$, the expected cost is
\[
c(a,F)=c_o\int_{[0,\lambda]}(a-d)\,dF(d)
       +c_u(1-a)(1-F(\lambda)),
\]
which is affine in $a$. On the randomizing branch, it follows that
$\E[c(\widetilde\pi,F)\mid\bm I]
=c(\E[\widetilde\pi\mid\bm I],F)=c(\pi,F)$.
On the other branches the policies coincide. Therefore, for every $F\in\mathcal F_\lambda$,
\begin{equation}\label{eq:ca_hs_second_equal}
\E_F[R(\pi,F)]=\E_F[R(\widetilde\pi,F)].
\end{equation}

We next prove that restricting nature to $\mathcal F_\lambda$ does
not reduce the worst-case regret of $\widetilde\pi$:
\begin{equation}\label{eq:ca_hs_flat_restriction}
\reg(\widetilde\pi\mid\pastI)
=\sup_{F\in\mathcal F_\lambda}\E_F[R(\widetilde\pi,F)].
\end{equation}
To establish the nontrivial inequality, we construct, for every $F$,
a distribution in $\mathcal F_\lambda$ with at least as much expected
regret under $\widetilde\pi$.
If $\lambda=1$, the policies coincide and the support restriction is
vacuous. Suppose now that $\lambda<1$. For any $F$ and
$h\in\{F(\lambda^-),1\}$, define the CDF
\[
F^h(z)=
\begin{cases}
F(z),&0\le z<\lambda,\\
h,&\lambda\le z<1,\\
1,&z=1.
\end{cases}
\]
This moves all demand at or above $\lambda$ to $1$ when
$h=F(\lambda^-)$ and to $\lambda$ when $h=1$.
For every $r$ and $d\ge\lambda$,
\[
\min\{d,x_r\}=\min\{\lambda,x_r\}=\min\{1,x_r\}=x_r.
\]
Applying this transformation to each historical demand preserves
every sale and, using the same seed $V$, the action of $\widetilde\pi$.
In particular, for any $z\in[0,1]$,
\[
\Prob_{F^h}(\widetilde\pi\le z)
=\Prob_F(\widetilde\pi\le z).
\]
Moreover,
$\widetilde\pi\in[0,\lambda]\cup\{1\}$, so
$\Prob_F(\widetilde\pi>z)=\Prob_F(\widetilde\pi=1)$ for
$\lambda\le z<1$.
Choose
\[
h^*\in\argmax_{h\in\{F(\lambda^-),1\}}
\left\{\Prob_F(\widetilde\pi=1)(h-q)+(q-h)_+\right\}.
\]
The expression being maximized is convex in $h$, and
$F(z)\in[F(\lambda^-),1]$ on $[\lambda,1)$.
Since $F^{h^*}(z)=F(z)$ below $\lambda$ and the action laws agree,
\Cref{lem:integral_form} gives
\[
\begin{aligned}
&\E_{F^{h^*}}[R(\widetilde\pi,F^{h^*})]
 -\E_F[R(\widetilde\pi,F)]\\
&\quad=(c_u+c_o)\int_\lambda^1
\left[\Prob_F(\widetilde\pi=1)(h^*-F(z))
 +(q-h^*)_+-(q-F(z))_+\right]dz\ \ge0.
\end{aligned}
\]
Since $F^{h^*}\in\mathcal F_\lambda$, taking the supremum over $F$
proves the $\le$ inequality in \eqref{eq:ca_hs_flat_restriction}.
The reverse inequality follows from
$\mathcal F_\lambda\subseteq\Delta([0,1])$.
We can now make both comparisons needed for equality explicit:
\begin{equation*}
    \reg(\pi\mid\pastI)
\stackrel{(a)}{\le}\reg(\widetilde\pi\mid\pastI)
\stackrel{(b)}{=}\sup_{F\in\mathcal F_\lambda}\E_F[R(\widetilde\pi,F)]
\stackrel{(c)}{=}\sup_{F\in\mathcal F_\lambda}\E_F[R(\pi,F)]
\le\reg(\pi\mid\pastI),
\end{equation*}

where $(a)$ follows from conditional Jensen's inequality, since
$c(\cdot,F)$ is convex and $\E[\widetilde\pi\mid\bm I]=\pi$;
$(b)$ follows from \eqref{eq:ca_hs_flat_restriction}; and $(c)$ follows
from \eqref{eq:ca_hs_second_equal}.
Every inequality is consequently an equality, proving
\eqref{eq:ca_hs_second_flat}.

\smallskip
\noindent\emph{Step 2: Finite-dimensional characterization.}
We prove the regret formula \eqref{eq:ca_hs_second_reduction} by
establishing piecewise-separability of $\widetilde\pi$ and applying
\Cref{thm:general_reduction}.
Define the two policies
\[
\pi_1=
\begin{cases}
\widehat a_{\widehat{\mathcal U}},&\widehat{\mathcal U}\ne\emptyset,\\
1,&\widehat{\mathcal U}=\emptyset,
\end{cases}
\qquad
\pi_\lambda=
\begin{cases}
\widehat a_{\widehat{\mathcal U}},&\widehat{\mathcal U}\ne\emptyset,\\
\lambda,&\widehat{\mathcal U}=\emptyset.
\end{cases}
\]
The policy $\pi_1$ is C-SAA with target $q$ and weights
\[
w_i^{(r)}(\bm H_{\pastI})=
\begin{cases}
\mathds{1}\{r\in\widehat{\mathcal U}\}/N_{\widehat{\mathcal U}},
 &\widehat{\mathcal U}\ne\emptyset,\\
0,&\widehat{\mathcal U}=\emptyset.
\end{cases}
\]
Since every sale is at most $\lambda$, the pooled quantile is at most
$\lambda$, and hence $\pi_\lambda=\min\{\pi_1,\lambda\}$.
Conditional on the data, $\widetilde\pi$ selects $\pi_1$ with probability
$\mathds{1}\{\widehat{\mathcal U}=\emptyset,\ W_K<q-\zeta_K\}\theta(W_K)$
and $\pi_\lambda$ otherwise. This probability depends only on
$\bm H_{\pastI}$, since each $W_r$ is the average of the strict
comparisons $\mathds{1}\{S_i^{(r)}<x_r\}$.
\Cref{lem:ca_randomization} therefore implies that $\widetilde\pi$ is
piecewise-separable.

Let $P_k^{\widetilde\pi}$ and $\Psi_k^{\widetilde\pi}$ be the functions
from \Cref{def:piecewise,thm:general_reduction} associated with $\widetilde\pi$.
The weights and mixing probability use only the strict comparisons determining
$(W_r)$; on interval $k$, the remaining inputs are the local
comparisons $\mathds{1}\{D_i^{(r)}\le z\}$ for $r>k$.
The comparison construction can therefore omit the weak boundary
comparisons, so the functions $P_k^{\widetilde\pi}$ can be chosen
independent of their second vector argument. We suppress this unused
argument in both $P_k^{\widetilde\pi}$ and $\Psi_k^{\widetilde\pi}$ below.
Moreover, on the last interval,
\[
\Prob(\widetilde\pi\le z\mid\bm I)
=1-\mathds{1}\{\widehat{\mathcal U}=\emptyset,\ W_K<q-\zeta_K\}\theta(W_K),\qquad \lambda\le z<1,
\]
so $P_K^{\widetilde\pi}$ can also be chosen independent of its last
argument.

By \eqref{eq:ca_hs_second_flat} and \Cref{thm:general_reduction},
\begin{equation}\label{eq:ca_hs_second_reduction}
\reg(\pi\mid\pastI)
=(c_u+c_o)\max_{\bm f^-,\bm f^+}
\sum_{k=0}^K(x_{k+1}-x_k)
\Psi_k^{\widetilde\pi}(\bm f^-,f_k^+),
\end{equation}
where $\bm f^-\in[0,1]^K$ and $\bm f^+\in[0,1]^{K+1}$ satisfy
\[
\begin{gathered}
0\le f_0^+\le f_1^-,\qquad
f_k^-\le f_k^+\le f_{k+1}^-\quad(1\le k<K),\\
f_K^+\in\{f_K^-,1\}.
\end{gathered}
\]
As in \Cref{thm:general_reduction}, fix $f_1^-=0$ if $x_1=0$ and
$f_K^+=1$ if $x_K=1$.
Here the unused weak boundary coordinates $f_j$ have been eliminated
by setting $f_j=f_j^+$. The restriction on $f_K^+$ follows because
only the last summand depends on this coordinate, and it is convex
in $f_K^+$ since $P_K^{\widetilde\pi}$ is independent of its local
argument.
The maximum exists by continuity and compactness. The final coordinate
$f_K^+$ has only two choices once $\bm f^-$ is fixed; hence there are
at most $2K$ continuous variables.
\end{proof}

\section{Proof and Computational Details for \Cref{sec:new_policies}}
\label{sec:proof_new_policies}

\begin{proof}[\textbf{Proof of \Cref{thm:new_policy_design}}]
We first establish piecewise-separability, then apply \Cref{thm:general_reduction} to obtain the LP and its discretization guarantee.

\smallskip
\noindent\emph{Piecewise-separability.}
Fix $z\in[x_k,x_{k+1})$ and let $N(z)$ count the sales at most $z$. As in the proof of \Cref{prop:ca_piecewise},
\[
 N(z)=\sum_{\ell=1}^{k}n_\ell
       +\sum_{\ell=k+1}^{K}\sum_{i=1}^{n_\ell}
          \mathds{1}\{D_i^{(\ell)}\le z\}.
\]
Recall that $Y_1\le\cdots\le Y_n$ are the pooled sales order statistics, with $Y_0=0$ and $Y_{n+1}=1$. Since $Y_r\le z$ exactly when $r\le N(z)$, the conditional action CDF is
\[
 \Prob(\pi^{\bm\lambda}_{\bm n}\le z\mid\bm I)
 =\sum_{r=0}^{N(z)}\lambda_{H^{\le}_{\pastI},r}.
\]
Thus it depends only on the weak historical comparisons and the local comparisons with $z$.

We next use the uniform representation from the proof of \Cref{prop:ca_piecewise}: let the $U_i^{(\ell)}$ be independent uniforms on $(0,1)$ and define
\[
\begin{aligned}
 H^U(\bm u)
 &=\left(\sum_{i=1}^{n_\ell}\mathds{1}\{U_i^{(\ell)}\le u_j\}
       \;\middle|\;\ell\in[K],\ j\in[\ell]\right),\\
 N_k^U(v)
 &=\sum_{\ell=1}^{k}n_\ell
       +\sum_{\ell=k+1}^{K}\sum_{i=1}^{n_\ell}\mathds{1}\{U_i^{(\ell)}\le v\}.
\end{aligned}
\]
At $\bm u=F(\pastI)$ and $v=F(z)$, these reproduce the joint law of $H^{\le}_{\pastI}$ and $N(z)$. Consequently,
\begin{equation*}
 P_k^{\pi^{\bm\lambda}_{\bm n}}(\bm u,v)
 =\E_{\bm U}\!\left[\sum_{r=0}^{N_k^U(v)}\lambda_{H^U(\bm u),r}\right]
\end{equation*}
is the action-CDF function. For arguments outside the feasible CDF region, set $\lambda_{h,r}=0$ whenever $h\notin\mathcal H_{\pastI}$. This defines a $[0,1]$-valued function on the full cube. Continuity follows from the same argument as in \Cref{prop:ca_piecewise}.

Hence the policy is piecewise-separable, with no dependence on $F(\pastI^-)$. We suppress this unused argument in $P_k$ and $\Psi_k$ below.

\begin{samepage}
\smallskip
\noindent\emph{Applying the reduction and forming the LP.}
Because the action CDF does not depend on $\bm f^-$, \Cref{thm:general_reduction} gives
\begin{equation}\label{eq:new_policy_cdf_reduction}
 \reg(\pi^{\bm\lambda}_{\bm n}\mid\pastI)
 =(\cu+\co)\max_{\bm f,\bm f^+}
 \sum_{k=0}^K(x_{k+1}-x_k)
 \Psi_k^{\pi^{\bm\lambda}_{\bm n}}(\bm f,f_k^+),
\end{equation}
\end{samepage}
where the maximum is over CDF coordinates satisfying
\[
 0\le f_0^+\le f_1\le f_1^+\le\cdots\le f_K\le f_K^+\le1,
\]
with $f_K=f_K^+=1$ if $x_K=1$. %Every such pair can be completed to a feasible point of the master reduction by setting $f_j^-=f_{j-1}^+$.

To express these coordinates as the probabilities in \eqref{eq:new_policy_silp}, let $J$ be the number of historical inventories strictly below $1$ and set
\[
 p_0=f_0^+,\qquad
 p_{x_j}=f_j-f_{j-1}^+,\qquad
 p_{x_j^+}=f_j^+-f_j\quad(j\in[J]),\qquad
 p_1=1-f_J^+.
\]
This forms a bijection between the reduced feasible set and $\Delta_d$.

We now represent this reduced objective as an expected regret to identify its LP coefficients. For a fixed $\bm p$, draw $n$ independent historical labels from $\mathcal B_{\pastI}$ with probabilities $\bm p$. Each label carries two pieces of information. Its position in the order $0<x_1<x_1^+<\cdots<x_J<x_J^+<1$ determines the weak comparisons in $H^{\le}_{\pastI}$. Its numerical value determines sales and costs: both $x_j$ and $x_j^+$ have value $x_j$. In particular, $x_j^+$ is recorded as above $x_j$ in the comparison counts, although its numerical value is $x_j$. Form $Y_r$ by censoring these numerical values at their historical inventories and sorting the resulting sales.

Let $F_{\bm p}$ denote the distribution of numerical values, and let $\E_{\bm p}$ denote expectation over the sampled labels and any policy randomization. If we select $R$ with the policy's probabilities $\lambda_{H^{\le}_{\pastI},r}$, then, for $z\in(x_k,x_{k+1})$,
\[
 F_{\bm p}(z)=f_k^+,
 \qquad
 \Prob_{\bm p}(Y_R\le z)
 =P_k^{\pi^{\bm\lambda}_{\bm n}}(\bm f,f_k^+).
\]
The regret integral in the proof of \Cref{lem:integral_form} therefore identifies the objective in \eqref{eq:new_policy_cdf_reduction} with expected regret in this finite experiment. The contribution of rank $r$ on signature $h$ is
\begin{equation}\label{eq:new_policy_coefficients}
 a_{h,r}(\bm p)
 =\E_{\bm p}\!\left[
   \mathds{1}\{H^{\le}_{\pastI}=h\}\,R(Y_r,F_{\bm p})
 \right].
\end{equation}
Averaging over the rank selected by the policy gives
\begin{equation*}
 \reg(\pi^{\bm\lambda}_{\bm n}\mid\pastI)
 =\max_{\bm p\in\Delta_d}
 \underbrace{\sum_{h\in\mathcal H_{\pastI}}\sum_{r=0}^{n+1}
       a_{h,r}(\bm p)\lambda_{h,r}}_{L(\bm\lambda,\bm p)}.
\end{equation*}
The coefficients are continuous: this is a finite expectation, each label-history probability is a monomial in $\bm p$, and both the cost of a fixed action and the optimal cost under $F_{\bm p}$ are continuous in $\bm p$. Compactness of $\Delta_d$ and $\Lambda_{\pastI}$ therefore ensures attainment of the maximum and the policy minimum. Problem \eqref{eq:new_policy_silp} minimizes this maximum in epigraph form, proving \eqref{eq:new_policy_km_dominance}.

\smallskip
\noindent\emph{Finite-grid guarantee.}
Fix a censoring-aware randomized order-statistic policy associated with $\bm\lambda \in \Lambda_{\pastI}$ and a probability vector $\bm p \in \Delta_d$. Choose $\bm g\in\Gamma_\epsilon$ with $\frac12\|\bm p-\bm g\|_1\le\epsilon$. Let $M_c=\max\{\co,\cu\}$. For any function $\varphi$ of a single label with values in $[0,M_c]$,
\[
 \left|\sum_{b\in\mathcal B_{\pastI}}(p_b-g_b)\varphi(b)\right|
 \le \frac{M_c}{2}\|\bm p-\bm g\|_1
 \le M_c\epsilon.
\]

The expected cost of the policy can be computed by drawing the $n$ historical labels and one independent label for the demand against which the action is evaluated. Replace the distribution of these labels from $\bm p$ to $\bm g$ one at a time. In each replacement, averaging over all other labels and the policy's randomization leaves a function $\varphi$ in $[0,M_c]$. The displayed bound shows that each replacement changes expected cost by at most $M_c\epsilon$. Hence, by a telescoping argument, one can show that the total change in the policy's expected cost is at most $(n+1)M_c\epsilon$. For the optimal cost, the same single-label bound gives
\[
 |\opt(F_{\bm p})-\opt(F_{\bm g})|
 \le\sup_{a\in[0,1]}|c(a,F_{\bm p})-c(a,F_{\bm g})|
 \le M_c\epsilon.
\]
Subtracting optimal cost from policy cost yields
\begin{equation*}
 |L(\bm\lambda,\bm p)-L(\bm\lambda,\bm g)|
 \le(n+2)M_c\epsilon.
\end{equation*}
In particular, let $(\bm\lambda^\epsilon,t_\epsilon)$ solve $\mathrm{LP}(\Gamma_\epsilon)$. Then
\[
 L(\bm\lambda^\epsilon,\bm p)
 \le L(\bm\lambda^\epsilon,\bm g)+(n+2)M_c\epsilon
 \le t_\epsilon+(n+2)M_c\epsilon.
\]
Maximizing over $\bm p$ proves the upper bound in \eqref{eq:new_policy_finite_guarantee}. Restricting the adversary's constraints gives $t_\epsilon\le t^*$, while optimality over the policy class gives $t^*\le\reg(\pi^{\bm\lambda^\epsilon}_{\bm n}\mid\pastI)$, completing the proof.
\end{proof}

\subsection{Numerical Implementation}\label{sec:new_policy_computation}
We solve \eqref{eq:new_policy_silp} by constraint generation, maintaining a lower bound $\underline t$ on the optimal value $t^*$ and an upper bound $\overline t$ on the worst-case regret of a computed policy. The stopping criterion is $\overline t-\underline t\le10^{-6}$. Thus the reported accuracy concerns the original semi-infinite problem.

\smallskip
\noindent\emph{Policy optimization and a lower bound.}
For a finite set $\Gamma\subset\Delta_d$, we solve the LP obtained by restricting the regret constraints to $\bm p\in\Gamma$. The coefficients in \eqref{eq:new_policy_coefficients} are computed by finite sums over comparison counts. The LP returns a policy $\bm\lambda$ and dual weights on the demand vectors in $\Gamma$. These weights are nonnegative and sum to one, by the dual constraint associated with $t$ and we obtain a probability vector $\bm w$. Any such vector gives the lower bound
\[
 \underline t
 =\min_{\bm\mu\in\Lambda_{\pastI}}
      \sum_{\bm p\in\Gamma}w_{\bm p}L(\bm\mu,\bm p)
 =\sum_{h\in\mathcal H_{\pastI}}\min_{0\le r\le n+1}
      \sum_{\bm p\in\Gamma}w_{\bm p}a_{h,r}(\bm p)
 \le t^*.
\]
We evaluate this expression with exact rational arithmetic, so the lower bound does not rely on the LP solver's numerical optimality tolerance.

\smallskip
\noindent\emph{A global upper bound for the policy.}
The experiments use historical inventories $x$ and $1$, with boundary labels $\{0,x,x^+,1\}$. In the finite experiment defined above, numerical demand and the selected action belong to $\{0,x,1\}$. Since an optimal action under $F_{\bm p}$ also belongs to this set, the regret can be written as
\[
 L(\bm\lambda,\bm p)
 =\max_{a\in\{0,x,1\}}
   \left\{\E_{\bm p}[c(Y_R,F_{\bm p})]-c(a,F_{\bm p})\right\}.
\]
For each fixed $a$, the expression in braces is a polynomial in $\bm p$ of degree at most $n+1$, since policy cost averages over $n$ historical labels and one independent demand label. Thus evaluating worst-case regret reduces to maximizing three polynomials on the probability simplex.

We use Bernstein branch-and-bound to obtain an upper bound $\overline t$ within $10^{-7}$ of the largest evaluated regret. Exact coefficient initialization and outward-rounded arithmetic ensure that this bound remains valid despite numerical rounding.

\smallskip
\noindent\emph{Constraint generation and termination.}
If $\overline t-\underline t>10^{-6}$, we add violating demand vectors obtained in the global calculation to $\Gamma$ and solve the enlarged LP. We retain the best lower bound and the policy with the smallest upper bound across iterations. At termination, the returned policy satisfies
\[
 \underline t\le t^*\le
 \reg(\pi^{\bm\lambda}_{\bm n}\mid\pastI)
 \le\overline t,
 \qquad \overline t-\underline t\le10^{-6}.
\]
We initialize $\Gamma$ with the vectors whose coordinates are multiples of $1/16$ and use HiGHS for the finite LPs.

\subsection{The Uncensored Minimax Benchmark}\label{sec:new_policy_uncensored}

The uncensored setting corresponds to the historical design $\pastI=(1)$, so that $Y_r=D_{(r)}$. The following theorem characterizes the benchmark in \Cref{fig:new_policy_curves} using the binomial tails $B_{r,n}$ from \Cref{sec:SAA_reduction} and the regret integrand from \Cref{thm:general_reduction}.

\begin{theorem}\label{thm:new_policy_uncensored}
Let $n\ge1$ and $0<q<1$. For $r\in\{0,\ldots,n\}$ and $\lambda\in[0,1]$, let $\pi_n^{r,\lambda}$ select $D_{(r)}$ with probability $1-\lambda$ and $D_{(r+1)}$ with probability $\lambda$, with $D_{(0)}=0$ and $D_{(n+1)}=1$. Its action-CDF function and regret integrand are
\[
\begin{aligned}
 P_0^{\pi_n^{r,\lambda}}(v)
 &=(1-\lambda)B_{r,n}(v)+\lambda B_{r+1,n}(v),\\
 \Psi_0^{\pi_n^{r,\lambda}}(v)
 &=\bigl(1-P_0^{\pi_n^{r,\lambda}}(v)\bigr)(v-q)+(q-v)_+,
\end{aligned}
\]
where $B_{0,n} = 1$ and $B_{n+1,n} = 0$. The minimax regret with $n$ fully observed i.i.d. demands satisfies
\begin{equation}\label{eq:uncensored_additive_minimax}
 \inf_{\pi_n}\reg(\pi_n\mid1)
 =(\cu+\co)\min_{\substack{r\in\{0,\ldots,n\}\\\lambda\in[0,1]}}
      \max_{v\in[0,1]}\Psi_0^{\pi_n^{r,\lambda}}(v),
\end{equation}
where the infimum is over all measurable randomized data-driven policies.
\end{theorem}

\begin{proof}[\textbf{Proof of \Cref{thm:new_policy_uncensored}}]
For every $z\in[0,1)$, $D_{(r)}\le z$ if and only if at least $r$ observations are at most $z$, with the stated endpoint conventions. Hence
\[
 \Prob(\pi_n^{r,\lambda}\le z)
 =P_0^{\pi_n^{r,\lambda}}(F(z)).
\]
Applying \Cref{thm:general_reduction} with $x_1=1$ gives $\reg(\pi_n^{r,\lambda}\mid1)=(\cu+\co)\max_{v\in[0,1]}\Psi_0^{\pi_n^{r,\lambda}}(v)$. The distribution $F_p=p\delta_0+(1-p)\delta_1$ has $F_p(z)=p$ on $[0,1)$ and attains the corresponding value of the integrand. Hence a worst-case distribution for each policy $\pi_n^{r,\lambda}$ belongs to this Bernoulli family.

\smallskip
\noindent\emph{Optimality among all policies.}
Fix any randomized policy $\pi_n$ and restrict demand to $F_p$. Let $N$ count the zero demands, so that $N\sim\operatorname{Bin}(n,p)$, and write $\E_p$ for expectation under this binomial law. Conditional on $N=j$, all arrangements of the sample are equally likely, independently of $p$. Let $a_j\in[0,1]$ be the policy's mean action over these arrangements and its private randomization. Because newsvendor cost under $F_p$ is affine in the action, its expected regret is
\[
 \E_{\bm D\sim F_p^n}[R(\pi_n,F_p)]
 =(\cu+\co)\bigl[(p-q)\E_p[a_N]+(q-p)_+\bigr].
\]
For any $r\in\{0,\ldots,n\}$ and $\lambda\in[0,1]$, let $a_j^*$ denote the mean action of $\pi_n^{r,\lambda}$ conditional on observing exactly $j$ zeros. By definition of $\pi_n^{r,\lambda}$ we have that,
\[
 a_j^*=\begin{cases}
 1,&j<r,\\
 \lambda,&j=r,\\
 0,&j>r,
 \end{cases}
 \qquad
 \E_q[a_N^*]=\Prob_q(N<r)+\lambda\Prob_q(N=r).
\]
We next show that we can choose $r$ and $\lambda$ so that $\E_q[a_N^*]=\E_q[a_N]$. Let $\mu=\E_q[a_N]\in[0,1]$ and define
\[
 r=\min\{j\in\{0,\ldots,n\}:\Prob_q(N\le j)\ge\mu\}.
\]
This set is nonempty because $\Prob_q(N\le n)=1$. By the definition of $r$,
$\Prob_q(N<r)\le\mu\le\Prob_q(N\le r)$. Set
\[
 \lambda
 =\frac{\E_q[a_N]-\Prob_q(N<r)}{\Prob_q(N=r)}.
\]
Since $0<q<1$, the denominator is positive, and this choice belongs to $[0,1]$ and satisfies $\E_q[a_N^*]=\E_q[a_N]$.

We now show that, for every $p\in[0,1]$,
\[
 \E_{F_p^n}[R(\pi_n^{r,\lambda},F_p)]
 \le \E_{F_p^n}[R(\pi_n,F_p)].
\]
By the regret expression above, this is equivalent to
$(p-q)\E_p[a_N^*-a_N]\le0$. It suffices, therefore, to prove
$\E_p[a_N^*]\ge\E_p[a_N]$ for $p<q$ and
$\E_p[a_N^*]\le\E_p[a_N]$ for $p>q$.

For $0<p<q$, the likelihood ratio
\[
 \ell_j(p)=\frac{\Prob_p(N=j)}{\Prob_q(N=j)}
 =\left(\frac p q\right)^j
  \left(\frac{1-p}{1-q}\right)^{n-j}
\]
is decreasing in $j$. Changing from the binomial law with parameter $p$ to that with parameter $q$, and using $\E_q[a_N^*-a_N]=0$, gives
\[
\begin{aligned}
 \E_p[a_N^*-a_N]
 &=\E_q[(a_N^*-a_N)\ell_N(p)]\\
 &=\E_q\!\left[(a_N^*-a_N)\bigl(\ell_N(p)-\ell_r(p)\bigr)\right].
\end{aligned}
\]
For $j<r$, both $a_j^*-a_j$ and $\ell_j(p)-\ell_r(p)$ are nonnegative; for $j>r$, both are nonpositive; and for $j=r$, the second factor is zero. Every term in the expectation is therefore nonnegative, proving $\E_p[a_N^*-a_N]\ge0$. For $q<p<1$, the likelihood ratio is increasing, so these products are nonpositive and $\E_p[a_N^*-a_N]\le0$. Both comparisons extend to $p=0,1$ by continuity. Multiplying by $p-q$ shows that $\pi_n^{r,\lambda}$ has no greater regret than $\pi_n$ under every $F_p$; both regrets vanish at $p=q$.

Since $\pi_n^{r,\lambda}$ has a Bernoulli worst-case distribution, its worst-case regret is at most that of $\pi_n$, proving \eqref{eq:uncensored_additive_minimax}. 
\end{proof}

\section{Optimal Exploratory Inventory Design for BSAA}\label{appendix:opt-bsaa-inventory}

In this section, we look at the meta problem where the decision maker can decide on the inventory design $\bm x$ to be used for exploration. In particular, we assume that there is global inventory budget $B \in \mathbb Z_+$ and the decision maker aims to solve the following iterated minimax optimal design problem:
\begin{equation}\label{eq:P_global}
\mathcal P(B)
\;\coloneqq\;
\inf_{\substack{K\ge 1,\ \bm n\in\mathbb N^K\\ 0\le x_1\le\cdots\le x_K\le 1\\ \sum_{k=1}^K n_k x_k \le B}}
\ \mathsf{Reg}(\BSAA_{\bm n}\mid \bm x).
\end{equation}

\eqref{eq:P_global} allows the decision maker to specify the inventories used for data collection: the number of distinct inventories $K$, the inventory levels $\bm x$ and the number of samples $n_k$ to be collected at the inventory level $x_k$. These decisions are made subject to the constraint that the total inventory used during data collection  $\sum_k n_k \cdot x_k$ is at most the inventory budget $B$, with the goal of minimizing the expected regret BSAA would incur from using this data. In this section, we make the assumption that $q \geq 0.5$.

Note that our model (\Cref{sec:model}) assumes distinct historical inventory levels, whereas \eqref{eq:P_global} allows $x_k = x_{k+1}$. The definitions of the information vector $\bm I$, of the policy $\BSAA_{\bm n}$, and of the worst-case regret in \eqref{eq:worst-case-regret} extend to any non-decreasing $\bm x \in [0,1]^K$. The following lemma shows that, for BSAA, allowing multiplicity in inventory levels is without loss of generality: repeated levels generate the same sales data as the design that pools them, and the characterization of \Cref{thm:BSAA} extends to allow for these repetitions.

\begin{lemma}\label{lemma:ties-bsaa}
Let $\bm n \in \mathbb N^K$ and $0 \le x_1 \le \cdots \le x_K \le 1$. Let $x'_1 < \cdots < x'_{K'}$ be the distinct values of $\bm x$ and $n'_j \coloneqq \sum_{k : x_k = x'_j} n_k$ for $j \in [K']$. Then
\begin{equation*}
\mathsf{Reg}(\BSAA_{\bm n}\mid \bm x)\ =\ \mathsf{Reg}(\BSAA_{\bm n'}\mid \bm x')\,,
\end{equation*}
and the formula of \Cref{thm:BSAA} holds for $(\bm n, \bm x)$: with $x_0 \coloneqq 0$, $x_{K+1} \coloneqq 1$, and $\sigma_k = \sum_{\ell=1}^k n_\ell$,
\begin{equation*}
\mathsf{Reg}(\BSAA_{\bm n} \mid \bm x) = \max_{-1 \leq i \leq K}  \sup_{\substack{v \in [0,q], \\ w \in [q,1]}} \quad
        \sum_{k=0}^i (x_{k+1}  - x_k) \cdot \psibsaa_k(v) + \sum_{k=i+1}^K (x_{k+1}  - x_k) \cdot \psibsaa_k(w)\,.
\end{equation*}
\end{lemma}

\begin{proof}
For the first claim, note that under both designs the demands are i.i.d.\ from $F$ and, for each $j \in [K']$, exactly $n'_j$ of them are censored at level $x'_j$. Hence the empirical distribution of the sales, and therefore the order quantity chosen by BSAA, has the same law under $(\bm n, \bm x)$ as under $(\bm n', \bm x')$, for every $F \in \Delta([0,1])$. Taking expectations and the supremum over $F$ in \eqref{eq:worst-case-regret} yields the first claim.

For the second claim, note that each block of identical inventory levels collapses to a single interval of the partition: all of its terms vanish except the last one, which is precisely the corresponding term of the pooled design. Formally, write $\sigma'_j \coloneqq \sum_{\ell=1}^{j} n'_\ell$ for the cumulative counts of the pooled design and $\psi'_j$ for its $j$-th function in \Cref{thm:BSAA}, that is, $\psibsaa_j$ with $\sigma'_j$ in place of $\sigma_j$, and adopt the conventions $x'_0 \coloneqq 0$ and $x'_{K'+1} \coloneqq 1$. Because $\bm x$ is non-decreasing, the indices $k$ with $x_k = x'_j$ form a block of consecutive integers; call $k_j$ its last element, and put $k_0 \coloneqq 0$, so that $0 = k_0 < k_1 < \cdots < k_{K'} = K$. Inside a block, consecutive levels coincide and the interval lengths $x_{k+1} - x_k$ vanish. The only survivors sit at the ends of the blocks, where both the interval and the cumulative count are those of the pooled design: for $j \in \{0, \ldots, K'\}$,
\begin{equation*}
x_{k_j+1} - x_{k_j} = x'_{j+1} - x'_j
\qquad\text{and}\qquad
\sigma_{k_j} = \sum_{\ell=1}^{k_j} n_\ell = \sum_{m=1}^{j} n'_m = \sigma'_j\,.
\end{equation*}
Since $\psibsaa_k$ depends on $k$ only through $\sigma_k$ (and on the common total $n$), $\psibsaa_{k_j} = \psi'_j$. Consequently, for every $i \in \{-1,0,\ldots,K\}$ and all $(v,w) \in [0,q] \times [q,1]$, with $i'=-1$ if $i=-1$ and $i'=\max\{j:k_j\le i\}$ otherwise, we have
\begin{equation*}
\sum_{k=0}^i (x_{k+1}-x_k)\,\psibsaa_k(v) + \sum_{k=i+1}^K (x_{k+1}-x_k)\,\psibsaa_k(w)
\;=\;
\sum_{j=0}^{i'} (x'_{j+1}-x'_j)\,\psi'_j(v) + \sum_{j=i'+1}^{K'} (x'_{j+1}-x'_j)\,\psi'_j(w)\,,
\end{equation*}
For $i=-1$, we have $i'=-1$ and both first sums are empty. Conversely, $i'=-1$ is matched by $i=-1$, and every $i'\in\{0,\ldots,K'\}$ is matched by $i=k_{i'}$. Taking the supremum over $(v,w)$ and the maximum over the crossing index on both sides completes the proof.
\end{proof}

In light of \Cref{lemma:ties-bsaa}, we allow inventory levels to be equal in the remainder of this section, and continue with our analysis of the optimal design problem \eqref{eq:P_global}. First, we show that one can focus on designs where each inventory level receives exactly one sample.

\begin{lemma}\label{lemma:unit-count reduction}
    For each $N\ge 1$, define the unit-count subproblem
\begin{equation*}
\mathcal P_N(B)
\;\coloneqq\;
\inf_{\substack{0\le x_1\le\cdots\le x_N\le 1\\ \sum_{j=1}^N x_j \le B}}
\ \mathsf{Reg}(\BSAA_{\bm 1_N}\mid \bm x).
\end{equation*}
Then $\mathcal P(B)=\inf_{N\ge 1}\ \mathcal P_N(B)$.
\end{lemma}
\begin{proof}
Fix any feasible $(K,\bm n,\bm x)$ for~\eqref{eq:P_global} and let $N\coloneqq\sum_{k=1}^K n_k$, with cumulative counts $\sigma_k\coloneqq\sum_{\ell=1}^k n_\ell$ and $\sigma_0=0$. Define the unit-count expansion $\tilde{\bm x}\in[0,1]^N$ by repeating each $x_k$ exactly $n_k$ times:
\[
\tilde x_j \;=\; x_k \quad \text{for } j=\sigma_{k-1}+1,\ldots,\sigma_k,\qquad k\in[K].
\]
Then $\tilde x_1\le\cdots\le \tilde x_N$ and $\sum_{j=1}^N \tilde x_j=\sum_{k=1}^K n_k x_k\le B$, so $(N,\bm 1_N,\tilde{\bm x})$ is feasible.

We claim that $\mathsf{Reg}(\BSAA_{\bm n}\mid \bm x)=\mathsf{Reg}(\BSAA_{\bm 1_N}\mid \tilde{\bm x})$. To see this, set $x_0=\tilde x_0\coloneqq 0$ and $x_{K+1}=\tilde x_{N+1}\coloneqq 1$. Theorem~\ref{thm:BSAA} shows that for any design the worst-case regret depends on the design only through the interval lengths and the associated indices in the cumulative-count sequence: specifically, it can be written as
\[
\mathsf{Reg}(\BSAA_{\bm n}\mid \bm x)
=
\max_{-1\le i\le K}\ \sup_{\substack{v\in[0,q]\\ w\in[q,1]}}
\left\{
\sum_{k=0}^i (x_{k+1}-x_k)\,\psi_{k}^{(N,\sigma_k)}(v)\;+\;
\sum_{k=i+1}^K (x_{k+1}-x_k)\,\psi_{k}^{(N,\sigma_k)}(w)
\right\},
\]
where $\psi_{k}^{(N,\sigma)}(t)=\big(1-B_{\lceil qN\rceil-\sigma,\;N-\sigma}(t)\big)(t-q)+(q-t)^+$ and $\sigma_k=\sum_{\ell=1}^k n_\ell$.

Apply Theorem~\ref{thm:BSAA} to the expanded unit-count design $(N,\bm 1_N,\tilde{\bm x})$. In that case the cumulative counts are $\tilde\sigma_j=j$ for $j=0,\ldots,N$, so
\[
\mathsf{Reg}(\BSAA_{\bm 1_N}\mid \tilde{\bm x})
=
\max_{-1\le i\le N}\ \sup_{\substack{v\in[0,q]\\ w\in[q,1]}}
\left\{
\sum_{j=0}^i (\tilde x_{j+1}-\tilde x_j)\,\psi_{j}^{(N,j)}(v)\;+\;
\sum_{j=i+1}^N (\tilde x_{j+1}-\tilde x_j)\,\psi_{j}^{(N,j)}(w)
\right\}.
\]
Since $\tilde{\bm x}$ is constant on each block $(\sigma_{k-1},\sigma_k]$, we have $\tilde x_{j+1}-\tilde x_j=0$ unless $j=\sigma_k$ for some $k\in\{0,\ldots,K\}$, and for $j=\sigma_k$ we have
\[
\tilde x_{\sigma_k+1}-\tilde x_{\sigma_k}=x_{k+1}-x_k,
\qquad
\psi_{\sigma_k}^{(N,\sigma_k)}=\psi_{k}^{(N,\sigma_k)}.
\]
For any $i\in\{-1,0,\ldots,N\}$, set $\ell=\max(\{-1\}\cup\{k\in\{0,\ldots,K\}:\sigma_k\le i\})$. The surviving intervals with $j\le i$ are exactly those indexed by $k\le\ell$, so the inner objective equals
\[
\sup_{\substack{v\in[0,q]\\ w\in[q,1]}}
\left\{
\sum_{k=0}^{\ell}(x_{k+1}-x_k)\,\psi_k^{(N,\sigma_k)}(v)
+\sum_{k=\ell+1}^{K}(x_{k+1}-x_k)\,\psi_k^{(N,\sigma_k)}(w)
\right\}.
\]
Every grouped crossing index is obtained: choose $i=-1$ for $\ell=-1$, and $i=\sigma_\ell$ for $0\le\ell\le K$. In particular, $i=N$ gives $\ell=K$.
Taking the maximum over the expanded indices is therefore equivalent to taking the maximum over $\ell\in\{-1,0,\ldots,K\}$ in the grouped expression, with empty sums interpreted as zero. Hence
\[
\mathsf{Reg}(\BSAA_{\bm 1_N}\mid \tilde{\bm x}) \;=\; \mathsf{Reg}(\BSAA_{\bm n}\mid \bm x),
\]
and taking infima over feasible completes the proof.
\end{proof}

For the unit-count case, Theorem~\ref{thm:BSAA} yields
\[
\mathsf{Reg}(\BSAA_{\bm 1_N}\mid \bm x)
=
\Phi_N(\bm x)
\;\coloneqq\;
\max_{-1\le i\le N}\Big\{\sup_{v\in[0,q]} A_{i,N}(\bm x,v)+\sup_{w\in[q,1]} B_{i,N}(\bm x,w)\Big\},
\]
where $x_0\coloneqq 0$, $x_{N+1}\coloneqq 1$ and
\[
A_{i,N}(\bm x,v)\coloneqq \sum_{k=0}^i (x_{k+1}-x_k)\,\psi_{k}^{(N)}(v),
\qquad
B_{i,N}(\bm x,w)\coloneqq \sum_{k=i+1}^N (x_{k+1}-x_k)\,\psi_{k}^{(N)}(w),
\]
Here $i\in\{-1,0,\ldots,N\}$, $A_{-1,N}=0$, and $B_{N,N}=0$. The functions are defined by
\begin{equation}\label{eq:psi_def_unit}
\psi_k^{(N)}(t)
\;=\;
\Big(1-B_{\lceil qN\rceil-k,\;N-k}(t)\Big)(t-q) + (q-t)^+,
\qquad k=0,1,\dots,N.
\end{equation}

Let $V\subset[0,q]$ and $W\subset[q,1]$ be finite grids and define the grid objective
\[
\widehat\Phi_{N,V,W}(\bm x)
\;\coloneqq\;
\max_{-1\le i\le N}\Big\{\max_{v\in V} A_{i,N}(\bm x,v)+\max_{w\in W} B_{i,N}(\bm x,w)\Big\}.
\]
Then $\widehat{\mathcal P}_{N}(B;V,W)\coloneqq \inf_{\bm x}\widehat\Phi_{N,V,W}(\bm x)$ is the optimal value of a \emph{finite linear program} (the epigraph LP with variables $\bm x,\ t,\ (\alpha_i,\beta_i)_{i=-1}^N$ and constraints $\alpha_i\ge A_{i,N}(\bm x,v)$ for $v\in V$, $\beta_i\ge B_{i,N}(\bm x,w)$ for $w\in W$, and $t\ge \alpha_i+\beta_i$).

\begin{lemma}\label{lemma:discretization-error}
    Let $\Delta_V\coloneqq \max\{|v-v'|: v,v'\text{ adjacent in }V\}$ and similarly $\Delta_W$.
Then for every $N\ge 1$,
\begin{equation*}
0\ \le\ \mathcal P_N(B)-\widehat{\mathcal P}_{N}(B;V,W)
\ \le\ 2(N+2)\,\max\{\Delta_V,\Delta_W\}.
\end{equation*}
\end{lemma}

\begin{proof}
    Fix $N$ and any feasible $\bm x$. Let $\Delta\coloneqq \max\{\Delta_V,\Delta_W\}$.
We claim that each $\psi_k^{(N)}$ is $(N+2)$-Lipschitz on $[0,1]$; then since
$\sum_{k=0}^N (x_{k+1}-x_k)=1$, both $v\mapsto A_{i,N}(\bm x,v)$ and $w\mapsto B_{i,N}(\bm x,w)$ are also $(N+2)$-Lipschitz, uniformly in $i$.
Therefore
\[
\sup_{v\in[0,q]}A_{i,N}(\bm x,v)
\ \le\
\max_{v\in V}A_{i,N}(\bm x,v)\ +\ (N+2)\Delta,
\qquad
\sup_{w\in[q,1]}B_{i,N}(\bm x,w)
\ \le\
\max_{w\in W}B_{i,N}(\bm x,w)\ +\ (N+2)\Delta,
\]
and taking the maximum over $i$ yields
\[
\Phi_N(\bm x)\ \le\ \widehat\Phi_{N,V,W}(\bm x)\ +\ 2(N+2)\Delta.
\]
Since $\widehat\Phi_{N,V,W}(\bm x)\le \Phi_N(\bm x)$ pointwise, minimizing over $\bm x$ yields the lemma.

It remains to justify the Lipschitz bound. For $t\neq q$, $\psi_k^{(N)}$ is differentiable and
\[
\big|\tfrac{d}{dt}\psi_k^{(N)}(t)\big|
\ \le\
1+\big| \tfrac{d}{dt} B_{\lceil qN\rceil-k,\;N-k}(t)\big|.
\]
For $r\in\{0,\dots,m\}$, the derivative of the binomial tail is a Beta density:
\[
\frac{d}{dt}B_{r,m}(t)
=
\frac{t^{r-1}(1-t)^{m-r}}{\mathrm{B}(r,m-r+1)}.
\]
When $r\in\{1,\dots,m\}$ this is the $\mathrm{Beta}(r,m-r+1)$ pdf; for $r\le 0$ or $r>m$ the tail is constant and the derivative is $0$.
For integer parameters $(a,b)=(r,m-r+1)$, the Beta pdf attains its maximum at $t=(a-1)/(a+b-2)$ (or at the boundary if $a=1$ or $b=1$), and one may bound
\[
\sup_{t\in[0,1]}\frac{t^{a-1}(1-t)^{b-1}}{\mathrm{B}(a,b)}
\ \le\ a+b-1
\ =\ m+1,
\]
by combining the closed form $1/\mathrm{B}(a,b)=(a+b-1)\binom{a+b-2}{a-1}$ with the standard inequality
$\binom{n}{k}\le n^n/(k^k(n-k)^{n-k})$ (take $n=a+b-2$, $k=a-1$). Hence
$\sup_t |B'_{r,m}(t)|\le m+1$, and since $m=N-k\le N$, we obtain
$\sup_t |\psi_k^{(N)\prime}(t)|\le 1+(N+1)=N+2$ as claimed.
\end{proof}

\begin{lemma}\label{prop:Kmax_alt}
Fix $B\in\mathbb Z_+$ and $q \geq 0.5$. Consider the global BSAA design problem
\[
\mathcal P(B)
\;=\;
\inf_{\substack{K\ge 1,\ \bm n\in\mathbb N^K\\ 0\le x_1\le\cdots\le x_K\le 1\\ \sum_{k=1}^K n_k x_k \le B}}
\ \mathsf{Reg}(\BSAA_{\bm n}\mid \bm x).
\]
Let $\bar U(B,q)$ be the worst-case regret of the \emph{fully uncensored} feasible design that spends the entire budget on $x=1$, i.e.\ $(K,\bm n,\bm x)=(1,(B),(1))$. Then:
\begin{enumerate}
\item[\textbf{(i)}] $\mathcal P(B)\le \bar U(B,q)$.
\item[\textbf{(ii)}] For any feasible design with total sample size $N=\sum_k n_k$,
\begin{equation}\label{eq:LB_delta1_again}
\mathsf{Reg}(\BSAA_{\bm n}\mid \bm x)
\;\ge\;
q\left(1-\min\!\left\{1,\ \frac{B}{\,N-\lceil qN\rceil+1\,}\right\}\right).
\end{equation}
\item[\textbf{(iii)}] There exists an optimal design whose unit-count expansion has
\begin{equation}\label{eq:Nmax_alt}
N^\star\ \le\ N_{\max}(B,q)
\;\coloneqq\;
\left\lceil
\frac{B}{(1-q)\left(1-\frac{\bar U(B,q)}{q}\right)}
\right\rceil,
\end{equation}
and therefore it suffices to enumerate $N=1,2,\dots,N_{\max}(B,q)$ (equivalently, $K$ in the unit-count representation) when solving $\mathcal P(B)$.
\end{enumerate}
\end{lemma}

\begin{proof}
\textbf{Step 1.}
The design $(K,\bm n,\bm x)=(1,(B),(1))$ is feasible since $Bx_1=B$. Hence $\mathcal P(B)\le \mathsf{Reg}(\BSAA_{(B)}\mid (1))\eqqcolon \bar U(B,q)$.

\textbf{Step 2.}
Fix any feasible design and let $N=\sum_k n_k$. Evaluate regret at $F=\delta_1$. Then $\opt(\delta_1)=0$ (order $a=1$) and ordering $a$ incurs cost $c(a,\delta_1)=q(1-a)$. Under $\delta_1$, sales equal inventories deterministically, so BSAA outputs the empirical $q$-quantile of the multiset of historical inventories, i.e.\ the smallest inventory level whose cumulative count is at least $\lceil qN\rceil$. Call this value $x_{k^\star}$. Using $\sum_k n_k x_k\le B$ and monotonicity of $\bm x$, we have
\[
B\ \ge\ \sum_{\ell\ge k^\star} n_\ell x_\ell\ \ge\ x_{k^\star}\sum_{\ell\ge k^\star}n_\ell
\ \ge\ x_{k^\star}\,(N-\lceil qN\rceil+1),
\]
so $x_{k^\star}\le \min\{1,\ B/(N-\lceil qN\rceil+1)\}$. Therefore
\[
\mathsf{Reg}(\BSAA_{\bm n}\mid \bm x)
\;\ge\;
\mathbb E_{\bm D\sim \delta_1^N}[R(\BSAA_{\bm n}(\bm x,\bm I),\delta_1)]
\;=\;
q(1-x_{k^\star})
\;\ge\;
q\left(1-\min\!\left\{1,\ \frac{B}{N-\lceil qN\rceil+1}\right\}\right),
\]
which is~\eqref{eq:LB_delta1_again}.

\textbf{Step 3.}
When $q \geq 0.5$, the Newsvendor loss function is $q$-Lipschitz. As a consequence, we must have $\bar U(B,q)<q$. Therefore, any $N$ for which the right-hand side of~\eqref{eq:LB_delta1_again} is at least $\bar U(B,q)$ cannot be optimal, since there exists a feasible design achieving $\bar U(B,q)$. Using $N-\lceil qN\rceil+1\ge (1-q)N$, it suffices that
\[
q\left(1-\frac{B}{(1-q)N}\right)\ \ge\ \bar U(B,q),
\]
which rearranges to $N\ge B/((1-q)(1-\bar U/q))$, yielding~\eqref{eq:Nmax_alt}.
\end{proof}

Combining everything, the following proposition provides an algorithm for computing the optimal design for BSAA.

\begin{proposition}\label{prop:bsaa-exploration-LP}
    Fix $\varepsilon>0$, $q \geq 0.5$, and set $\delta\coloneqq \varepsilon/(2(N_{\max}+2))$, where $N_{\max}=N_{\max}(B,q)$.
For each $N=1,2,\dots,N_{\max}$:
\begin{enumerate}
\item take uniform grids $V_N\subset[0,q]$ and $W_N\subset[q,1]$ with mesh at most $\delta$;
\item solve the corresponding epigraph LP to obtain $\widehat{\mathcal P}_N(B;V_N,W_N)$ and an optimizer $\hat{\bm x}^{(N)}$.
\end{enumerate}
Return $\widehat N\in\argmin_{1\le N\le N_{\max}}\widehat{\mathcal P}_N(B;V_N,W_N)$ and $\hat{\bm x}\coloneqq \hat{\bm x}^{(\widehat N)}$ (with $\hat{\bm n}=\bm 1_{\widehat N}$).

Then the returned design is $\varepsilon$-optimal for~\eqref{eq:P_global}:
\begin{equation}\label{eq:eps_opt}
\widehat{\mathcal P}(B)\ \le\ \mathcal P(B)\ \le\ \widehat{\mathcal P}(B)+\varepsilon,
\qquad
\widehat{\mathcal P}(B)\coloneqq \min_{1\le N\le N_{\max}}\widehat{\mathcal P}_N(B;V_N,W_N).
\end{equation}
\end{proposition}

\end{document}